\documentclass[mnsc,nonblindrev]{informswp}
\OneAndAHalfSpacedXI
\usepackage{geometry}
\usepackage{endnotes}
\usepackage{amsmath}
\usepackage{amsfonts}
\usepackage{amssymb}
\usepackage{mathrsfs} 
\usepackage{graphicx}
\usepackage{subfig}
\usepackage{color}
\usepackage{multirow}
\usepackage{dsfont}
\usepackage{float}
\usepackage{epstopdf}
\usepackage{lipsum}
\usepackage{appendix}
\usepackage{enumitem}
\usepackage{bm}
\usepackage{algorithmic}
\usepackage{algorithm}
\usepackage{url}
\usepackage{hyperref}
\usepackage{bbm}
\usepackage{gensymb}

\hypersetup{
    colorlinks=true, 
    linktoc=all,     
    linkcolor=black!70!green,  
    citecolor=black!70!green 
}

\usepackage{nicefrac} 
\usepackage{soul}
\usepackage{algorithmic}
\usepackage{algorithm}
\usepackage{booktabs} 
\usepackage{comment}
\usepackage{multicol}

\newcommand{\mb}[1]{\mbox{\boldmath $#1$}}
\newcommand{\mbt}[1]{\mbox{\boldmath $\tilde{#1}$}}

\newcommand{\E}{\mathbb{E}}

\newcommand{\mbc}[1]{\mbox{\boldmath $\check{#1}$}}

\newcommand{\bmb}[1]{\bm{\bar{#1}}}
\newcommand{\bmc}[1]{\bm{\check{#1}}}
\newcommand{\bmh}[1]{\bm{\hat{#1}}}

\newcommand{\bmt}[1]{\bm{\tilde{#1}}}

\newcommand{\cd}{\mathcal{D}}

\newcommand{\cu}{\mathcal{U}}

\newcommand{\cw}{\mathcal{W}}
\newcommand{\cx}{\mathcal{X}}
\newcommand{\cy}{\mathcal{Y}}
\newcommand{\cz}{\mathcal{Z}}

\newcommand{\bbr}{\mathbb{R}}

\usepackage{tabularx}
\newcolumntype{L}[1]{>{\raggedright\arraybackslash}p{#1}}
\newcolumntype{C}[1]{>{\centering\arraybackslash}p{#1}}
\newcolumntype{R}[1]{>{\raggedleft\arraybackslash}p{#1}}

\newtheorem{illustration}{Illustration}

\newcommand{\mbh}[1]{\mbox{\boldmath $\hat{#1}$}}

\newcommand{\sstext}[1]{\scriptscriptstyle{\text{#1}}}

\usepackage{color}
\usepackage[usenames,dvipsnames]{xcolor}
\definecolor{strcolor}{rgb}{0.6, 0.2, 0.6}
\definecolor{commentcolor}{rgb}{0.3125, 0.5, 0.3125}
\definecolor{keycol}{rgb}{0, 0, 1}
\definecolor{ballblue}{rgb}{0.13, 0.67, 0.80}
\newcommand{\gged}[1]{{\color{Mahogany} GG: #1}}

\newcommand{\code}[1]{\texttt{\bfseries{#1}}}

\usepackage{color-edits}
\addauthor{rv}{black}
\addauthor{xz}{magenta}
\addauthor{ks}{cyan}
\addauthor{gg}{ForestGreen}
\addauthor{yg}{purple}

\usepackage{listings}
\lstdefinelanguage{Julia}%
{morekeywords={abstract,break,case,catch,const,continue,do,else,elseif,%
		end,export,false,for,function,immutable,import,importall,if,in,%
		macro,module,otherwise,quote,return,switch,true,try,type,typealias,%
		using,while},%
	sensitive=true,%
	alsoother={\$},%
	morecomment=[l]\#,%
	morecomment=[n]{\#=}{=\#},%
	morestring=[s]{"}{"},%
	morestring=[m]{'}{'},%
}[keywords,comments,strings]%

\usepackage{natbib}
\bibpunct[, ]{(}{)}{,}{a}{}{,}%
\def\bibfont{\small}%
\def\bibsep{\smallskipamount}%
\def\bibhang{24pt}%

\TheoremsNumberedThrough     
\ECRepeatTheorems
\EquationsNumberedThrough    

\begin{document}
\RUNAUTHOR{Loke et al.}
\RUNTITLE{Decision-Driven Regularization}
\TITLE{Decision-Driven Regularization \\ \large{A Blended Model for Learning and Optimization}}
\ARTICLEAUTHORS{
\AUTHOR{Gar Goei Loke}
\AFF{Department of Management and Marketing, Durham University Business School, Waterside, Durham DH1 1SL, 
United Kingdom, \EMAIL{gar.g.loke@durham.ac.uk}}

\AUTHOR{Qinshen Tang\footnote{Corresponding author. The authors are listed in alphabetical order with implied equal authorship.}}
\AFF{Nanyang Business School, 
Nanyang Technological University, Singapore, 
\EMAIL{qinshen.tang@gmail.com}}

\AUTHOR{Yangge Xiao}
\AFF{
Faculty of Business and Economics, The University of Melbourne, 
\EMAIL{yangge.xiao@unimelb.edu.au}}

\AUTHOR{Xun Zhang}
\AFF{International Institute of Finance, School of Management, University of Science and Technology of China, Hefei 230036, China.
\EMAIL{xunzhang2023@outlook.com}}
} 

\ABSTRACT{
In contextual optimization, the decision-maker seeks optimal decisions to minimize a cost function, that varies based on observed features. This context is common in many business applications ranging from on-demand delivery and retail operations to portfolio optimization and inventory management. In this paper, we study the learning and optimization approach, which first learns how outcomes result from the features, and then selects optimal decisions based on these outcomes. We focus on the integrated learning and optimization literature, and identify that a lack of control for prediction accuracy can lead to overfitting and a loss of decision effectiveness against simple separate learning and optimization models. Instead, we propose a bi-objective formulation that balances prediction accuracy and cost minimization, termed \emph{decision-driven regularization}. It also addresses ambiguity in the definition of the cost function via a surrogate that depends on a new hyperparameter. We additionally show that alternative perspectives for formulating the problem, namely robust optimization and regret minimization, lead to models that are closely related to our proposed model. As a consequence, our framework generalizes models such as SPO+ in \cite{Elmachtoub_Grigas_2022smart}.
Our model is shown to be numerically superior to other benchmarks, such as OLS, Random Forest, XGBoost, SPO+, Perturbation Gradient, and Learning and Rank, in our synthetic studies.
}

\KEYWORDS{learning and optimization, robust optimization, worst-case regret minimization, regularization, 
contextual optimization }

\maketitle

\vspace{-1cm}
\section{Introduction}

In this paper, we consider the decision-making under uncertainty setting, where a decision-maker solves a cost minimization problem
$
    \min\limits_{\bm y \in \cy} ~ \mathbb{E}_{\bm z} [c(\bm y; \bm z)],
$
for some decisions $\bm y$ in a feasibility set $\mathcal{Y}\subseteq \mathbb{R}^d$ and some unobserved outcomes $\bm z \in \cz \subseteq \mathbb{R}^d$. In most practical settings, the distribution of $\bm z$ is unknown and the focus is on using a collection of, say $N$, data points $\{\bmt{z}_1, \ldots, \bmt{z}_N\}$ to make good decisions $\bm y$, cognizant that this is a sample that may not fully represent the true distribution.
such as \cite{BenTal_etal_2013robust}, \cite{Delage_Ye2010distributionally},
\cite{Esfahani_Kuhn2018data}, \cite{Bertsimas_Kallus_2020predictive}.

Efforts to improve the inference of the distribution of $\bm z$ has led to the \emph{contextual} setting, where features (also called side information or covariates), denoted as $\bm x \in \cx \subseteq \mathbb{R}^p$, are related to unobserved outcomes $\bm z$ via $\bm z = \bm g(\bm x) + \bm \epsilon$ for some map $\bm g: \cx \rightarrow \cz$ and some noise $\bm \epsilon \in \mathbb{R}^d$. The presence of features $\bm x$ allows the decision-maker not just to improve the estimation of uncertain outcomes $\bm z$, but also to allow decisions $\bm y$ to adapt to features $\bm x$ \citep{Hertog_Postek_2016_bridging}. Specifically, the decision-maker usually solves the following problem:
\begin{align}
    \bm y(\bm x) &\in \argmin_{\bm y \in \mathcal{Y}}~\mathbb{E}_{\bm z\mid \bm x}\big[c(\bm y; \bm z) \mid \bm x \big], \qquad \forall \bm x \in \mathcal{X} \label{eq.contextual}
\end{align}
where $\bm z | \bm x$ are the conditional random outcomes.

This setting is often referred to as contextual stochastic optimization (or sometimes, decision-aware learning, or joint prediction and optimization). 
It is increasingly common and can be seen in various contexts ranging from on-demand delivery (\citealp{Liu_He_Shen_2021time}) to retail operations (\citealp{
Perakis_Sim_Tang_Xiong_2023joint}), and from portfolio optimization (\citealp{Ban_ElKaroui_Lim_2018machine}) to inventory management (\citealp{Siege_Wagner_l2020profit, Qi_2023practical}), among many others. 
A substantial and growing body of work proposes different approaches to solve the contextual stochastic optimization problem (see, for example, \citealt{Mandi_2024decision, Sadana_et_al_2024survey}). In what follows, we organize our discussion on related works according to the classification framework of \cite{Sadana_et_al_2024survey}. In particular, we focus on the learning and optimization framework.

The learning and optimization framework \eqref{eq.contextual}, which includes sequential learning and optimization (SLO) and integrated learning and optimization (ILO), proceeds in two stages. First, it fits a predictive model to approximate the conditional distribution of $\bm z | \bm x$ in \eqref{eq.contextual}; second, it solves an optimization problem that minimizes the expected cost under this predicted distribution. 

\subsubsection*{Sequential learning and optimization.} 

The SLO approach decouples the process of learning from the subsequent optimization step. In other words, in no part is any information about the cost minimization problem used in the estimation of the relationship of $\bm z|\bm x$. It is broadly divided into the predict-then-optimize (PTO) and estimate-then-optimize (ETO) streams \citep{Qi_Grigas_Shen_2021integrated}. In the former, `point predictions' $\mathbb{E}[\bm z | \bm x]$ are estimated from the learning model to solve for the optimal decision under $\min_{\bm y} c(\bm y; \mathbb{E}[\bm z | \bm x])$. Examples of this literature include, for example, \cite{Ferreira_Lee_Simchi_Levi_2016analytics} and \cite{Liu_He_Shen_2021time}.

PTO is increasingly less popular as it interchanges the cost function $c$ and the expectation operator $\mathbb{E}$, as compared to the original stochastic problem \eqref{eq.contextual}. When $c$ is affine in $\bm z$, this is equivalent to \eqref{eq.contextual}. In more general settings, others have tried to explore how to estimate the distribution $\bm z | \bm x$. This is termed as ETO. A notable stream of work focuses on constructing a distribution from machine learning predictions \citep{Hannah_et_al_2010nonparametric, Bertsimas_Kallus_2020predictive, Srivastava_et_al_2021data, Lin_et_al_2022data, Notz_Pibernik_2022}, or equivalently from the regression residuals \citep{Deng_Sen_2018learning, Ban_Gallien_Mersereau_2019dynamic,  Kannan_et_al_2025data}. To better account for uncertainty in the estimation of the conditional distribution, distributionally robust optimization techniques have been applied \citep{Bertsimas_VanParys2022bootstrap, Esteban_Morales_2022distributionally, Perakis_Sim_Tang_Xiong_2023joint, Kannan_Bayraksan_Luedtke_2024residuals}. 

A growing body of literature highlights the potential suboptimality of separating the learning and optimization stages \cite[see \emph{e.g.},][and the ILO literature that we discuss later]{Liyanage_Shanthikumar_2005_practical, Mundru_2019predictive}. The core issue is that high prediction accuracy in the learning stage does not necessarily translate into good performance in the subsequent decision-making step. To illustrate this, consider the widely adopted parametric setting for problem~\eqref{eq.contextual}, which posits that the true map $\bm g: \mathcal{X} \rightarrow \mathcal{Z}$ lies in a parametric family $\mathcal{H} = \left\{ \bm{f}(\bm x; \bm w) \,:\, \bm w \in \mathcal{W} \right\}$. Given a dataset of observed feature-outcome pairs $\{(\bmt{x}_n, \mbt{z}_n)\}_{n=1}^N$, the decision-maker estimate parameters $\bmh w  \in\argmin_{\bm w} L(\bm w)$, where $L(\bm w)$ denotes an accuracy-based loss function. For SLO, as explained in \cite{Elmachtoub_Grigas_2022smart}, the direction in the parameter space of $\bm w$ that minimizes loss function $L(\bm w)$ does not necessarily align with the direction that minimizes the cost function $c(\bm y; \bm{f}(\bm x; \bm w))$. To further illustrate this misalignment, we present a knapsack example in Appendix~\ref{append.relationship}.

\subsubsection*{Integrated learning and optimization.}

Most initially, ILO departs from SLO by embedding the subsequent optimization problem directly into the training objective, so that model parameters are chosen to minimize empirical decision cost rather than the conventional accuracy loss. This paradigm dates back to \cite{Bengio_1997using}, who, in the context of portfolio management, first advocated tuning predictive models to maximize realized returns rather than to accurately predict prices.

In recent years, this concept has gained significant attention. One major stream of work focuses on adapting the prediction models in SLO methods so that it takes into consideration the subsequent cost minimization problem.
When PTO methods are adapted to the ILO context, this is sometimes termed the expected value-based model (\textit{e.g.}, \citealt{Amos_Kolter_2017optnet}). ETO methods have also been adapted to factor in cost minimization (\textit{e.g.}, \citealt{Donti_Amos_Kolter_2017task}, \citealt{Qi_Grigas_Shen_2021integrated}, \citealt{Kallus_Mao_2023stochastic}). This also includes robust models that consider uncertainty in the estimation (\textit{e.g.}, \citealt{Sim_Tang_Zhou_Zhu_2024analytics}, \citealt{Wang_Chen_Wang_2024Contextual}). 

A separate stream of work focuses on minimizing regret, which measures the gap between the decision cost under the learned model and the oracle's ``hindsight" decision (\textit{e.g.}, \citealt{Elmachtoub_Grigas_2022smart}, \citealt{Jeong_et_al_2022exact}, \citealt{Estes_Richard_2023smart}). However, the regret minimization problem is often non-convex, leading to a focus on algorithms for optimizing the surrogate loss (\textit{e.g.}, \citealt{Athey_Tibashirani_Stefan_2019generalized,
el_Elmachtoub_Grigas_2019generalization, 
Vlastelica_et_al_2019differentiation,
Wilder_et_al_2019end,
Mandi_Demirovic_Stuckey_Guns_2019smart,
Chung_et_al_2022decision,
HoNguyen_KilincKarzan_2022risk, 
Hu_Kallus_Mao_2022fast,  
Lawless_Zhou_2022note,
Munoz_Pineda_Morales_2022bilevel, 
Estes_Richard_2023smart}). 
\rvedit{
Recent studies have also developed differentiable or solver-efficient approaches for decision-focused learning, including differentiating through regularized integer linear programs \citep{mckenzie2024differentiating}, designing surrogate losses for differentiable optimization layers \citep{mandi2025minimizing}, and constructing solver-free losses for linear optimization problems \citep{berden2026solverfree}.}

A final stream assumes that the optimizer is exposed to an oracle who is able to generate $\bm y^\star(\bm x)$ given $\bm x$, and seeks a function that best imitates the optimal solution $\bm y^\star$. This is referred to as imitation learning (\textit{e.g.}, \citealt{Kong_et_al_2022end}).
When the model assumed by the optimizer explicitly models the learning through an assumed learner, we sometimes see the application of inverse optimization techniques, see \cite{Chan_Mahmood_Zhu_2025inverse} for a comprehensive review.

In relation to our work, we find that the first two streams, namely, extending SLO to ILO and regret minimization to be most closely related. The setting we consider is linear, thus both adaptations of PTO and ETO to the ILO setting are relevant. Additionally, we later theoretically illustrate deep connections between the regret minimization problem and the cost minimization approach (Theorem \ref{thm.regret ddr}), drawing our link to the regret minimization stream.

A commonality observed in the numerical results of works of these two streams is that ILO only begins to outperform the SLO model with ordinary least squares (OLS) estimator when the misspecification is raised to a high degree.  
This is sometimes justified as being due to the ``inherent robustness properties of the least squares loss" \citep{Elmachtoub_Grigas_2022smart}. Additionally, these works note that their models consistently perform similarly or better than other SLO models that employ state-of-the-art machine learning techniques to form predictions, to further justify the outlying behavior of OLS. We, however, beg to differ. There is growing evidence in the consistent superiority of OLS across multiple applications and cost functions against ILO models \citep{Hu_Kallus_Mao_2022fast}, culminating in the work \cite{elmachtoub2023estimate} which proves the superiority of SLO under the non-misspecified regime.

In our work, we take a different approach---we view OLS as the most basic benchmark that any ILO model should defeat, and aspire to craft a general ILO framework that can do so in the low misspecification regime and by extension, outperform other ILO models. We believe the missing ingredient thus far in the ILO literature is the lack of careful control for prediction accuracy when undertaking cost or regret minimization, which we discuss in depth in \S\ref{subsec.cost fun ambiguity}.

\subsection{Approach and contributions} \label{subsec.approach and contributions}

In this paper, we propose a model within the ILO setting that balances prediction accuracy and cost minimization. This is done by incorporating a \emph{decision-driven regularization} in the learning process that captures the optimal value of the decisions that can be obtained if a particular choice of weights $\mb w$ was chosen in the predictive stage. In particular, we make the following contributions:

\begin{enumerate}[label = \alph*., itemsep = 2pt]
    \item \ul{We explain the importance of prediction accuracy in ILO to motivate our blended bi-objective framework, termed `DDR'.} Because in ILO, the decision-maker can alter the weights of the estimation model to one that artificially depresses the cost function so that an extreme decision becomes attainable where it would not be optimal or feasible under the true cost function. This will lead to overfitting, which we exhibit in Illustration \ref{illust.bad_weights}. Hence, we propose a bi-objective formulation that balances prediction accuracy and cost minimization.
	
    \item \ul{We propose a combination of the empirical and estimated cost to define a surrogate for the cost function}. To balance prediction accuracy and cost minimization, we have to first define the cost function, which is ambiguous. We illustrate that the most accurate depiction of the cost function is not the estimated cost, but rather one with a slight bias in the estimate, as a result of bias-variance trade-off (Proposition~\ref{prop.bv}). In particular, we propose a surrogate of the cost function by minimizing the predicted cost while controlling its difference from the empirical cost. 
	
    \item \ul{We prove that the robust optimization and regret minimization starting points can lead to DDR, under specific circumstances}. Specifically, we construct model formulations from the perspectives of robust optimization and regret minimization, and prove their relations to our model in Theorems \ref{thm.rddr} and \ref{thm.regret ddr} respectively. Thus, we show that multiple models in the literature are special cases of DDR, including SPO$+$; \label{point.persp} 
	
    \item \ul{We provide comprehensive synthetic numerical evidence to illustrate the superiority of DDR against the benchmarks tested}, including popular prediction models for SLO and multiple ILO models. DDR outperforms all these benchmarks under low misspecification. Only SPO+ manages to outperform DDR under large sample sizes and high misspecification, but loses out to DDR when small-order polynomial terms are added to both learners.
\end{enumerate}

\section{Decision-Driven Regularization} \label{section::model description} 

In this work, we restrict our attention to the simple but widely used bilinear cost function $c(\bm y; \bm z) := \bm y^{\top} \bm z$. We also assume a parametric model for the map $\bm g: \mathcal{X} \rightarrow \mathcal{Z}$, namely, $\bm g(\bm x) = \bm{f}(\bm x; \bm w)$ for some $\bm w \in \cw \subseteq \mathbb{R}^q$. Under no misspecification, there exist some true but unobserved weights $\bmc w$ such that $\bm g(\bm x) = \bm{f}(\bm x; \bmc w)$. In practice, however, the decision-maker observes only finite, noisy samples $\bmt z_n = \bm g(\mbt{x}_n) + \mbt{\epsilon}_n$ for $n \in [N]$, where each $\bmt \epsilon_n$ is the realized error and $[N] = \{1,2,\cdots, N\}$ is the set of indices up to $N$. We collect these into the training data set $\mathcal{D}^N = \left\{(\mbt{x}_n,\mbt{z}_n)_{n\in[N]} \right\}$. Under the SLO and ILO frameworks, the decision process occurs in two stages:
\begin{enumerate}
    \item (Prediction) Estimate the weights $\bmh{w}$ from training data set $\mathcal{D}^N$ through optimizing a loss function $L(\bm w)$, \textit{i.e.},
    \begin{equation*}
        \bmh w \in \argmin_{\bm w}\; L(\bm w),
    \end{equation*}
    yielding the predictor $\bmh z(\bm x) =\bm{f}(\bm x; \bmh w)$ for any new features $\bm{x}$.
    \item (Optimization) Decisions are then made by solving the approximate problem
    \begin{equation*}
        \min_{\bm y \in \mathcal{Y}} \; \bm y^\top \bmh z(\bmh x).
    \end{equation*}
\end{enumerate}
Hereinafter, we denote true values with the accent $\check{ }$\;, empirical values with the accent $\tilde{ }$\;, and estimated or inferred values with the accent $\hat{ }$\;.

In SLO, the decision-maker estimates weights $\bmh{w}$ by minimizing some \textit{accuracy-based} loss function $L(\bm w)$ that is chosen independently of the cost function. The loss can be based on a measure of prediction error, such as $L(\bm w) = \mathbb{E}_{\bm x, \bm z}[\ell(\bm{f}(\bm x; \bm w), \bm z)]$, where $\ell: \cz \times \cz \mapsto \bbr$ measures the closeness between predicted and observed outcomes. A common example is the mean squared error (MSE), where $\ell(\bmh z; \bmt z) = \|\bmh z - \mbt z\|_2^2$. This loss is typically estimated in-sample using the dataset $\cd^N$ as
$
L(\bm w) = \frac{1}{N} \sum_{n \in [N]} \ell\big(\bm{f}(\bmt x_n; \bm w); \bmt z_n\big).
$ In many cases, the loss function is augmented with a regularization term to prevent overfitting, for instance Lasso regression with $L(\bm w) = \frac{1}{N}\sum\limits_{n \in [N]} \ell\big(\bm{f}(\bmt x_n; \bm w); \bmt z_n\big) + \theta \|\bm{w}\|_1$ or Ridge regression with $L(\bm w) = \frac{1}{N}\sum\limits_{n \in [N]}\ell\big(\bm{f}(\bmt x_n; \bm w); \bmt z_n\big) + \theta \|\bm{w}\|_2^2$.

\subsection{The peril of forgoing prediction accuracy} \label{subsec.The peril of forgoing prediction accuracy}
As discussed in the introduction, when the loss function $L(\bm w)$ does not factor in the cost function---as in SLO---the direction of optimal $\mbh w$ might not align with the one leading to lower cost. ILO addresses this limitation by involving the cost function $c(\bm y; \bm z)$ in the prediction stage. Let $\bm y^\star(\bm z) = \argmin_{\bm y \in \mathcal{Y}} \bm y^\top \bm z$. Many ILO approaches adopt a \textit{decision-focused} loss function $L(\bm w)$\footnote{For simplicity, we use the same notation $L(\bm w)$ to denote both accuracy-based and decision-focused loss functions.}, which is explicitly related to expected cost or regret, such as, 
\begin{subequations}\label{eqn.decision-focused loss funtion}
\begin{align}
    L(\bm w) &= \E_{\bm x, \bm z} \left[\bm y^\star(\bm{f}(\bm x, \bm w))^\top\bm z \right] \label{model:fully_cost-driven} \\
  ~~\text{or}~~ L(\bm w) &= \E_{\bm x, \bm z}\left[\bm y^\star(\bm{f}(\bm x, \bm w))^\top\bm z - \min_{\bm y\in\mathcal{Y}} \; \bm y^\top\bm z \right]. \label{model:fully_cost-driven_expost}
\end{align}
\end{subequations}

Before proceeding, we remark here that in \eqref{model:fully_cost-driven_expost}, since $\bm z$ in the oracle term $\min \; \bm y^{\top} \bm z$ is unaffected by $\bm w$ (such as when one uses SAA to approximate \eqref{model:fully_cost-driven_expost}), the $\mbh w$'s obtained from minimizing \eqref{model:fully_cost-driven} and \eqref{model:fully_cost-driven_expost} coincide.

One representative ILO approach is the SPO loss, proposed by  \cite{Elmachtoub_Grigas_2022smart}. 
They study a model to find the weights for the prediction problem that minimizes the regret: 

\begin{equation}\label{model:spoloss} \tag{\code{SPO}}
	\bmh w^{\rm SPO}  := 
    \argmin_{\bm{w}}~\frac{1}{N} \sum_{n \in [N]}\Big[\max_{\bm y\in \mathcal{Y}^*(\bm{\hat z})}\{\bm y^{\top} \bmt z_n\} - \min_{\bm y_n\in\cy}\{ \bm{y}_n^{\top} \bmt z_n\}\Big],
\end{equation} 
where $\mathcal{Y}^*(\bm{\hat z}) = \arg\min_{\bm y\in\mathcal{Y}}\bm y^{\top}\bm {\hat z}$.
This is a sample average approximation of \eqref{model:fully_cost-driven_expost}. The~\ref{model:spoloss} formulation can potentially be nonconvex. Instead, the following surrogate, which is convex and Fisher consistent to~\ref{model:spoloss}, is proposed. This is named ``SPO+": 
\begin{align}
	\bmh{w}^{\rm SPO+} = \argmin\limits_{\bm{w}} & \quad 2\frac{1}{N}\sum_{n\in [N]} \bm y^\star(\bmt z_n) ^{\top}\bm{f}(\bmt x_n; \bm w)  + \frac{1}{N}\sum_{n\in[N]} \max\limits_{\bm y_n \in \cy} \big\{\bm y_n^{\top} \bmt z_n - 2 \bm y_n^{\top}\bm{f}(\bmt x_n; \bm w)\big\}. \label{model:spo+}\tag{\code{SPO+}}
\end{align}

\ref{model:spo+} and many ILO literature are cost-driven, in the sense that the 
criteria for selecting $\mbh{w}$ is the cost function (or its regret) and does not take into account prediction accuracy. While these works argue that it is possible to obtain high-performing decisions without prediction accuracy, it is worthwhile to note that they do not in fact prove or argue that prediction accuracy is a \emph{necessary sacrifice} in the path to seeking high-performing solutions.

Instead, there may be drawbacks to not accounting for prediction accuracy. We scrutinized the numerical results in \ref{model:spo+} and note that under low mis-specification (lower than order $8$), \ref{model:spo+} is systematically outperformed by an approach that has a larger focus on predictive accuracy. We recover these findings in our numerical simulations (see \S\ref{subsec.misspecified} later). 

To better understand why forgoing prediction accuracy in \eqref{model:fully_cost-driven} or \eqref{model:fully_cost-driven_expost} is not necessarily desirable, we examine the following illustration:

\begin{illustration}\label{illust.bad_weights}

\rvedit{
We consider the knapsack problem, which involves two routes, and the decision maker must choose one to traverse. Each route $ j \in \{a,b\} $ incurs a cost $z_j$, which depends on a two-dimensional feature $\bm x $ such as weather information and current congestion, modeled as 
\begin{equation*}
   \begin{bmatrix}
z_a \\
z_b
\end{bmatrix} =  \begin{bmatrix}
\check w_{11} & \check w_{12} \\
\check w_{21} & \check w_{22}
\end{bmatrix}\begin{bmatrix}
x_a \\
x_b
\end{bmatrix} + \begin{bmatrix}
\epsilon_a \\
\epsilon_b
\end{bmatrix}
\end{equation*}
where $\epsilon_j$ for $j\in\{a,b\}$ are independently sampled noise. Suppose the true coefficient matrix is given by $\mbc W = \begin{bmatrix} 1 & 0 \\ 0 & 1.5 \end{bmatrix}$, and the decision-maker is availed $N=4$ data points: 
\begin{align*}
    \bmt{x}_1 = (1, 1)^\top, & \qquad \bmt{z}_1 = (1.6, 1.3)^\top, \\
    \bmt{x}_2 = (2, 1)^\top, & \qquad \bmt{z}_2 = (1.8, 1.7)^\top, \\
    \bmt{x}_3 = (1, 2)^\top, & \qquad \bmt{z}_3 = (0.8, 3.1)^\top, \\
    \bmt{x}_4 = (3, 2)^\top, & \qquad \bmt{z}_3 = (3.0, 2.9)^\top.
\end{align*} 
\noindent \ul{Claim}: Using SLO with OLS estimator, route $b$ is chosen if and only if $1.5 \hat x_2\le \hat x_1$, which is the ground truth; whereas using ILO under cost-driven rule \eqref{model:fully_cost-driven} leads to a degenerate solution where route $b$ is chosen if and only if $\hat x_2\le \omega \hat x_1$ for some $1\le \omega <2$ that corresponds to the chosen optimal coefficient matrix under ILO. (Please refer to Appendix~\ref{Proof of Claim 1} for the proof.)
}


\rvedit{We visualize the decision difference described in the Claim in Figure~\ref{fig:Illustration} below.}
\begin{figure}[h!]
     \centering
     \includegraphics[width=0.48\linewidth]{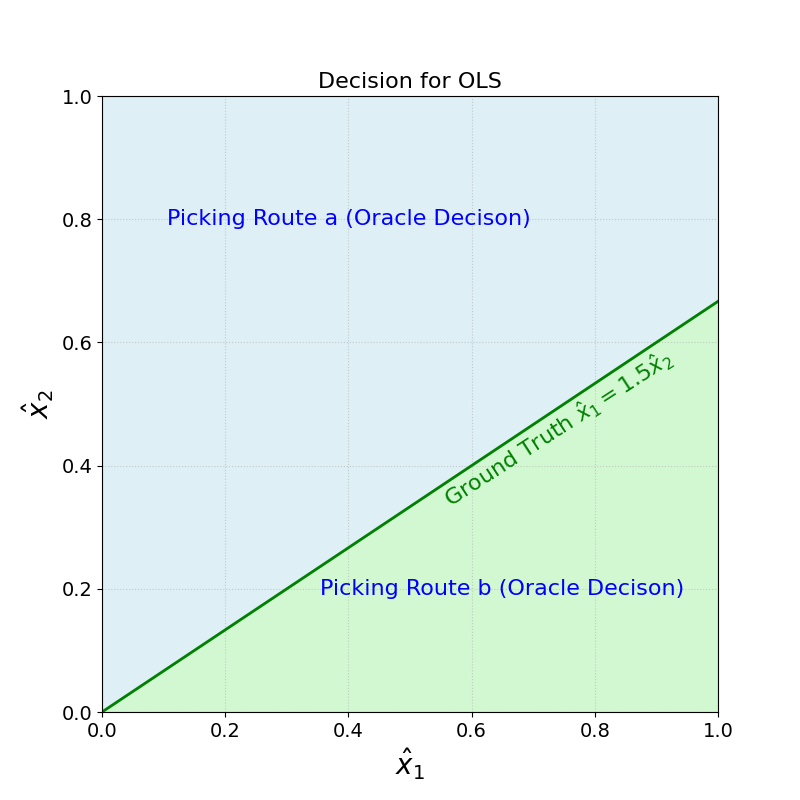}
\includegraphics[width=0.48\linewidth]{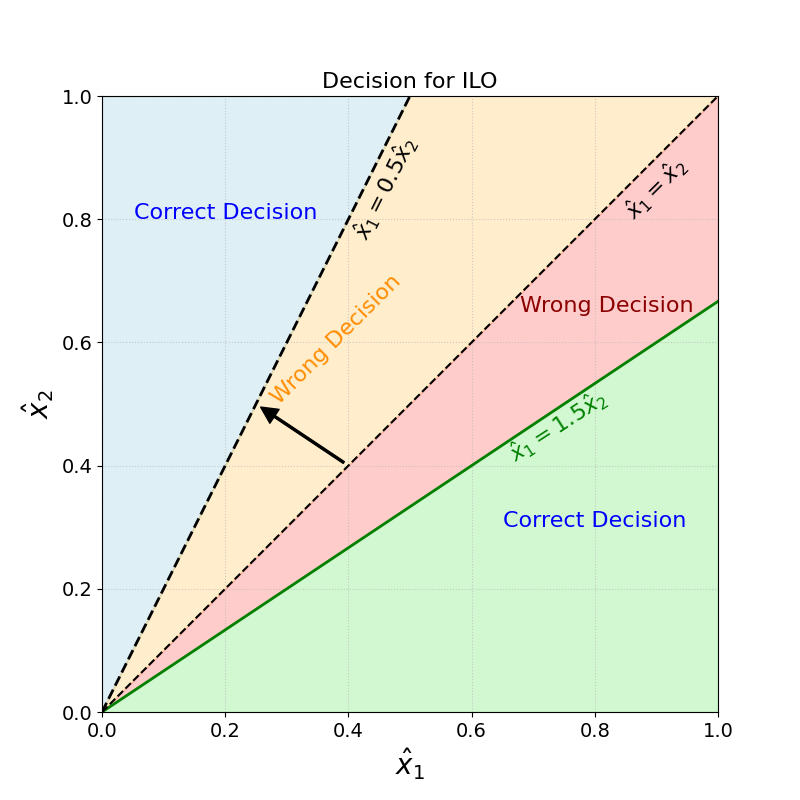}
     \caption{Visualization of decision for OLS (left) and ILO (right). The green line represents the Ground Truth ($\hat{x}_1 = 1.5\hat{x}_2$).} 
     \label{fig:Illustration}
\end{figure}
The decision region for picking route $b$ with OLS lies underneath the green Ground Truth line (see the green region in the left panel of Figure~\ref{fig:Illustration}). Notice that by construction, OLS coincides with the ground truth, thus this decision also represents the oracle's decision. As for ILO, if the \rvedit{ratio $\omega$} 
is chosen at the lowest extreme \rvedit{$\omega = 1$},
then ILO's choice of route $b$ lies underneath the line $\hat{x}_1 = \hat{x}_2$, which includes the correct decision region, plus the red region, see the right panel of Figure~\ref{fig:Illustration}. If the \rvedit{ratio $\omega$}
is chosen at the other extreme \rvedit{$\omega = 2$},
then the decision region of choosing route $b$ contains the correct decision region plus both yellow and red regions, see right panel of Figure~\ref{fig:Illustration}. In either case, both yellow and red regions are sub-optimal under the truth: ILO will pick route $b$ when in fact it is optimal to pick route $a$. 
\end{illustration}

\rvedit{To see why ILO arrives at the wrong decision region in the right pannel of Figure~\ref{fig:Illustration}, we retrace what constrains ILO to require $\omega \geq 1$ that eliminated the true decision region $\omega = 2/3$. A closer scrutiny of the proof of Claim 1 reveals that this arises from the first data point. If we had the oracle cost for this data point, we would obtain $f(\mbt x_1; \mbc w) = (1,1.5)^{\top}$ which implies that route $a$ should be chosen, however, empirically the noisy outcomes $\mbt z_1 = (1.6, 1.3)^{\top}$ flipped the decision to choose route $b$. Choosing route $b$ induces the constraint $\omega \geq 1$, whereas choosing route $a$ results in $\omega \leq 1$ which would not eliminate the true decision region. We can liken the knapsack problem with learning in Illustration 1 to a binary classification problem of choosing route $a$ or $b$ as a function of covariates $\bm x$, using a linear classifier. Unfortunately, the true labels are not available and ILO uses the empirical costs to label the data point as route $a$ if $\tilde{z}_a \leq \tilde{z}_b$ and vice versa. But $\mbt z$ is noisy, thus ILO undertakes a binary classification with \emph{noisy labels} in Illustration 1. 
More generally, any decision-focused learning scheme cannot avoid using $\mbt z$ to define its in-sample loss in some way or another, which introduces noise into the objective. 
}

\subsection{Cost function ambiguity} \label{subsec.cost fun ambiguity}

To involve the cost function in the prediction stage, each ILO model must define the cost function. However, this definition differs across the ILO literature. For example, in \ref{model:spoloss} studied by \cite{elmachtoub2023estimate}, the authors have defined the cost $\bm y^\star(\bm{f}(\bmt x, \bm w))^\top\bmt z$, using the observed outcomes $\bmt z$ in the data; in \cite{Tulabandhula_Rudin_2013machine} and \cite{Zhu_Xie_Sim_2022joint}, they use $\bm y^\top \bm{f}(\bmh x, \bm w)$---cost estimated on the testing data---in the prediction stage. In other words, these models have used different surrogates for the definition of the cost function $\bm y^\top \bm z$ for a given set of decisions $\bm y$. In reality, the true cost $\bm y^\top \bmc z$ will never be known at the point of prediction---we can have a sample or model estimates. In other words, there is ambiguity in the cost function.

In this work, we are not considering the use of testing data in the estimation of the cost function as in \cite{Tulabandhula_Rudin_2013machine} and \cite{Zhu_Xie_Sim_2022joint} since there is no ground truth established by the model. Instead, we utilize the estimated or predicted outcomes $\bmh z = \bm{f}(\bmt x, \bm w)$, which leads to what we term as the estimated cost---$\bm y^\top \bm{f}(\bmt x, \bm w)$. Solely using estimated outcomes $\bmh z$ incurs a bias-variance trade-off in the estimation of the cost function $\bm y^\top\bmc z$, as illustrated in the following proposition.

\begin{proposition}\label{prop.bv} Let $\bm x \in \mathcal{X}$ be fixed. Define $\mathcal{D}^N$ as a data set of $N$ noisy observations $\bm z$ generated from $\bm{z}|\bm{x} \sim \bm g(\bm{x}) + \bm{\epsilon}$ for some noise $\bm{\epsilon}$ with zero mean. Then using any estimator $\bm{f}(\bm x; \bmh w)$ in place of the expectation $\mathbb{E}_{\bm{z}|\bm{x}}[\bm{z}]$ incurs a bias-variance trade-off in the estimation of the true cost function $\bm{y}^{\top} \mathbb{E}_{\bm{z}|\bm{x}}[\bm{z}]$. Specifically, for all decisions $\bm y \in \mathcal{Y}$, the expectation over the sampling distribution of size-$N$ datasets $\mathcal{D}^N$ decomposes as
    \begin{align}
\!\!\!\!\!\!\!\!\!\!\!\!\!\!\!\mathbb{E}_{\mathcal{D}^N} &\left[ \Big( \bm{y}^{\top}\mathbb{E}_{\bm{z}|\bm{x}}[\bm{z}] - \bm{y}^{\top} \bm{f}(\bm{x}; \bmh w) \Big)^2
        \right] \nonumber \\
        =\; &\bigg[\bm{y}^{\top}\underset{\textrm{Bias of } \bm{f}(\bm{x}; \bmh w)}{\underbrace{\Big( \bm g(\bm{x}) - \mathbb{E}_{\mathcal{D}^N} \left[\bm{f}(\bm{x}; \bmh w) \right]\Big) } } \bigg]^2 + \underset{\textrm{Variance of } \bm{y}^{\top} \bm{f}(\bm{x}; \bmh w)}{\underbrace{\mathbb{E}_{\mathcal{D}^N} \left[ \left(\bm{y}^{\top} \bm{f}(\bm{x}; \bmh w)\right)^2 \right] -  \left(\mathbb{E}_{\mathcal{D}^N}\left[\bm{y}^{\top} \bm{f}(\bm{x}; \bmh w)\right]\right)^2}}. \label{eq.biasvariance}
    \end{align}
\end{proposition}
The proposition, being reminiscent of the bias-variance trade-off in regularization, indicates that a sharper estimate of the cost function is possible under a biased estimator. Therefore, the most accurate predictor for the cost function does not necessarily occur on the predictor $f(\bm x, \bmh w)$ that maximizes prediction accuracy. Instead, by adding bias, we can further improve the precision of the estimate, such as adding the empirical cost generated over the training data---$\bm y^\top\bmt z$.  Hence, we propose using a combination of the empirical outcomes $\bmt z$ and the estimated outcomes $f(\bmt x, \bmh w)$ to define the cost function, which is elaborated in the following subsection. 

\subsection{A bi-objective framework for ILO}\label{subsec.DDR}
To overcome the potential drawbacks in the ILO framework, we propose a blended bi-objective framework that balances prediction accuracy and cost minimization, and term it as \textit{decision-driven regularization} (DDR). Unlike the SLO literature, it accounts for the cost structure in the prediction stage, and unlike the ILO literature, it retains a closeness
to true outcomes:
\begin{equation}\label{model:ddr_expectation_form}\tag{\code{DDR}}
\bmh{w}^{\rm DDR} = \argmin_{\bm w}~ \Big\{
	L(\bm w) - \lambda \mathbb{E}_{\bm x, \bm z} \Big[\min\limits_{\bm y \in \cy} \big\{\mu\, \bm y^{\top}\bm z + (1 - \mu)\, \bm y^{\top}\bm{f}(\bm x, \bm{w})\big\} \Big] \Big\},
\end{equation}
for some parameter $\mu < 1$, regularization parameter $\lambda \geq 0$, and a loss function $L(\bm w)$ that is typically accuracy-based but may also take a decision-focused form. 

\medskip
\noindent \ul{Balancing prediction accuracy and cost minimization}: 
To simultaneously consider prediction loss and cost, we include two terms in the objective of \ref{model:ddr_expectation_form}---one that is specified by the loss function $L(\bm w)$, which typically controls prediction accuracy, and one that is a surrogate for the cost function. This balance is moderated by the parameter $\lambda$. An important consideration here is the negative sign that precedes the regularization parameter $\lambda$. We mentioned that it is important for prediction accuracy to limit the flexibility of the weights in defining the cost function. With the negative sign, we ensure that the relationship between loss and cost remains antagonistic \citep[or `pessimistic' in][]{Tulabandhula_Rudin_2013machine}. This interpretation is rooted in robust optimization. Later in \S\ref{section::DDR}, we prove that \ref{model:ddr_expectation_form} is equivalent to a robust optimization problem with an uncertainty set characterizing the uncertainty arising from the learning of weights. There, it is clear that an antagonistic relationship between cost and loss is required for the equivalence to hold. While this does lead to the impression that \ref{model:ddr_expectation_form} seeks higher costs, this is controlled by $L(\bm w)$, since while costs can be unbounded if $\bm w \rightarrow \infty$, so does $L(\bm w) \rightarrow \infty$.

\medskip
\noindent \ul{Balancing between empirical and estimated cost definitions}: 
To ensure greater consistency in the cost function, we can anchor it using the observed outcomes $\bmt z$ in the data set. We propose the following surrogate for some constant $\gamma>0$, 
\begin{align}\label{eq.valueconst}
	\min\limits_{\bm y \in \cy} & \  \bm y^{\top}\bm{f}(\bmt x, \bm w) \\
	\mbox{s.t.} & \ \left| \bm y^{\top}\mbt{z} -\bm y^{\top}\bm{f}(\mbt x, \bm w)\right|  \leq \gamma. \nonumber 
\end{align}
The constraint restricts the choice of $\bm w$ such that the estimated cost is close to the empirical cost that is observed from the data.
\begin{proposition}\label{prop.valueequiv}
    The formulation (\ref{eq.valueconst}) has the following Langragian relaxation:
    \begin{align*}
		\nu_{\mu}(\bm w):= \min\limits_{\bm y \in \cy}& \ \big\{\mu\, \bm y^{\top}\mbt{z} + (1-\mu)\, \bm y^{\top}\bm{f}(\mbt{x}, \bm{w})\big\},
    \end{align*}
    for some $\mu < 1$. 
\end{proposition}

We term this the \emph{valuation function}, $\nu_{\mu}(\cdot):\, \mathbb{R}^q \rightarrow \mathbb{R}$, mapping weights $\bm{w}$ to the space of objective values in the decision problem.
Proposition \ref{prop.valueequiv} can be thought of as a convex combination of the empirical and estimated cost functions if $\mu\in[0,1]$. 
This leans towards the estimated cost if $\mu$ is small and conversely towards the empirical cost if $\mu$ is close to $1$. It is possible for $\mu < 0$. However, $\mu$ needs to be less than $1$ (see Remark \ref{rem.mu=1} in the proof of Theorem~\ref{thm.rddr}). 

\subsection{Solving the \ref{model:ddr_empirical} model}\label{subsec.solve}

To solve \ref{model:ddr_expectation_form}, consider its sample average approximation using training data. We abuse notations to still refer to this as $\code{DDR}$ hereafter.
\begin{equation}\label{model:ddr_empirical}\tag{$\code{DDR}$}
    \bmh{w}^{\rm DDR} = \argmin\limits_{\bm{w}} \,\, L(\bm{w}) - \lambda \frac{1}{N} \sum_{n \in [N]} \min_{\bm y_n \in \cy} \Big\{ \mu\, \bm{y}_n^{\top} \mbt{z}_n + (1-\mu)\, \bm{y}_n^{\top}\bm{f}(\mbt x_n;\bm{w})\Big\}.
\end{equation}

The inner minimum is linear in $\bm y_n$ and can be easily computed if the constraint set $\cy$ has good structure. For example, if it is a polyhedron, then we obtain the following result. 

\begin{proposition}\label{prop.compute}
If constraint set $\cy$ is a polyhedron represented by $\{\mb{A}\bm{y} \leq \mb{b}, \bm{y} \geq \mb{0}\}$, {\rm \ref{model:ddr_empirical}} has the following reformulation: 
\begin{equation}\label{eq.solution_style}
\begin{array}{rll}
    \argmin\limits_{\bm w; \bm \kappa_n \geq 0\ \forall n \in [N]}~ &\displaystyle L(\bm{w}) + \lambda \frac{1}{N} \sum_{n \in [N]}\bm \kappa_n^{\top} \mb{b}& \\
    \mbox{s.t.}~ & \mb{A}^{\top}\bm \kappa_n \geq - \mu\mbt{z}_n - (1-\mu)\bm{f}(\mbt x_n;\bm{w}) & \quad\forall n\in [N], 
\end{array}    
\end{equation}
\end{proposition}

When directly solving for \ref{model:ddr_empirical}, it would be necessary to seek the best regularization parameter $\lambda$. This can be achieved through cross-validation, where a portion of the data is set aside as a validation set, or by performing $k$-fold cross-validation. Additionally, we can first use the OLS solution to find the rough ratio between the loss function and the cost function, denoted as $\hat{\lambda}$. It is then possible to initiate the search for $\lambda$ within the neighborhood of $\hat{\lambda}$. In particular, we can follow the procedure below:
\begin{enumerate}
    \item Estimate the prediction cost by computing $\displaystyle C^{p} = \min_{\bm w} L(\bm w)$, and let $\displaystyle\bmh w = \argmin_{\bm w} L(\bm w)$ be the estimated model parameters.
    \item Given $\bmh w$ and $\mu$, compute the optimization cost as follows:
    $$C^{o} = \frac{1}{N}\sum_{n \in [N]}\min_{\bm y \in \mathcal{Y} }\big[\mu \bm y_{n}^{\top}\bmt z_{n} + (1-\mu)\bm y_{n}^{\top}\bm{f}(\bmt x_{n}, \bmh w)\big].$$
    \item Set the candidate value $\hat{\lambda} = \frac{C^{p}}{C^{o}}$, and then choose $\lambda$ from a neighborhood around $\hat{\lambda}$.
\end{enumerate} 

We remark here that the computational difficulty of Problem~\eqref{eq.solution_style} is governed by the tractability of the predictor $\bm{f}(\bm x; \bm w)$ with respect to the weights $\bm w$. If  
$\bm{f}(\bm x; \bm w)$ is concave (or affine) in $\bm w$, the right–hand side of the constraint is a convex function, so the feasible set remains convex, which can be solved by standard first‐ or second‐order convex solvers.

In many cases, such as with most multilayer neural networks, the resultant problem is non-convex. In this case, one may resort to (\textit{i}) \emph{gradient–based heuristics} that differentiate through the \ref{model:ddr_empirical} objective and apply stochastic optimizers such as ADAM \citep{Kingma_Ba_2015_Adam}; (\textit{ii}) \emph{implicit-layer differentiation} methods that treat the inner maximization (or its dual) as a differentiable layer, \textit{e.g.}, OptNet, cvxpylayers, Alt-Diff, and related schemes \citep{Amos_Kolter_2017optnet,Agrawal_et_al_2019differentiable,Sun_et_al_2023maximum}; or (\textit{iii}) \emph{convex surrogates} that replace the original network with a tractable convex approximation such as input‐convex neural networks or certified outer-polytope relaxations \citep{Amos_Kolter_2017_ICNN,Wong_Kolter_2018_Polytope}. These approaches restore scalability at the cost of global optimality guarantees, yet have proved effective in large-scale decision-focused learning applications. In some special cases, as we illustrate in Appendix \ref{append.tree} for a tree-based learner, there are natural ways to adapt DDR for the learner.

\section{Relationship between regularized learning, robust optimization and regret minimization} \label{section::DDR}

In this section, we establish connections between \ref{model:ddr_empirical}, posed from the learning perspective with a cost-based regularization, and two alternative viewpoints: (\textit{i}) a robust optimization formulation and (\textit{ii}) a worst-case regret minimization approach.  

In \cite{Xu_Caramanis_Mannor_2010robust}, regularization on a prediction problem is shown to be dual to a robust decision problem with a defined uncertainty set controlling the level of robustness to variations in data. 
This is generalized and echoed in works such as 
\cite{Bertsimas_Copenhaver_2018characterization} and \cite{Blanchet_Kang_Karthyek_2019robust}. 
Here, we propose a robust optimization formulation that leads to the same set of solutions as what would be obtained from \ref{model:ddr_empirical}. 

\begin{definition}
The Robust DDR model is defined as
\begin{align}
\min_{\bm y_n\in\cy,\, n \in [N]}\, \max_{\bm{w} \in \cu(\rho)} & \ \frac{1}{N} \sum_{n \in [N]} \Big[\mu\, \bm y_n^{\top} \bmt z_n + (1-\mu)\, \bm{y}_n^{\top}\bm{f}(\bmt x_n; \bm w) \Big], \label{model:robustness_ddr}\tag{\code{Robust-DDR}}
\end{align}
with $\cu(\rho) :=\big\{\bm w: L(\bm{w}) - L(\bmt w) \leq \rho\big\}$ and $\displaystyle\bmt w = \argmin_{\bm w}\; L(\bm w)$.
\end{definition}

In~\ref{model:robustness_ddr}, the decision-maker attempts to solve for the best decisions $\bm y_n$, by balancing prediction accuracy and cost minimization via an adversarial approach to worst-case weights under the uncertainty set $\mathcal{U}(\rho)$. The uncertainty set is crafted under the geometry of the loss function, which has specific statistical interpretations (see Illustration \ref{illust.neypear} in Appendix \ref{app.loss function}).

\begin{theorem}[Robustness]\label{thm.rddr} 
Assume that $\cy$ is convex, closed and compact, and $\bm{f}(\mb{x}; \bm{w})$ is concave in $\bm{w}$ for all $\mb{x} \in \cx$. Given any $\lambda > 0$, there exists some $\rho>0$ such that the worst-case weights $\bm{w}$ achieved for optimal decisions $\bm{y}_n,\, n \in [N]$ in~{\rm \ref{model:robustness_ddr}} coincides with the solution of~{\rm \ref{model:ddr_empirical}}, when $\mu \in [0,1)$. 
\end{theorem}

Theorem~\ref{thm.rddr} explains the robust motivation of the negative sign before $\lambda$ in~\ref{model:ddr_empirical}. If the sign before $\lambda$ were positive, this corresponds to assuming $L(\bm{w}) - L(\bmt w) \geq \rho$ in the uncertainty set, that is, the prediction accuracy must be poorer than a given level, which may run contrary to the spirit of learning. Also notice that, in the literature, the construction of the dual uncertainty set (\emph{cis} the regularization) does not often have an intuitive form \citep[\emph{e.g.}, see][]{Gao_Chen_Kleywegt_2017_wasserstein}. In our construction, the uncertainty sets and the regularizers are both very intuitive, where the former relates to the loss function and the latter to the cost function. 

The dual form in Theorem \ref{thm.rddr} can also lead to a separate solution methodology, where the robust counterpart is taken over the loss function, as opposed to the cost function in Proposition \ref{prop.compute}. This can be useful if the robust counterpart of the loss function can be easily computed. The reader is referred to Appendix \ref{append.pseudo} for more details.

The assumption that $\bm{f}(\mb{x}; \bm{w})$ is concave in $\bm{w}$ for all $\mb{x} \in \mathcal{X}$ holds for linear and affine predictors, which are widely studied in the literature \citep[\textit{e.g.},][]{Tulabandhula_Rudin_2013machine, Liu_He_Shen_2021time, Elmachtoub_Grigas_2022smart, HoNguyen_KilincKarzan_2022risk}. Beyond these cases, concavity of $\bm{f}(\mb{x}; \bm{w})$ in $\bm{w}$ is a structural property that is not generically satisfied by common nonlinear predictive models. However, the role of this assumption in our paper is limited to the theoretical development establishing the robust interpretation of DDR. The formulation, tractability, and practical applicability of DDR do \emph{not} rely on this assumption.

Two models in the literature are closely related to \ref{model:robustness_ddr}.
\cite{Tulabandhula_Rudin_2013machine} propose the simultaneous process (SP) model:
\begin{equation}\label{model:sp}\tag{\code{SP}}
	\bmh{w}^{\rm SP} = \argmin_{\bm w}\ L(\bm w) - \lambda \min\limits_{\bm y \in \cy} \bm y^{\top}\bm{f}(\bmh x; \bm w),
\end{equation}
where $L(\bm w)$ is the prediction loss over the training dataset, and $\lambda \in \bbr$ is the regularization parameter. 
When $\lambda > 0$, they term this `pessimistic'; and $\lambda < 0$ `optimistic'. This model is related to \cite{Zhu_Xie_Sim_2022joint}'s joint estimation and robustness optimization (JERO) model, which maximizes the
robustness on the mis-estimation of the prediction loss $L(\bm w)$, while meeting a specified mean (estimated) cost target $\tau$:
\begin{equation}
\begin{array}{rll}
	\max\limits_{\rho > 0;\,\bm y\in\cy} &\quad \rho \nonumber\\
	\mbox{s.t.} &\displaystyle \quad \bm y^{\top}\bm{f}(\mbh x; \bm w) \leq \tau, ~ &\quad\forall \bm{w} \in \cu(\rho) := \{\bm{w}:L(\bm{w}) - L(\bmt w) \leq \rho \}, \label{model:jero}\tag{\code{JERO}}
\end{array}
\end{equation}
where $\displaystyle\bmt w \in \argmin_{\bm w} L(\bm w)$. 
These two models turn out to be dual to each other, in the sense that for every 
$\lambda>0$ in~\ref{model:sp}, there exists a corresponding $\tau$ in~\ref{model:jero}, such that the minimizer of $\bm w$ in ~\ref{model:sp} and the worst-case $\bm w$ in \ref{model:jero} coincide, with the same decision $\bm y$ up to degeneracy. Henceforth, for ease of reference, we use \ref{model:jero} to refer to both models.

The key difference between \ref{model:jero} and the SLO and ILO literature is that \ref{model:jero} solves $\bm w$ as a function of $\bmh x$, the new \emph{testing} data. Hence, it does not lead to any single true relationship $g(\bm x) \approx\bm{f}(\bm x; \bm w)$ for some $\bm w$, which instead changes for each $\mbh{x}$. We can obtain a model, termed JERO-Like, that is consistent with the SLO and ILO literature by changing $\bmh x$ to $\bmt x$, the training data.
\begin{equation}
\begin{array}{rll}
	\max\limits_{\rho > 0;\,\bm y\in\cy} &\quad \rho \nonumber\\
	\mbox{s.t.} &\displaystyle \quad \frac{1}{N}\sum\limits_{n\in [N]}\bm y_n^{\top}\bm{f}(\mbt x_n; \bm w) \leq \tau, ~ &\quad\forall \bm{w} \in \cu(\rho) := \{\bm{w}:L(\bm{w}) - L(\bmt w) \leq \rho \}; \label{model:jerolike}\tag{\code{JERO-Like}}
\end{array}
\end{equation}

\begin{proposition}[\ref{model:jerolike} as a \ref{model:ddr_empirical}]\label{prop.jeroequiv}
Under the same assumptions as in Theorem~\ref{thm.rddr}, define
    \begin{equation*}
        \mathcal{T}:= \left\{ \tau \in \mathbb{R} \,\middle|\, \exists\bm{y}_n \in \cy,\, n \in [N] : \displaystyle\frac{1}{N}\sum_{n \in [N]} \bm{y}_n^{\top}\bm{f}(\bmt x_n; \bmt w) \leq \tau \right\}.
    \end{equation*}
    When $\mu = 0$, for all $\tau \in \mathcal{T}$,
    there exists some $\lambda := \lambda(\tau) \geq 0$ for which the solutions of \ref{model:ddr_empirical} and the worst-case weights attained under optimal uncertainty set size $\rho^\star$ in \ref{model:jerolike} coincide. 
\end{proposition}

\subsection{Regret interpretation} \label{section::regret interpretation}

In \cite{Elmachtoub_Grigas_2022smart}, the authors motivate their model by deriving it from a regret-based formulation. Similarly, we relate~\ref{model:ddr_empirical} to a worst-case regret formulation.
Traditionally, regret minimization and robust optimization are viewed as two different perspectives leading to two different model formulations, and papers have sought to analyze the differences between optimal policies obtained from the two models (\emph{e.g.}, \citealt{Perakis_Roels_2008regret
}).  In \cite{Poursoltani_Delage_2022adjustable}, the authors describe a worst-case regret minimization formulation that can be cast as a two-stage robust optimization problem. Nonetheless, it is difficult to specifically understand the nature of this relationship due to the lifting technique. Our work here is inspired by theirs; and we attempt to relate regret minimization to robust optimization under the data-driven setting, in a manner that is instructive about their relationship.

We consider the following sequence of events:
{\itshape
\begin{enumerate}[label = \textbf{Step \arabic*.}, align=left] 
	\item The training dataset $\mathcal{D}^N$ is availed to the decision-maker.
	
	\item A particular choice of weights $\bm w$ is committed to, inducing decisions $\bm y^\star(\bmh z)$.
	
	\item The oracle then reveals what the true outcomes $\bmc z$ really were, and evaluates the quality of the decisions made by the decision-maker $\bm y^\star(\bmh z)^{\top} \bmc z$ measured under these true outcomes, against the lowest possible cost that could be attained had the true outcomes been known a prior, termed the oracle costs, $\bm y^\star(\bmc z)^{\top}\bmc z$.
\end{enumerate}
}

\begin{definition}
The worst-case Regret-DDR problem is defined as
\begin{align}
    \argmin_{\bm w} \max_{\bmc z \in \cz(\bm w)} \bigg\{\dfrac{1}{N}\sum_{n\in[N]} \Big[ \bm y^\star(\bmh z_n)^{\top} \bmc z_n - \min_{\bm y_n \in \cy}~\bm y_n^{\top} \bmc z_n \Big] \bigg\}, \label{model::regret_ddr}\tag{\code{Regret-DDR}}
\end{align}
where $\cz(\bm w)$ is defined as 
\begin{equation*}
    \cz(\bm w) = \left\{ \bm z := \big\{\bm z_n \in \bbr^s \big\}_{ n\in [N]} ~\left|~
    \begin{aligned}
        &\bigg|\frac{1}{N} \sum_{n \in [N]} \Big[\ell( \bm z_n; \bmt z_n) - \ell(\bm z_n; \bmh z_n) \Big]\bigg| \leq t\\
        &\forall~ \bm y_n\in \mathcal{Y},\,\, n \in [N]:\\
        &\qquad \Big|\bm y_n^{\top} \bm z_n - \bm y_n^{\top}\bmt z_n\Big| \leq \phi \\
        &\qquad \Big|\bm y_n^{\top}\bm z_n -  \bm y_n^{\top}\bmh z_n\Big| \leq \psi \\
        &\qquad \Big|\bm y^\star(\bm z_n)^{\top}\bmt z_n -  \bm y^\star(\bmt z_n)^{\top} \bmt z_n\Big| \leq \eta 
    \end{aligned}\right. 
    \right\},
\end{equation*}
for some known constants $\phi, \psi, \eta, t \in \bbr_+$. 
\end{definition}

Ambiguity set $\mathcal{Z}(\bm w)$ guards against the fact that the true outcomes $\bmc z_n$ are not known. The first condition in $\mathcal{Z}(\bm w)$ 
ensures predictive error is small. The next two conditions require alignment between the cost defined by empirical data, the estimated model, and the oracle for every data point. The final condition assumes the sample data $\bmt z_n$ is a good proxy for decision-making, considering only oracle manifestations $\bm z_n$ that lead to decisions $\bm y^\star(\bmc z_n)$ evaluated favorably under empirical costs. Hence, the size of $\mathcal{Z}(\bm w)$ reflects the quality of the training data available for learning.

In~\ref{model:spoloss}, the authors simply define the regret under the empirical costs. In~\ref{model::regret_ddr}, the decision-maker is cognizant that they do not know what are the costs that might emerge or be assigned by the oracle, and hence plans for the worst possible costs that might manifest at the end of the day, subject to some limitations as described in $\mathcal{Z}(\bm w)$. Hence, $\mathcal{Z}(\bm w)$ is the uncertainty set of all possible oracle outcomes that the decision-maker is willing to consider. In other words, as long as the true oracle outcomes lie within $\mathcal{Z}(\bm w)$, then the solution of~\ref{model::regret_ddr} provides an upper bound for the true regret. 

\begin{theorem}[Regret]\label{thm.regret ddr}
Denote $V(\bm w)$ as the objective value of \ref{model:ddr_empirical}, 
\[
V(\bm w) = L(\bm w) - \lambda \dfrac{1}{N}\sum_{n \in [N]}\min\limits_{\bm y_n \in \cy}\Big\{\mu\, \bm y_n^{\top}\mbt{z}_n + (1-\mu)\, \bm y_n^{\top}\bm{f}(\mbt{x}_n, \bm{w}))\Big\}.
\]
Suppose that $\ell (\bm u; \bm v)$ satisfies the triangular inequality, {\em i.e.}, $\ell (\bm u; \bm v) \le \ell (\bm u; \bm r) + \ell (\bm r; \bm v) ~ \forall \bm u, \bm v, \bm r \in \bbr^s$;
Then for any dataset $\mathcal{D}^N$, there exists some constant $a \in \mathbb{R}$, that may depend on $\mathcal{D}^N$ but does not depend on $\bm w$, such that, for every $\bm w \in \bbr^{q}$,
\[
\max_{\bmc z \in \cz(\bm w)} \bigg\{\dfrac{1}{N}\sum_{n\in[N]}\Big[\bm y^\star(\bmh z_n)^{\top} \bmc z_n - \min_{\bm y_n \in \cy}\bm y_n^{\top} \bmc z_n \Big] \bigg\} \leq \frac{1}{\lambda}V(\bm w) + a.
\]
In particular, the worst-case regret for choosing the weights $\bmh{w}^{\rm DDR}$ is bounded and, 
\[
\max_{\bmc z \in \cz(\bmh w^{\rm DDR})} \bigg\{\dfrac{1}{N}\sum_{n\in[N]} \Big[\bm y^\star(\bmh z_n )^{\top} \bmc z_n - \min_{\bm y_n \in \cy} \bm y_n^{\top} \bmc z_n\Big] \bigg\} \leq \frac{1}{\lambda}V(\bmh{w}^{\rm DDR}) + a.
\]
\end{theorem}

Requiring triangle inequality is reasonable---$\ell(\cdot,\cdot)$ is a closeness measure, and would usually be at least a pseudo-metric. Theorem \ref{thm.regret ddr} connects regret minimization to the original regularized problem, implying that~\ref{model:ddr_empirical} is its canonical approximation for the cost-ambiguous regret-minimization problem under general fidelity measure.

How~\ref{model:ddr_empirical} is obtained as an approximation of~\ref{model::regret_ddr} is akin to how \ref{model:spo+} was obtained from the regret minimization problem in \cite{Elmachtoub_Grigas_2022smart}. In fact, the very next result states that under such a perspective, \ref{model:spo+} is a special case of our model. 

\begin{proposition}[\ref{model:spo+} as a~\ref{model:ddr_empirical}]\label{prop.spoequiv}
	The solution of~{\rm \ref{model:spo+}} coincides with that of ~{\rm \ref{model:ddr_empirical}} for $\lambda =1$, $\mu = -1$, and the loss function is chosen as $\displaystyle L(\bm{w}) = 2\frac{1}{N}\sum\limits_{n\in[N]} \bm y^\star(\bmt z_n)^{\top}\bm{f}(\mbt{x}_n; \bm{w})$. 
\end{proposition}

\ref{model:spo+} is a very specific example of \ref{model:ddr_empirical}, where not just the loss function $L(\bm w)$ is chosen to be decision-focused in a particular way, but the actual Lagrange multiplier $\lambda$ is also specified. 

\section{A Numerical Illustration on the Shortest Path Problem}
In this section, we follow the setting of \cite{Elmachtoub_Grigas_2022smart} to compare the performance of our decision-driven regularization approach against various benchmarks on the shortest path problem defined over an $G \times G$ grid network. Let $\mathcal{V}$ and $\mathcal{E}$ denote the sets of nodes and arcs, respectively. For each arc $(i,j) \in \mathcal{E}$, let $z_{ij}$ represent the cost associated with traversing that arc. Let $s$ and $t$ denote the designated source and sink nodes, respectively, which are located at the opposite corners of the grid. The shortest path problem can be formulated as the following optimization problem: 
\begin{equation*}
\begin{array}{rll}
    \min~ &\displaystyle \sum_{(i,j) \in \mathcal{E} } y_{ij}z_{ij}  \\
    \mbox{s.t.}~ 
    & \displaystyle \sum_{j: (j,i) \in \mathcal{E}} y_{ji} - \sum_{j:(i,j) \in \mathcal{E}} y_{ij} = 
    \begin{cases}
         -1 & \mbox{ if } i = s \\
         1 & \mbox{ if } i = t\\
         0 & \mbox{ otherwise }
    \end{cases}
     &~\forall i \in \mathcal{V}\\
    & \displaystyle y_{ij} \in \{0,1\} &~\forall (i,j) \in \mathcal{E}.
\end{array}
\end{equation*} 

\smallskip
\noindent\ul{Data generation}. In this study, we generate synthetic datasets that closely mirror the procedure in \cite{Elmachtoub_Grigas_2022smart}. We assume that the true relationship is 
$$
\tilde{z}_{ij} = \Big[\frac{1}{\sqrt{p}}\big(\bmt x^\top \bmc w_{ij} + 3\big)^{\beta} + 1\Big]\cdot \tilde{\epsilon}_{ij},
$$
where $\bm{x} \in \mathbb{R}^{p}$ represents the features, $\bmc{w}_{ij} \in \mathbb{R}^{p}$ are the true (but unknown) weights associated with arc $(i,j)$, $\tilde{\epsilon}_{ij}$ are multiplicative noise terms independently sampled from a uniform distribution on the interval $[1 - \bar{\epsilon}, 1 + \bar{\epsilon}]$ for some noise level $\bar{\epsilon} \geq 0$, and $\beta$ is the misspecification parameter (when $\beta = 1$, there is no misspecification). Note that $p$ in the root is the same $p$ that denotes the dimension of features $\bm x$. The features $\bmt{x}$ are drawn from a multivariate normal distribution with independent standard normal components, and each entry of $\bmc{w}_{ij}$ is independently sampled from a Bernoulli distribution with success probability $0.5$. Let $\mathbf W$ denote the collection of all weights $(\bm w_{ij})_{(i,j) \in \mathcal{V}}$ and training dataset is denoted $\mathcal{D}^{N} = \{(\bmt{x}_n, \bmt{z}_n)\}_{n=1}^{N}$. Similarly, we simulate $\mathcal{D}^{M} = \{(\bmt{x}_m, \bmt{z}_m)\}_{m=1}^{M}$ as testing dataset.

\smallskip
\noindent\ul{Benchmarks}. We consider a range of SLO and ILO benchmarks. For SLO, we use the OLS, random forest (RF) and boosting (implemented by XGBoost) learners. We implement the RF model using the \texttt{scikit-learn} library and the boosting model using the \texttt{XGBoost} package. For the RF model, we set the number of trees to 100, chosen based on empirical performance metrics. For the XGBoost model, we calibrate the optimal number of boosting rounds to 2. For ILO, we consider \ref{model:spo+}, Perturbation Gradient (PG) proposed by \cite{huang2024decision}, and Learning to Rank (LTR) introduced by \cite{mandi2022decision}. 

\smallskip
\noindent\ul{Optimization Models}. 
Given the historical dataset $\mathcal{D}^{N}$, the OLS approach can be readily implemented to obtain parameter estimates. For \ref{model:spo+}, PG and LTR, we utilize the publicly available and well-documented \texttt{PyEPO} package developed by \cite{tang2024pyepo}.

The \ref{model:ddr_empirical} model for the shortest path problem is formulated as follows:

\begin{equation*}
\begin{array}{rll}
    \min\limits_{\mathbf W}~ &\displaystyle L(\mathbf W) + \frac{\lambda}{N}\sum_{n \in [N]}\max\limits_{y_{ij}\,:\, \forall (i,j) \in \mathcal{E}} \Big\{\sum_{(i,j) \in \mathcal{E} } y_{ij}^{n}\big(-\mu \tilde{z}^{n}_{ij} - & (1-\mu) f_{ij}(\bmt{x}^{n},\bm w_{ij}) \big)\Big\}  \\
    \mbox{s.t.}~ 
    & \displaystyle \sum_{j: (j,i) \in \mathcal{E}} y^n_{ji} - \sum_{j:(i,j) \in \mathcal{E}} y^n_{ij} = \left\{ 
    \begin{array}{ll}
         -1 & \mbox{ if } i = s \\
         1 & \mbox{ if } i = t\\
         0 & \mbox{ otherwise }
    \end{array}
    \right. &~\forall i \in \mathcal{V}, n \in [N]\\
    & \displaystyle y^n_{ij} \in \{0,1\} &~\forall (i,j) \in \mathcal{E}, n \in [N]
    \end{array}
\end{equation*}
Here, we use the mean square error for the loss function $L(\mathbf W)$. The model is reformulated using Proposition~\ref{prop.compute}, and solved by standard commercial optimization solvers such as Gurobi.

\smallskip
\noindent\ul{Measuring performance}. 
Given an estimate $(\bm{\hat{w}}, \bm{\hat{w}}_0)$ of the weights, we compute the predicted costs for each observed features $\bmt{x}_m$ in the test dataset as $\bm{\hat{z}}_m = \bm{\hat{w}}^\top \bmt{x}_m + \bm{\hat{w}}_0$. Based on the predicted cost $\bm{\hat{z}}_m$, we obtain the optimal solutions $\bm{\hat{y}}_m$ by solving the shortest path problem.
Next, we evaluate the true cost incurred by adopting the solutions $\bm{\hat{y}}_m$, using the oracle costs defined as $\bmc z_{m} = \big(\frac{1}{\sqrt{p}}(\bmt x_{m}^\top \bmc w_{ij} + 3\big)^{\beta} + 1$.
Finally, the average cost over the test dataset is given by
$$ 
P(\bmh w,\bmh w_{0}) = \frac{1}{M}\sum_{m \in [M]}\bmh y_{m}^\top \bmc z_{m}.
$$
Similarly, the oracle's decisions $\bmc y_{m}$ can be obtained by solving the shortest path problem based on true average cost $\bmc z_{m}$. 
We denote oracle's performance as $P^\star = \frac{1}{M}\sum\limits_{m \in [M]} \mbc{y}_m^{\top}\mbc{z}_m $. By definition, the oracle's decisions achieve the lowest possible costs.
Therefore, the regret associated with the estimate $(\bmh w,\bmh w_{0})$ is denoted as 
$$
R(\mbh{w}, \mbh{w}_0) := P(\mbh{w}, \mbh{w}_0) - P^\star.
$$

When comparing any two models $\left(\mbh{w}^A, \mbh{w}^A_0\right)$ and $\left(\mbh{w}^B, \mbh{w}^B_0\right)$, we use the following metrics:
\begin{enumerate}
    \item \ul{Regret reduction}: The regret reduction of Model A over Model B is defined as how much lesser regret (and thus improvement) Model A incurs, as a proportion of the regret incurred by Model B,
    \begin{equation*}
    \Delta R(A,B) := \frac{R\left(\mbh{w}^B, \mbh{w}^B_0\right) - R\left(\mbh{w}^A, \mbh{w}^A_0\right) }{R\left(\mbh{w}^B, \mbh{w}^B_0\right)} \times 100\%.
    \end{equation*}
    
    \noindent Equivalently, $\Delta R(A,B) := \left[P\left(\mbh{w}^B, \mbh{w}^B_0\right) - P\left(\mbh{w}^A, \mbh{w}^A_0\right)\right]\Big/(P\left(\mbh{w}^B, \mbh{w}^B_0\right) - P^\star) \times 100\%$. Note that if $\Delta R(A,B) > 0$, then Model A incurs smaller regret than Model B.
    \item \ul{Head-to-head}: Let the proportion of test data points $\mbh{x}_m$ where the optimal decisions $\mbh{y}_m$ disagree across the two models, be the discordance $D(A,B) = \Big[\sum\limits_{m \in [M]}\mathbbm{1}\left\{\mbh{y}_m^A \neq \mbh{y}_m^B\right\} \Big]\Big/M$. We define the head-to-head ratio between two models as the proportion of test data points $\mbh{x}_m$ where the first model achieves a lower cost, $H(A,B) = \dfrac{1}{D(A,B)}\Big[\sum\limits_{m \in [M]}\mathbbm{1} \left\{\mbc{z}_m^{\top} \mbh{y}_m^A <\mbc{z}_m^{\top} \mbh{y}_m^B\right\}\Big]$, out of all data points where the models disagree. We define $H(A,B) = 0.5$ if $D(A,B) = 0$.
\end{enumerate}

\subsection{Baseline Setting---No misspecification} \label{section:shortest-baseline}

We first consider a $3 \times 3$ grid network with a training dataset of size $N = 100$, a testing dataset of size $M = 1000$, and a feature dimension of $p = 5$. We also assume no model misspecification by setting $\beta = 1.0$, and the noise level is fixed at $\bar{\epsilon} = 0.5$. Further results demonstrating the robustness of our findings to different parameters (\textit{i.e.}, $N,\; p,\; \beta$, and $\bar{\epsilon}$) and different network sizes are presented in Appendix~\ref{append.more_simu}. 

\smallskip

\noindent\ul{Computational efficiency}.
We first evaluate the computational efficiency of the proposed DDR approach by varying the grid size $G \in \{2,3,4,5\}$ across five different hyperparameter configurations for $(\mu, \lambda)$. Table~\ref{table:solution-time} presents details on the average computational time required to solve each model instance. The results demonstrate that the DDR model is computationally efficient. While the solution time increases monotonically with the grid size, the model is insensitive to variations in $\mu$ and $\lambda$, as these hyperparameters have negligible impact on computational performance.

\begin{table}[ht]
\centering
\caption{Average solution time (in seconds) required to compute DDR model.}
\label{table:solution-time}
\begin{tabular}{c|ccccc}
\hline
\multirow{2}{*}{Grid Size} & \multicolumn{5}{c}{$(\mu, \lambda)$}\\ \cline{2-6}
& (0.5,0.25) & (0.5,0.5) & (0.5,0.75) & (0.25,0.5) & (0.75,0.5) \\ \hline\hline
(2,2)         & 0.02       & 0.02      & 0.02       & 0.02       & 0.02       \\
(3,3)         & 0.11       & 0.11      & 0.12       & 0.12       & 0.11       \\
(4,4)         & 0.18       & 0.17      & 0.18       & 0.18       & 0.17       \\
(5,5)         & 0.31       & 0.30      & 0.30       & 0.31       & 0.32       \\ \hline
\end{tabular}
\end{table}

\smallskip
\noindent\ul{Calibration of DDR parameters $\lambda$ and $\mu$}. 
In order to select a suitable set of hyperparameters $\lambda$ and $\mu$ for DDR, we examine their influence on DDR's performance. Specifically, we consider values of $\mu$ from the set $\{0.6, 0.65, \ldots, 0.85\}$ and vary $\lambda$ over the range $\{0.0, 0.1, \ldots, 12.0\}$. Note that when $\lambda = 0$, DDR simplifies to OLS. Hence, we compare the regret reduction of DDR relative to OLS in Figure~\ref{fig.baseline calibration shortest path} below.
\begin{figure}[htbp]
	\centering
	\includegraphics[width=0.80\linewidth]{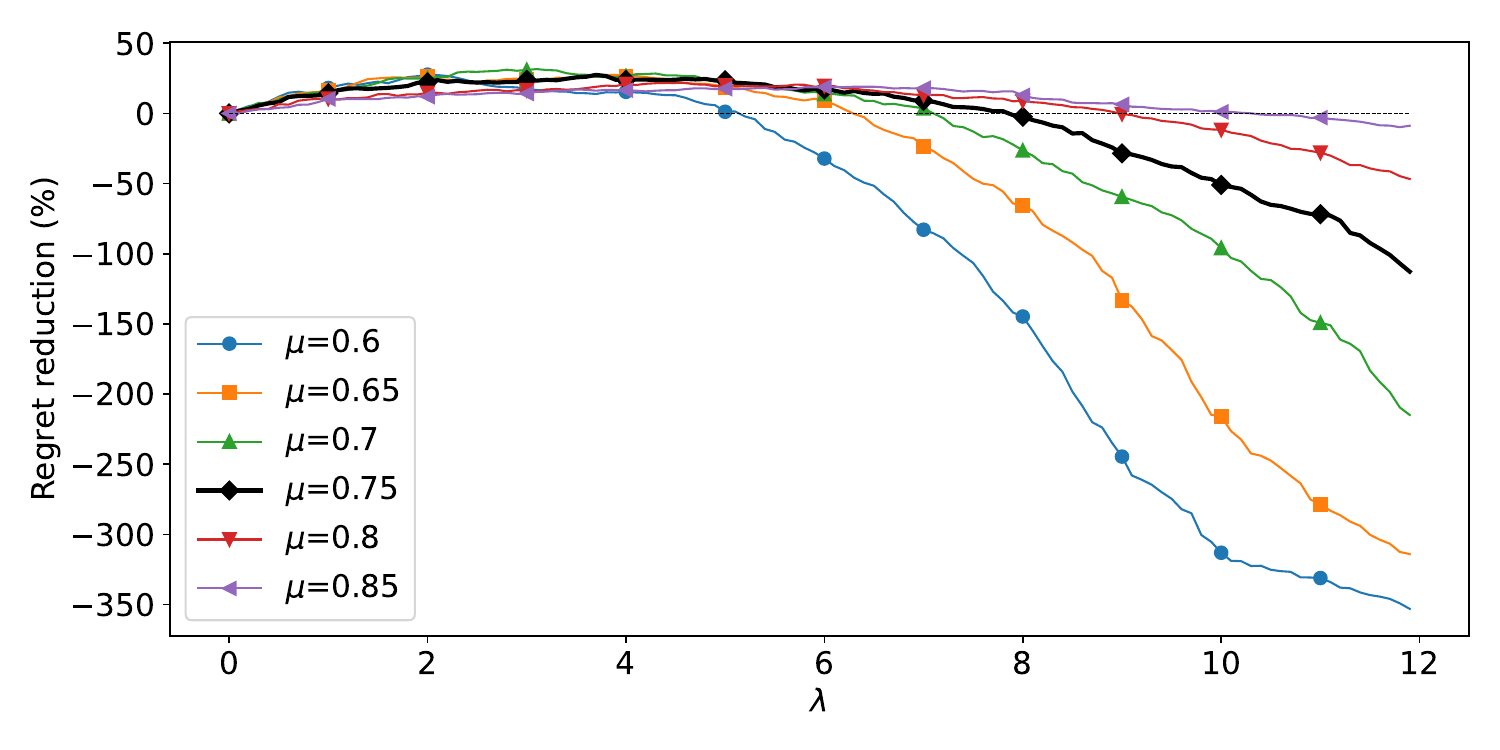}
	\vspace{-0.15cm}
	\caption{Performance of DDR against OLS for different parameters of $\mu$ and $\lambda$}
	\vspace{-0.2cm}
	\label{fig.baseline calibration shortest path}
\end{figure}

As expected, when $\lambda = 0$, DDR coincides with OLS and there is no regret reduction as they obtain the same solution. When $\lambda$ is relatively small (\textit{i.e.}, $\lambda \leq 5$), DDR consistently outperforms OLS across these values of $\mu$ considered. Indeed, $\lambda$ is the parameter that shifts the focus between prediction accuracy and cost minimization. Thus, it is expected that at small $\lambda$, it will behave like a regularizer exhibiting bias-variance trade-off. At large values of $\lambda$, prediction accuracy is de-prioritized to the point that the cost function can no longer be trusted to be close to the true cost function. It is poignant to also note that the fall off in performance from large $\lambda$ occurs earlier for smaller values of $\mu$. Indeed, when $\mu$ is small, there is a larger dependence on the estimated cost and so there is a higher chance that the cost function would depart from the true cost.  

Based on these observations, we fix $\mu = 0.75$ and $\lambda = 0.8$, as a more conservative choice of the regularization parameters, in the subsequent numerical experiments. Hereon, we will no longer mention the values of the hyperparameters unless they are varied.

\smallskip
\noindent\ul{Comparison with SLO benchmarks}. We first examine the performance of DDR against SLO models (OLS, RF, and XGBoost).  Figure~\ref{fig:DDR-3by3-mu=0.75-lamb=0.8} illustrates the performance comparison. Each point in the plot represents a new testing data set generated under the same true distribution.
\begin{figure}[htbp]
	\centering
	\includegraphics[width=0.30\linewidth]{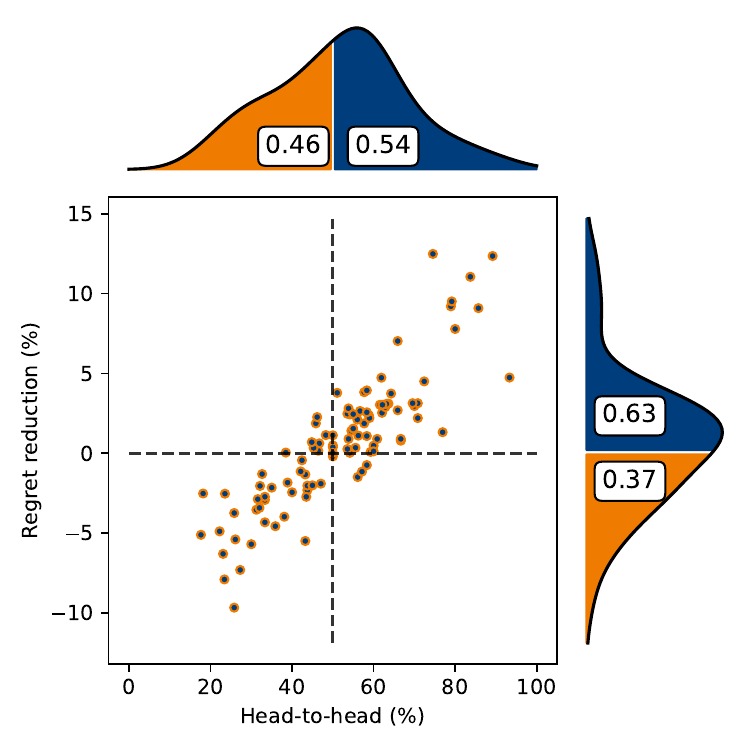}
	\includegraphics[width=0.3\linewidth]{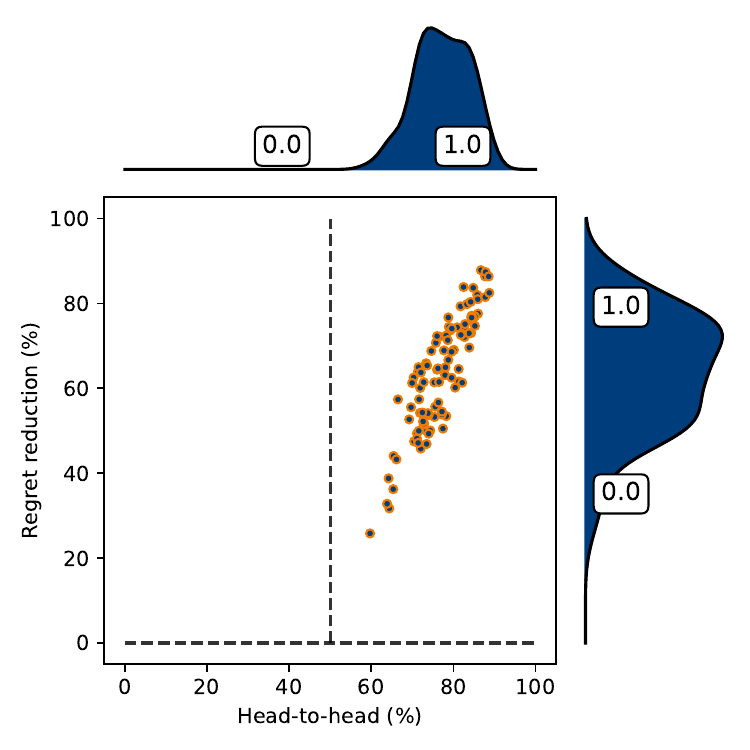} 
    \includegraphics[width=0.3\linewidth]{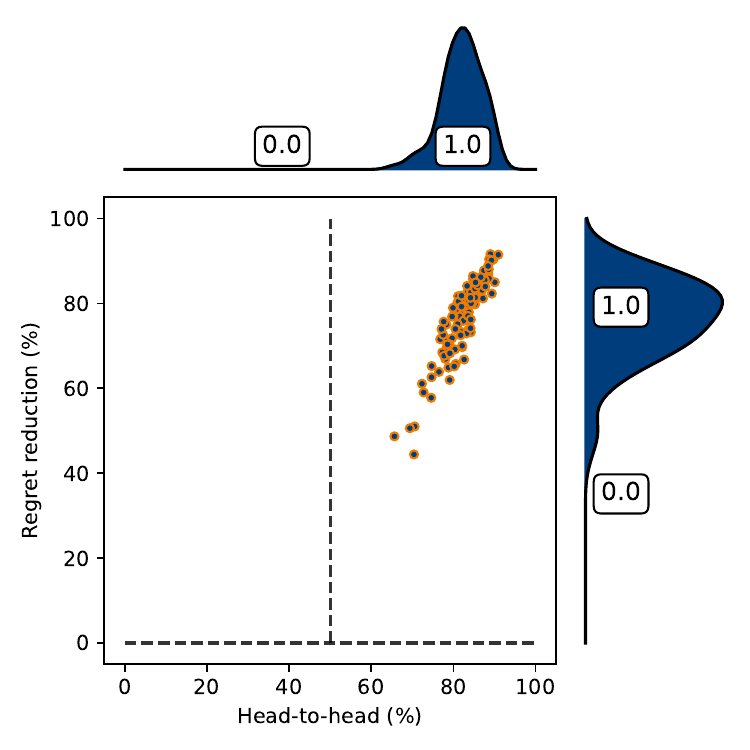} 
	\vspace{-0.15cm}
	\caption{Performance of DDR against SLO models: OLS (left), RF (middle), and XGBoost (right)} 
	\vspace{-0.2cm}
	\label{fig:DDR-3by3-mu=0.75-lamb=0.8}
\end{figure}

From the Figure \ref{fig:DDR-3by3-mu=0.75-lamb=0.8}, we can see that DDR outperforms all of the SLO models tested. OLS is the closest in performance to DDR; $63\%$ of the time, DDR obtains a better regret over OLS. On the other hand, RF and XGBoost models do not perform well, with DDR outperforming them in every instance. This runs contrary to expectations as RF and XGBoost come across as more `advanced' machine learning models, and thus may be expected to have stronger performance. To understand this, we zoom in to the in- and out-of-sample prediction accuracy of the SLO benchmarks in Table \ref{tab.accuracy} below. Here, as the experiment is repeated over multiple data sets, we present the means and extreme percentiles. 

\begin{table}[H]
\centering
\caption{Prediction accuracy of SLO models in terms of root mean square error (RMSE) across multiple datasets}\label{tab.accuracy}
\vspace{0.20cm}
\begin{tabular}{l||c|c|c||c|c|c}
\toprule
\multirow{2}{*}{\textbf{Models}} & \multicolumn{3}{c||}{\textbf{In-sample RMSE}} &\multicolumn{3}{c}{\textbf{Out-of-sample RMSE}}  \\ \cline{2-7}
& $5^{\scriptscriptstyle{\text{th}}}\%$-tile & Mean & $95^{\scriptscriptstyle{\text{th}}}\%$-tile & $5^{\scriptscriptstyle{\text{th}}}\%$-tile & Mean & $95^{\scriptscriptstyle{\text{th}}}\%$-tile \\
\midrule
OLS  & 3.8393 & 3.9416 & 4.0682 & 4.1279 & 4.1846 & 4.2395 \\\hline
RF   & 1.5872 & 1.6531 & 1.7039 & 4.3479 & 4.4255 & 4.5155 \\\hline
XGBoost  & 2.9835 & 3.1140 & 3.2322 & 4.4822 & 4.5833 & 4.6913 \\\hline
DDR  & 3.8418 & 3.9440 & 4.0705 & 4.1286 & 4.1830 & 4.2352 \\
\bottomrule
\end{tabular}
\end{table}

The in-sample accuracies of RF and XGBoost are better than OLS, but they possess a higher out-of-sample error, indicating over-fitting. The impact of this over-fitting results in a large degradation in the decision quality. \rvedit{To further confirm that this behavior is not an artifact of a specific hyperparameter choice, we report in Appendix~\ref{subsec.rmse-depth} (Table~\ref{tab.accuracy.detailed}) a finer-grained breakdown in which the maximum tree depth of RF and XGBoost is swept across multiple values. The in-sample vs.~out-of-sample gap persists across all tested depths, indicating that the deterioration in decision quality is driven by over-fitting of the predictive model rather than by a single hyperparameter setting.}
For more discussion on this from the perspective of regularization, please refer to Appendix \ref{append.relationship}. In contrast, DDR does not differ very much from OLS in terms of prediction accuracy, while making gains in terms of performance.

\smallskip
\noindent\ul{Comparison with ILO benchmarks}. We next examine the performance of DDR against ILO models (SPO+, PG, and LTR).
Figure~\ref{fig:DDR-EPO-3by3-mu=0.75-lamb=0.8} illustrates the performance comparison. 
\begin{figure}[htbp]
	\centering
	\includegraphics[width=0.32\linewidth]{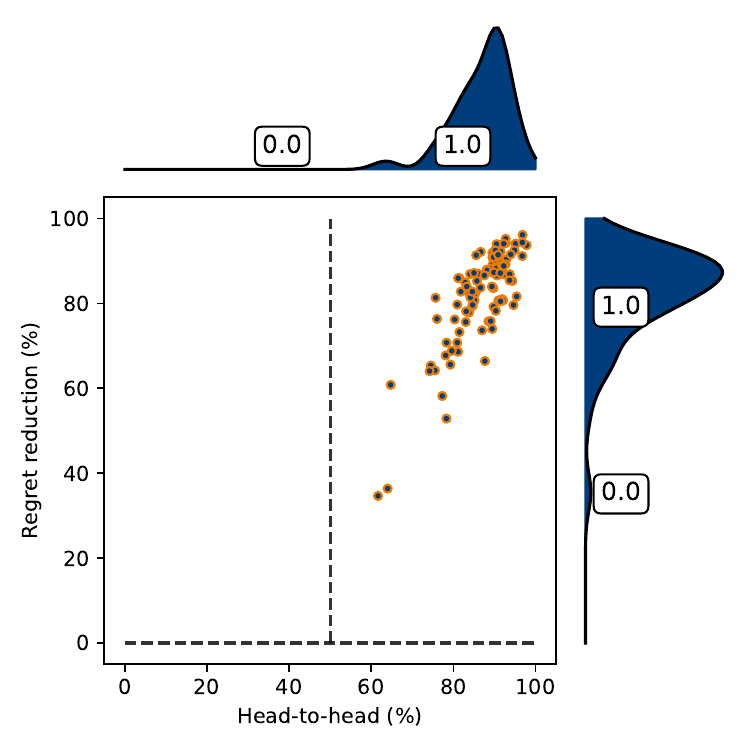}
    \includegraphics[width=0.32\linewidth]{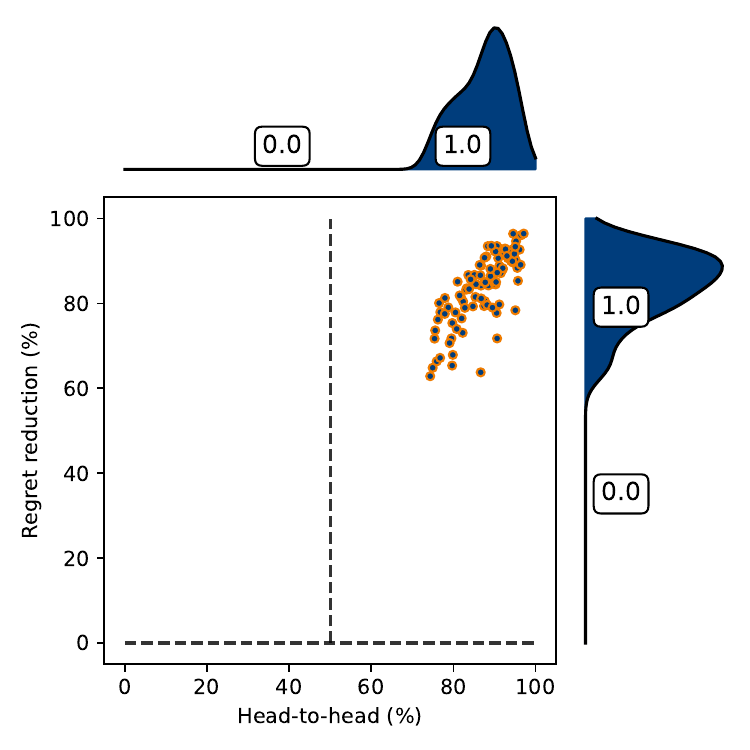}
	\includegraphics[width=0.32\linewidth]{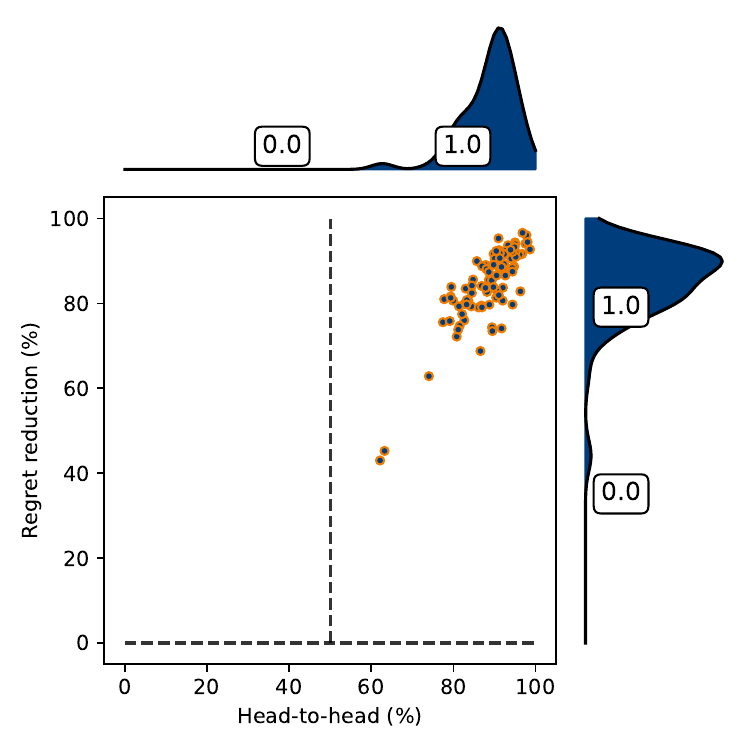} 
	\vspace{-0.15cm}
	\caption{Performance of DDR against ILO models: SPO+ (left), PG (centre) and LTR (right)}
	\vspace{-0.2cm}
	\label{fig:DDR-EPO-3by3-mu=0.75-lamb=0.8}
\end{figure}

Under no misspecification, as is the case here, ILO benchmarks do not perform well. DDR makes better decisions on every single data point in the testing data set compared to ILO benchmarks. Since SPO$+$ and these two decision-focused learning models (PG and LTR) utilize the empirical expected regret as their objective, they (\textit{i}) do not benefit from the improvement in precision in the regret function afforded by learning, and (\textit{ii}) are exposed to the possibility that their regret function departs from the true regret as there is no control over prediction accuracy. In the no (or low) misspecification regime, the error in the empirical outcomes $\mbt{z}$ is dominated by the error term $\bmt \epsilon$, resulting in over-fitting. We will see later that in the high misspecification regime, the error is dominated by misspecification error, and ILO benchmarks, in particular SPO+, will consequently regain better performance.

\smallskip
\noindent\ul{Impact of Network Structures}.
We also conduct numerical experiments to assess the performance of the \ref{model:ddr_empirical} model across various network structures. Specifically, we examine four grid networks, ranging from a $2 \times 2$ grid up to a $5 \times 5$ grid. The data generation process and parameter settings follow those described in \S\ref{section:shortest-baseline}.

Our simulation study explores multiple combinations of the \ref{model:ddr_empirical} hyperparameters $(\mu, \lambda)$, where both $\mu$ and $\lambda$ are selected from the set $\{0.2, 0.25, \ldots, 0.95\}$. For each combination, we compute the regret reduction based on 100 experimental runs. The \ref{model:ddr_empirical} models achieving optimal regret performance for each network configuration are illustrated in the left panel of Figure~\ref{fig:regret-reduction-networks}.

\begin{figure}[htbp]
	\centering
	\includegraphics[width=0.45\linewidth]{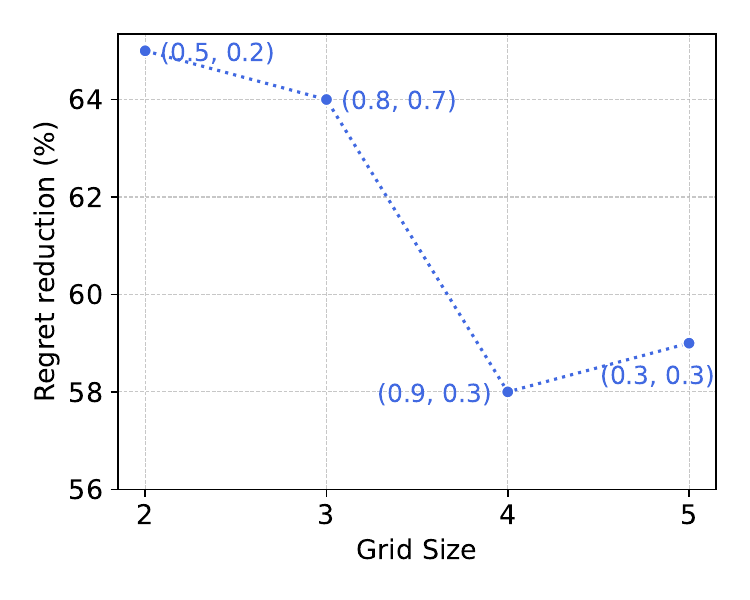}
    \includegraphics[width=0.45\linewidth]{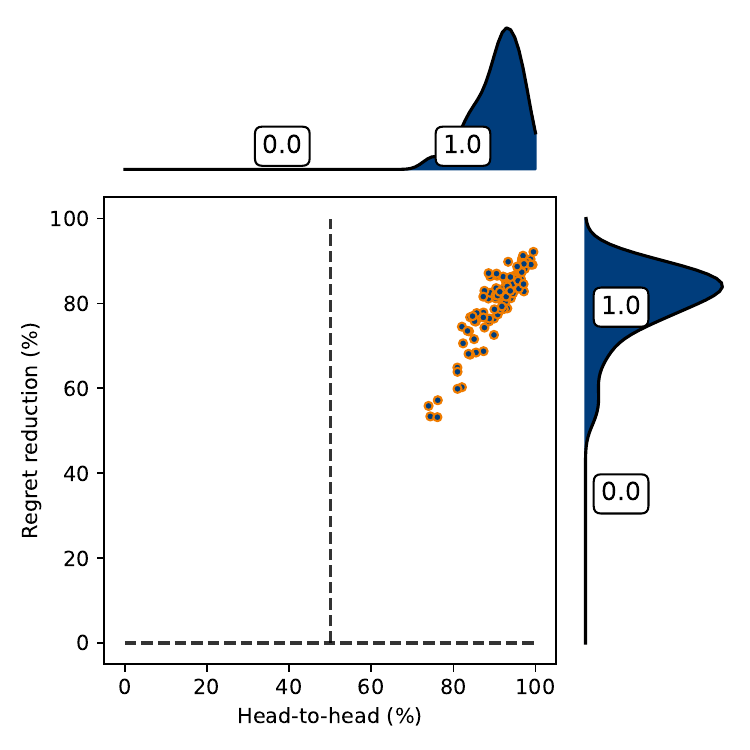}
	\vspace{-0.15cm}
	\caption{Regret reduction of the \ref{model:ddr_empirical} against OLS under different network structures (left panel) and comparison of~\ref{model:ddr_empirical} against~\ref{model:spo+} for a $5 \times 5$ grid network (right panel)}
	\vspace{-0.2cm}
	\label{fig:regret-reduction-networks}
\end{figure}

We observe that the advantage of the \ref{model:ddr_empirical} model tends to decrease as the network size grows, with a modest improvement noted from the $4 \times 4$ to the $5 \times 5$ grid. Nonetheless, the \ref{model:ddr_empirical} model consistently outperforms the OLS approach in terms of regret performance across all examined network sizes. We also compare the performance of~\ref{model:ddr_empirical} against~\ref{model:spo+} for a $5\times5$ grid, and the results are plotted in the right panel of Figure~\ref{fig:regret-reduction-networks}. We observe that our \ref{model:ddr_empirical} approach continues to outperform the~\ref{model:spo+} approaches under larger network structures.

\subsection{The misspecified setting}\label{subsec.misspecified}
Here, we compare the performance of DDR against ILO benchmarks in the misspecified setting, varying $\beta \in \{0.4,0.6,0.8,1.0,1.5,2.0,4.0,8.0\}$. To better differentiate the performance among different misspecified settings, we set $N = 500$; at the previous training dataset size of $N=100$, none of the ILO benchmarks have a stronger performance than DDR for any values of $\beta \leq 8.0$. 

\begin{figure}[htbp]
	\centering
	\includegraphics[width=0.32\linewidth]{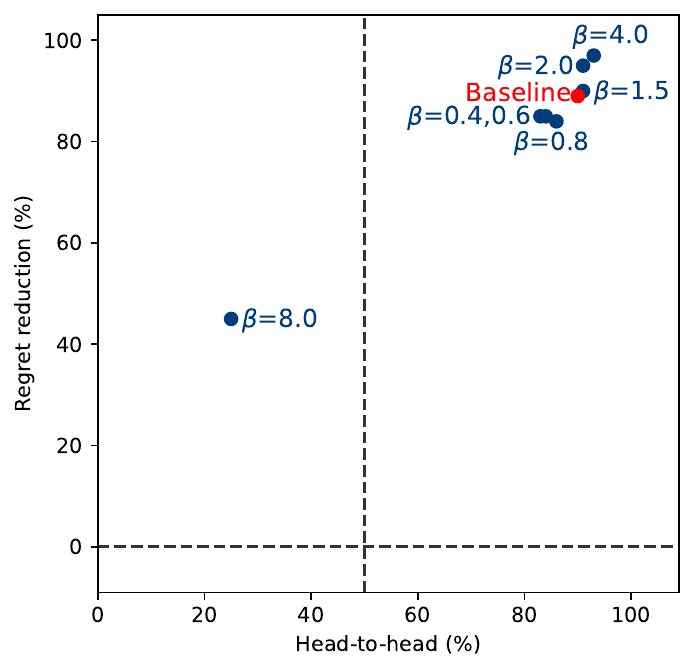}
    \includegraphics[width=0.32\linewidth]{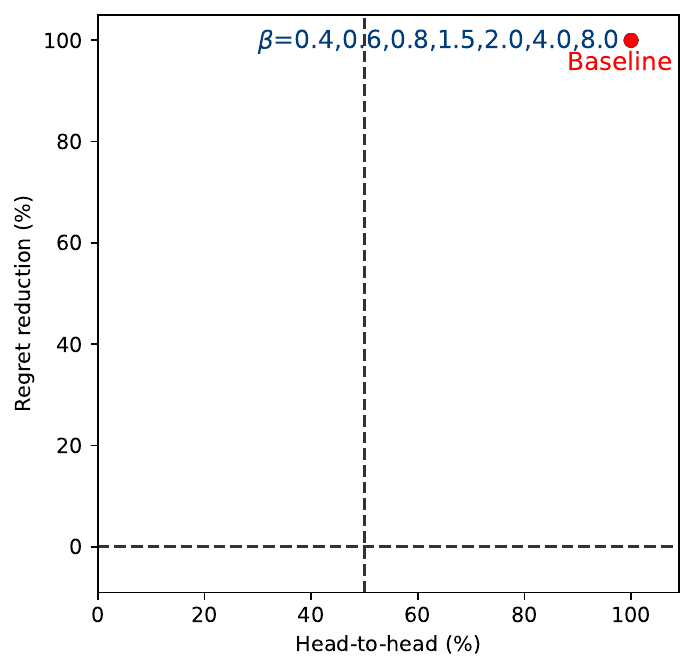}
	\includegraphics[width=0.32\linewidth]{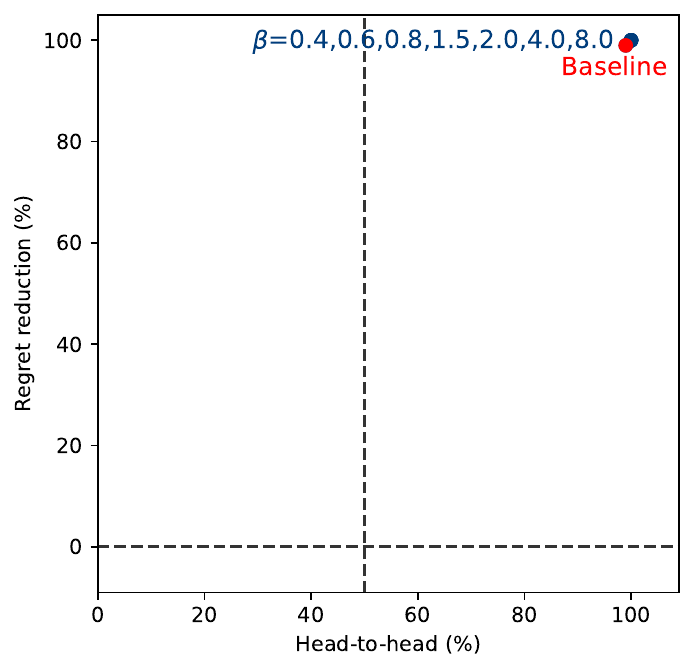} 
	\vspace{-0.15cm}
	\caption{Performance of DDR against ILO models: SPO+ (left), PG (centre) and LTR (right) in the misspecified setting}
	\vspace{-0.2cm}
	\label{fig:DDR-EPO-3by3-mu=0.75-lamb=0.8-mis}
\end{figure}

We can see that when misspecification rises to around $\beta = 8.0$, the performance of SPO+ eventually catches up to DDR; PG and LTR never do. Note that, for convenience of comparison, we have fixed a choice of $\lambda$ and $\mu$ and it would still be possible to further improve the performance of DDR by calibrating for $\lambda$ and $\mu$. Because empirical outcomes $\bmt z$ are heavily used in SPO+, this greatly reduces their susceptibility to misspecification as observed data $\mbt z = \bm g(\mbt x) + \bm \epsilon$ inherently includes information about the true relationship $\bm g(\cdot)$ which captures the misspecification. Instead, reliance on estimated $\bmh z = \bm{f}(\bm x; \bmh w)$ leads to a large misspecification error.  

Nonetheless, we stress that such levels of misspecification are high---large enough that a simple plot of the data points would easily detect it, notwithstanding the range of statistical tools and tests, such as residual analysis, that could easily point to the existence of misspecification. 
We aim to prove our point here, by zooming in to SPO+ as it is the only ILO model tested that beats DDR at misspecification level $\beta = 8.0$. Their relative performance is plotted in Figure \ref{fig:DDR-EPO-3by3-mu=0.75-lamb=0.8-mis} (left). Suppose now, that the decision-maker is cognizant of misspecification and introduces quadratic terms as features into the learner. Specifically, given the original feature set $x_{1}, \ldots,x_{p}$, we augmented the feature space by introducing quadratic interaction terms $x_{i}x_{j}$ for all $i,j \in [p]$. Note that the true misspecification level is $\beta = 8.0$, so incorporating quadratic terms does not lead to no misspecification, but simply just reduces the degree of misspecification. In Figure \ref{fig:DDR-EPO-3by3-mu=0.75-lamb=0.8-mis} (right), we examine the effect of doing so, when quadratic terms have been supplied to both DDR and the SPO+ benchmark.

\begin{figure}[htbp]
	\centering
	\includegraphics[width=0.45\linewidth]{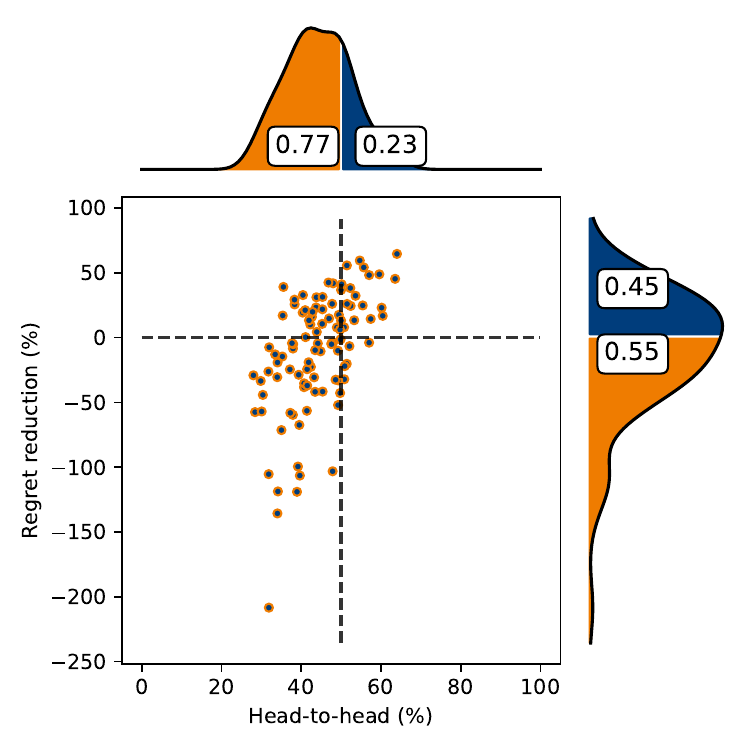}
    \includegraphics[width=0.45\linewidth]{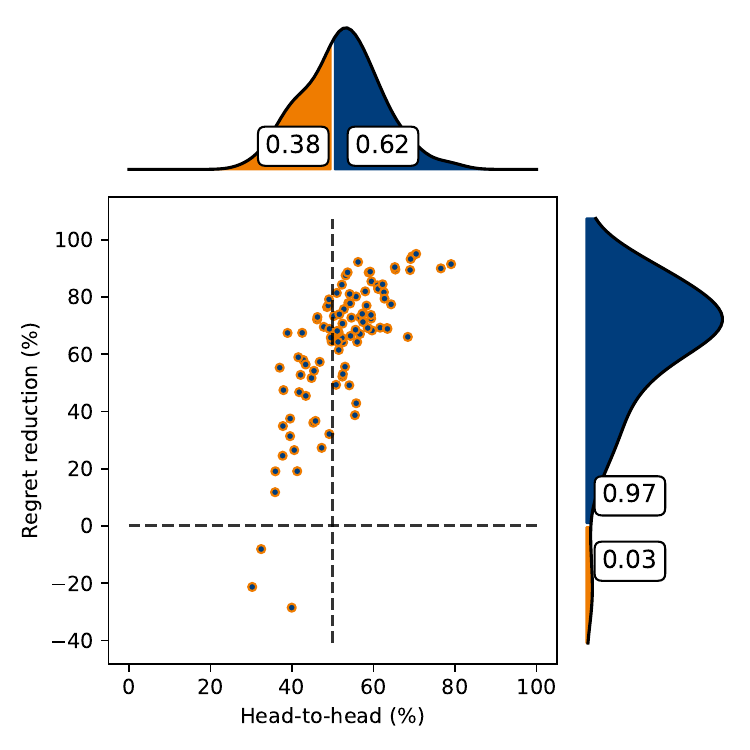}
	\vspace{-0.15cm}
	\caption{Performance of DDR against SPO+ in the misspecified setting $\beta=8.0$ and $N=500$ without (left) and with (right) quadratic terms} 
	\vspace{-0.2cm}
	\label{fig:DDR-EPO-3by3-mu=0.75-lamb=0.8-mis}
\end{figure}

Immediately, we notice that DDR rapidly regains superior performance against the SPO+ model. There are always natural ways to improve a model that is misspecified by better capturing the relationships and expanding the generality of the learner. However, the converse cannot be done---there are no natural ways to take a model that only performs well under high misspecification and somehow compel it to do better at low misspecification. We would like to stress that it is the low misspecification regime that is important, (\textit{i}) in light of current trends towards highly expressive models such as neural networks that would be misspecified with low probability, and (\textit{ii}) because it would otherwise be difficult to justify why the particular chosen learner is reasonable should the model only out-perform under high misspecification.

\section{Conclusion}\label{section::conclusion}

We propose a blended bi-objective framework, termed decision-driven regularization, for the contextual stochastic optimization problem. A key focus of our work is on the role that predictive accuracy plays in ILO, and how a lack of control for it can lead to over-fitting. We also illustrate the importance of applying bias-variance trade-off to improve the estimation of the cost function. We believe both of these discussions add important perspectives to the ILO literature when constructing ILO models. Because of these specific design elements in our model, we believe it explains DDR's superior numerical performance over benchmark models. We also illustrate the relationship between DDR and models crafted from the robust and regret perspectives, and as such generalizes SPO+. 

While we present our numerical evidence using a linear learner for DDR, we explained that DDR is compatible with any general learner, such as a tree-based learner that we illustrate in Appendix \ref{append.tree}. Nonetheless, there is much work ahead to examine the specialization of DDR to other learners popular in the machine learning literature, such as neural networks.

\bibliographystyle{informs2014} 
\bibliography{Regularization_abbrv} 

@article{mckenzie2024differentiating,
  title={Differentiating Through Integer Linear Programs with Quadratic Regularization and Davis-Yin Splitting},
  author={McKenzie, D and Wu\_Fung, S and Heaton, H},
  journal={Transactions on machine learning research},
  year={2024},
  publisher={OpenReview. net}
}

@inproceedings{mandi2025minimizing,
  title={Minimizing Surrogate Losses for Decision-Focused Learning Using Differentiable Optimization},
  author={Mandi, Jayanta and Mahmutogullari, Irfan and Guns, Tias and others},
  booktitle={28th European Conference on Artificial Intelligence, Including 14th Conference on Prestigious Applications of Intelligent Systems (PAIS 2025)},
  volume={413},
  pages={3888--3895},
  year={2025},
  organization={IOS Press}
}

@inproceedings{
    berden2026solverfree,
    title={Solver-Free Decision-Focused Learning for Linear Optimization Problems},
    author={Senne Berden and Ali {\.I}rfan Mahmuto{\u{g}}ullar{\i} and Dimos Tsouros and Tias Guns},
    booktitle={The Thirty-ninth Annual Conference on Neural Information Processing Systems},
    year={2026},
    url={https://openreview.net/forum?id=xMcKyUGTt1}
}

@article{Mandi_2024decision,
  title={Decision-focused learning: Foundations, state of the art, benchmark and future opportunities},
  author={Mandi, Jayanta and Kotary, James and Berden, Senne and Mulamba, Maxime and Bucarey, Victor and Guns, Tias and Fioretto, Ferdinando},
  journal={Journal of Artificial Intelligence Research},
  volume={80},
  pages={1623--1701},
  year={2024}
}

@article{Sadana_et_al_2024survey,
  title={A survey of contextual optimization methods for decision-making under uncertainty},
  author={Sadana, Utsav and Chenreddy, Abhilash and Delage, Erick and Forel, Alexandre and Frejinger, Emma and Vidal, Thibaut},
  journal={European Journal of Operational Research},
  year={2024},
  volume={320},
  number={2},
  pages={271--289},
  publisher={Elsevier}
}

@article{elmachtoub2023estimate,
  title={Estimate-then-optimize versus integrated-estimation-optimization versus sample average approximation: a stochastic dominance perspective.},
  author={Elmachtoub, Adam N and Lam, Henry and Zhang, Haofeng and Zhao, Yunfan},
  pages={\textit{Available at arXiv: 2304.06833}},
  year={2023}
}

@inproceedings{Mandi_Demirovic_Stuckey_Guns_2019smart,
  title={Smart predict-and-optimize for hard combinatorial optimization problems},
  author={Mandi, J. and Stuckey, P.J. and Guns, T.},
  booktitle={Proceedings of the AAAI Conference on Artificial Intelligence},
  volume={34},
  pages={1603--1610},
  year={2020}
}

@article{huang2024decision,
  title={Decision-focused learning with directional gradients},
  author={Huang, Michael and Gupta, Vishal},
  journal={Advances in Neural Information Processing Systems},
  volume={37},
  pages={79194--79220},
  year={2024}
}

@inproceedings{mandi2022decision,
  title={Decision-focused learning: Through the lens of learning to rank},
  author={Mandi, Jayanta and Bucarey, V{\i}ctor and Tchomba, Maxime Mulamba Ke and Guns, Tias},
  booktitle={International Conference on Machine Learning},
  pages={14935--14947},
  year={2022},
  organization={PMLR}
}

@article{tang2024pyepo,
  title={Pyepo: A pytorch-based end-to-end predict-then-optimize library for linear and integer programming},
  author={Tang, Bo and Khalil, Elias B},
  journal={Mathematical Programming Computation},
  volume={16},
  number={3},
  pages={297--335},
  year={2024},
  publisher={Springer}
}

@article{Notz_Pibernik_2022,
  title     = {Prescriptive Analytics for Flexible Capacity Management},
  author    = {Notz, Pascal M. and Pibernik, Richard},
  journal   = {Management Science},
  volume    = {68},
  number    = {3},
  pages     = {1756--1775},
  year      = {2022},
  month     = {March},
  publisher = {INFORMS}
}

@article{Ban_Gallien_Mersereau_2019dynamic,
  title={Dynamic procurement of new products with covariate information: The residual tree method},
  author={Ban, G.Y. and Gallien, J. and Mersereau, A. J},
  journal={Manufacturing \& Service Operations Management},
  volume={21},
  number={4},
  pages={798--815},
  year={2019},
  publisher={INFORMS}
}

@article{Lin_et_al_2022data,
  title={Data-driven newsvendor problems regularized by a profit risk constraint},
  author={Lin, Shaochong and Chen, Youhua and Li, Yanzhi and Shen, Zuo-Jun Max},
  journal={Production and Operations Management},
  volume={31},
  number={4},
  pages={1630--1644},
  year={2022},
  publisher={SAGE Publications Sage CA: Los Angeles, CA}
}

@article{Kannan_Bayraksan_Luedtke_2024residuals,
  title={Residuals-based distributionally robust optimization with covariate information},
  author={Kannan, Rohit and Bayraksan, G{\"u}zin and Luedtke, James R},
  journal={Mathematical Programming},
  volume={207},
  number={1},
  pages={369--425},
  year={2024},
  publisher={Springer}
}

@article{Sim_Tang_Zhou_Zhu_2024analytics,
  title={The analytics of robust satisficing: predict, optimize, satisfice, then fortify},
  author={Sim, Melvyn and Tang, Qinshen and Zhou, Minglong and Zhu, Taozeng},
  pages={\textit{Operations Research}},
  year={2024},
  publisher={INFORMS}
}

@article{Bertsimas_VanParys2022bootstrap,
  title={Bootstrap robust prescriptive analytics},
  author={Bertsimas, Dimitris and Van Parys, Bart},
  journal={Mathematical Programming},
  volume={195},
  number={1},
  pages={39--78},
  year={2022},
  publisher={Springer}
}

@article{Ferreira_Lee_Simchi_Levi_2016analytics,
  title={Analytics for an online retailer: Demand forecasting and price optimization},
  author={Ferreira, K.J. and Lee, B.H.A. and Simchi-Levi, D.},
  journal={Manufacturing \& Service Operations Management},
  volume={18},
  number={1},
  pages={69--88},
  year={2016},
  publisher={INFORMS}
}

@article{Hannah_et_al_2010nonparametric,
  title={Nonparametric density estimation for stochastic optimization with an observable state variable},
  author={Hannah, Lauren and Powell, Warren and Blei, David},
  journal={Advances in Neural Information Processing Systems},
  volume={23},
  year={2010}
}

@article{HoNguyen_KilincKarzan_2022risk,
  title={Risk guarantees for end-to-end prediction and optimization processes},
  author={Ho-Nguyen, Nam and K{\i}l{\i}n{\c{c}}-Karzan, Fatma},
  journal={Management Science},
  volume={68},
  number={12},
  pages={8680--8698},
  year={2022},
  publisher={INFORMS}
}

@article{Chan_Mahmood_Zhu_2025inverse,
  title={Inverse optimization: Theory and applications},
  author={Chan, Timothy CY and Mahmood, Rafid and Zhu, Ian Yihang},
  journal={Operations Research},
  volume={73},
  number={2},
  pages={1046--1074},
  year={2025},
  publisher={INFORMS}
}

@article{Kannan_et_al_2025data,
  title={Data-driven sample average approximation with covariate information.},
  author={Kannan, Rohit and Bayraksan, G{\"u}zin and Luedtke, James R},
  pages={\textit{Operations Research}},
  year={2025},
  publisher={INFORMS}
}

@article{Esteban_Morales_2022distributionally,
  title={Distributionally robust stochastic programs with side information based on trimmings},
  author={Esteban-P{\'e}rez, Adri{\'a}n and Morales, Juan M},
  journal={Mathematical Programming},
  volume={195},
  number={1},
  pages={1069--1105},
  year={2022},
  publisher={Springer}
}

@article{Perakis_Sim_Tang_Xiong_2023joint,
  title={Robust pricing and production with information partitioning and adaptation},
  author={Perakis, G. and Sim, M. and Tang, Q. and Xiong, P.},
  journal={Management Science},
  volume={69},
  number={3},
  pages={1398--1419},
  year={2023},
  publisher={INFORMS}
}

@article{Srivastava_et_al_2021data,
  title={On data-driven prescriptive analytics with side information: A regularized {N}adaraya-{W}atson approach.},
  author={Srivastava, Prateek R and Wang, Yijie and Hanasusanto, Grani A and Ho, Chin Pang},
  pages={\textit{Available at arXiv:2110.04855}},
  year={2021}
}

@article{Bengio_1997using,
  title={Using a financial training criterion rather than a prediction criterion},
  author={Bengio, Yoshua},
  journal={International Journal of Neural Systems},
  volume={8},
  number={04},
  pages={433--443},
  year={1997},
  publisher={World Scientific}
}

@inproceedings{Amos_Kolter_2017optnet,
  title={Optnet: Differentiable optimization as a layer in neural networks},
  author={Amos, Brandon and Kolter, J Zico},
  booktitle={International Conference on Machine Learning},
  pages={136--145},
  year={2017},
  organization={PMLR}
}

@article{Agrawal_et_al_2019differentiable,
  title={Differentiable convex optimization layers},
  author={Agrawal, Akshay and Amos, Brandon and Barratt, Shane and Boyd, Stephen and Diamond, Steven and Kolter, J Zico},
  journal={Advances in Neural Information Processing Systems},
  volume={32},
  year={2019}
}

@article{Vlastelica_et_al_2019differentiation,
  title={Differentiation of blackbox combinatorial solvers.},
  author={Vlastelica, Marin and Paulus, Anselm and Musil, V{\'\i}t and Martius, Georg and Rol{\'\i}nek, Michal},
  pages={\textit{Available at arXiv:1912.02175}},
  year={2019}
}

@article{Wilder_et_al_2019end,
  title={End to end learning and optimization on graphs},
  author={Wilder, Bryan and Ewing, Eric and Dilkina, Bistra and Tambe, Milind},
  journal={Advances in Neural Information Processing Systems},
  volume={32},
  year={2019}
}

@inproceedings{Jeong_et_al_2022exact,
  title={An exact symbolic reduction of linear smart predict+ optimize to mixed integer linear programming},
  author={Jeong, Jihwan and Jaggi, Parth and Butler, Andrew and Sanner, Scott},
  booktitle={International Conference on Machine Learning},
  pages={10053--10067},
  year={2022},
  organization={PMLR}
}

@article{Chung_et_al_2022decision,
  title={Decision-aware learning for optimizing health supply chains},
  author={Chung, Tsai-Hsuan and Rostami, Vahid and Bastani, Hamsa and Bastani, Osbert},
  journal={Available at arXiv:2211.08507},
  year={2022}
}

@article{Lawless_Zhou_2022note,
  title={A note on task-aware loss via reweighing prediction loss by decision-regret.},
  author={Lawless, Connor and Zhou, Angela},
  pages={\textit{Available at arXiv:2211.05116}},
  year={2022}
}

@article{Munoz_Pineda_Morales_2022bilevel,
  title={A bilevel framework for decision-making under uncertainty with contextual information},
  author={Mu{\~n}oz, Miguel Angel and Pineda, Salvador and Morales, Juan Miguel},
  journal={Omega},
  volume={108},
  pages={102575},
  year={2022},
  publisher={Elsevier}
}

@article{Estes_Richard_2023smart,
  title={Smart predict-then-optimize for two-stage linear programs with side information},
  author={Estes, Alexander S and Richard, Jean-Philippe P},
  journal={INFORMS Journal on Optimization},
  volume={5},
  number={3},
  pages={295--320},
  year={2023},
  publisher={INFORMS}
}

@inproceedings{Sun_et_al_2023maximum,
  title={Maximum optimality margin: A unified approach for contextual linear programming and inverse linear programming},
  author={Sun, Chunlin and Liu, Shang and Li, Xiaocheng},
  booktitle={International Conference on Machine Learning},
  pages={32886--32912},
  year={2023},
  organization={PMLR}
}

@article{Kong_et_al_2022end,
  title={End-to-end stochastic optimization with energy-based model},
  author={Kong, Lingkai and Cui, Jiaming and Zhuang, Yuchen and Feng, Rui and Prakash, B Aditya and Zhang, Chao},
  journal={Advances in Neural Information Processing Systems},
  volume={35},
  pages={11341--11354},
  year={2022}
}

@article{Qi_Grigas_Shen_2021integrated,
  title={Integrated conditional estimation-optimization.},
  author={Qi, Meng and Grigas, Paul and Shen, Zuo-Jun Max},
  pages={\textit{Available at arXiv:2110.12351}},
  year={2021}
}

@article{Bertsimas_Kallus_2020predictive,
	title={From predictive to prescriptive analytics},
	author={Bertsimas, D. and Kallus, N.},
	journal={Management Science},
	volume={66},
    number={3},
    pages={1025--1044},
    year={2020},
    publisher={INFORMS}
}

@article{Elmachtoub_Grigas_2022smart,
  title={Smart “predict, then optimize”},
  author={Elmachtoub, A. and Grigas, P.},
  journal={Management Science},
  volume={68},
  number={1},
  pages={9--26},
  year={2022},
  publisher={INFORMS}
}

@article{Tulabandhula_Rudin_2013machine,
	title={Machine learning with operational costs},
	author={Tulabandhula, T. and Rudin, C.},
	journal={Journal of Machine Learning Research},
	volume={14},
	number={25},
	pages={1989--2028},
	year={2013}
}

@article{Zhu_Xie_Sim_2022joint,
  title={Joint estimation and robustness optimization},
  author={Zhu, T. and Xie, J. and Sim, M.},
  journal={Management Science},
  volume={68},
  number={3},
  pages={1659--1677},
  year={2022},
  publisher={INFORMS}
}

@article{Ban_ElKaroui_Lim_2018machine,
  title={Machine learning and portfolio optimization},
  author={Ban, G.-Y. and {El~Karoui}, N. and Lim, A. EB},
  journal={Management Science},
  volume={64},
  number={3},
  pages={1136--1154},
  year={2018},
  publisher={INFORMS}
}

@inproceedings{el_Elmachtoub_Grigas_2019generalization,
  title={Generalization bounds in the predict-then-optimize framework},
  author={{El~Balghiti}, O. and Elmachtoub, A.N. and Grigas, P. and Tewari, A.},
  booktitle={Advances in Neural Information Processing Systems},
  pages={14412--14421},
  year={2019}
}

@article{Liu_He_Shen_2021time,
	title={On-time last mile delivery: Order assignment with travel time predictors},
	author={Liu, S. and He, L. and Shen, Z.J.M.},
	journal={Management Science.},
	volume={67},
    number={7},
    pages={4095--4119},
    year={2021}
}

@article{Liyanage_Shanthikumar_2005_practical,
  title={A practical inventory control policy using operational statistics},
  author={Liyanage, L.H. and Shanthikumar, G.},
  journal={Operations Research Letters},
  volume={33},
  number={4},
  pages={341--348},
  year={2005},
  publisher={Elsevier}
}

@article{Qi_2023practical,
  title={A practical end-to-end inventory management model with deep learning},
  author={Qi, M. and Shi, Y. and Qi, Y. and Ma, C. and Yuan, R. and Wu, D. and Shen, Z. M.},
  journal={Management Science},
  volume={69},
  number={2},
  pages={759--773},
  year={2023},
  publisher={INFORMS}
}

@article{Siege_Wagner_l2020profit,
  title={Profit Estimation Error in the Newsvendor Model Under a Parametric Demand Distribution},
  author={Siegel, A. F and Wagner, M. R},
  journal={Management Science.},
  volume={67},
  number={8},
  pages={4863--4879},
  year={2021},
  publisher={INFORMS}
}

@article{Wang_Chen_Wang_2024Contextual,
  title={Contextual optimization under covariate shift: A robust approach by intersecting {W}asserstein balls},
  author={Wang, Tianyu and Chen, Ningyuan and Wang, Chun},
  pages={\textit{Available at arXiv: 2406.02426}},
  year={2024}
}

@article{Hertog_Postek_2016_bridging,
  title={Bridging the gap between predictive and prescriptive analytics-new optimization methodology needed.},
  author={{den~Hertog}, D. and Postek, K.},
  year = {2016},
  pages = {\textit{Available at Optimization Online}}
}

@phdthesis{Mundru_2019predictive,
  title={Predictive and prescriptive methods in operations research and machine learning: an optimization approach},
  author={Mundru, N.},
  year={2019},
  school={Massachusetts Institute of Technology}
}

@article{Athey_Tibashirani_Stefan_2019generalized,
	title={Generalized Random Forests},
	author={Athey, S. and Tibshirani, J. and Wager, S.},
	journal={Annals of Statistics},
	volume={47},
	number={2},
	pages={1148--78},
	year={2019}
}

@article{BenTal_etal_2013robust,
  title={Robust solutions of optimization problems affected by uncertain probabilities},
  author={Ben-Tal, A. and {Den~Hertog}, D. and {De~Waegenaere}, A. and Melenberg, B. and Rennen, G.},
  journal={Management Science},
  volume={59},
  number={2},
  pages={341--357},
  year={2013},
  publisher={INFORMS}
}

@article{Delage_Ye2010distributionally,
	title={Distributionally robust optimization under moment uncertainty with application to data-driven problems},
	author={Delage, E. and Ye, Y.},
	journal={Operations Research},
	volume={58},
	number={3},
	pages={595--612},
	year={2010},
	publisher={INFORMS}
}

@article{Esfahani_Kuhn2018data,
	title={Data-driven distributionally robust optimization using the {W}asserstein metric: Performance guarantees and tractable reformulations},
	author={{Mohajerin~Esfahani}, P. and Kuhn, D.},
	journal={Mathematical Programming},
	volume={171},
    number={1},
    pages={115--166},
    year={2018},
    publisher={Springer}
}

@article{Gao_Chen_Kleywegt_2017_wasserstein,
	title={Wasserstein distributional robustness and regularization in statistical learning},
  author={Gao, Rui and Chen, Xi and Kleywegt, Anton J},
  journal={Available at arXiv:1712.06050},
  volume={2},
  number={4},
  year={2017}
}

@inproceedings{Kingma_Ba_2015_Adam,
  author    = {Diederik P. Kingma and Jimmy Ba},
  title     = {Adam: A Method for Stochastic Optimization},
  booktitle = {International Conference on Learning Representations},
  year      = {2015}
}

@inproceedings{Amos_Kolter_2017_ICNN,
  title={Input convex neural networks},
  author={Amos, Brandon and Xu, Lei and Kolter, J Zico},
  booktitle={International Conference on Machine Learning},
  pages={146--155},
  year={2017},
  organization={PMLR}
}

@inproceedings{Wong_Kolter_2018_Polytope,
  title={Provable defenses against adversarial examples via the convex outer adversarial polytope},
  author={Wong, Eric and Kolter, Zico},
  booktitle={International Conference on Machine Learning},
  pages={5286--5295},
  year={2018},
  organization={PMLR}
}

@article{Bertsimas_Copenhaver_2018characterization,
  title={Characterization of the equivalence of robustification and regularization in linear and matrix regression},
  author={Bertsimas, D. and Copenhaver, M.S.},
  journal={European Journal of Operational Research},
  volume={270},
  number={3},
  pages={931--942},
  year={2018},
  publisher={Elsevier}
}

@article{Bertsimas_Dunn_2017_classification,
  title={Optimal Classification Trees},
  author={Bertsimas, D. and Dunn, J.},
  journal={Machine Learning},
  volume={106},
  number={7},
  pages={1039--82},
  year={2017}
}

@article{Xu_Caramanis_Mannor_2010robust,
  title={Robust regression and {LASSO}},
  author={Xu, H. and Caramanis, C. and Mannor, S.},
  journal={IEEE Transactions on Information Theory},
  volume={56},
  number={7},
  pages={3561--74},
  year={2010}
}

@article{Poursoltani_Delage_2022adjustable,
  title={Adjustable Robust Optimization Reformulations of Two-Stage Worst-case Regret Minimization Problems},
  author={Poursoltani, M. and Delage, E.},
  journal={Operations Research},
  volume={70},
  number={5},
  pages={2906--2930},
  year={2022},
  publisher={INFORMS}
}

@article{Perakis_Roels_2008regret,
  title={Regret in the newsvendor model with partial information},
  author={Perakis, G. and Roels, G.},
  journal={Operations Research},
  volume={56},
  number={1},
  pages={188--203},
  year={2008},
  publisher={INFORMS}
}

@article{Blanchet_Kang_Karthyek_2019robust,
	title={Robust {W}asserstein profile inference and applications to machine learning},
	author={Blanchet, J. and Kang, Y. and Murthy, K.},
	journal={Journal of Applied Probability},
	volume={56},
	number={3},
	pages={830--857},
	year={2019},
	publisher={Cambridge University Press}
}

@article{Donti_Amos_Kolter_2017task,
  title={Task-based end-to-end model learning in stochastic optimization},
  author={Donti, P. and Amos, B. and Kolter, J.Z.},
  journal={Advances in Neural Information Processing Systems},
  volume={30},
  pages={5484--5494},
  year={2017}
}

@article{Deng_Sen_2018learning,
	title={Learning Enabled Optimization: Towards a Fusion of Statistical Learning and Stochastic Programming.},
	author={Deng, Y. and Sen, S.},
	pages={\textit{Available at Optimization Online}},
	year={2018},
note = {}
}

@article{Hu_Kallus_Mao_2022fast,
  title={Fast rates for contextual linear optimization},
  author={Hu, Y. and Kallus, N. and Mao, X.},
  journal={Management Science},
  volume={68},
  number={6},
  pages={4236--4245},
  year={2022},
  publisher={INFORMS}
}

@article{Kallus_Mao_2023stochastic,
  title={Stochastic optimization forests},
  author={Kallus, N. and Mao, X.},
  journal={Management Science},
  volume={69},
  number={4},
  pages={1975--1994},
  year={2023},
  publisher={INFORMS}
}

\newpage
\setcounter{page}{1}
{\centering\noindent\large\textbf{E-Companion of \\``Decision-Driven Regularization: A Blended Model for Predict-then-Optimize"\\}}
\vspace{0.5cm}

\renewcommand\thefigure{\thesection.\arabic{figure}}    
\setcounter{figure}{0}

\renewcommand{\thesection}{EC\arabic{section}}
\setcounter{section}{0}

\setcounter{equation}{0}
\renewcommand{\theequation}{\thesection.\arabic{equation}}

\setcounter{proposition}{0}
\renewcommand{\theproposition}{\thesection.\arabic{proposition}}
    
\setcounter{table}{0}
\renewcommand{\thetable}{\thesection.\arabic{table}}

\begin{APPENDICES}

\section{Omitted Proofs}\label{append.proofs}
In this segment, we present all deferred proofs from the main text.

\subsection{Proof of Claim 1 in Illustration \ref{illust.bad_weights}}\label{Proof of Claim 1}

\rvedit{The least squares minimization problem for OLS has the solution: $\bm W^ = (\bm X^{\top} \bm X)^{-1} (\bm X^{\top} \bm Z)$, where $\bm X = \begin{bmatrix} \mbt x_{1}^{\top} & \mbt x_{2}^{\top} & \mbt x_{3}^{\top} & \mbt x_{4}^{\top} \end{bmatrix}^{\top}$ and $\bm Z = (\mbt z_{1}^{\top},\mbt z_{2}^{\top},\mbt z_{3}^{\top},\mbt z_{4}^{\top})^{\top}$. It is easy to verify that these equations return $\mbh W^{\scriptscriptstyle{\text{SLO}}} = \mbc W = \begin{bmatrix} 1 & 0 \\ 0 & 1.5 \end{bmatrix}$. Hence, SLO picks route $b$ if and only if $1.5\hat{x}_2 \leq \hat{x}_1$.

To verify the conclusion for ILO, note that for any given feature vector $\bmh x$ and coefficient matrix $\bm W$, route $b$ is chosen if and only if $(\bm W \bmh x)_{1} - (\bm W \bmh x)_{2} = w_{11}\hat x_1 + w_{12}\hat x_2 - w_{21}\hat x_1 - w_{22}\hat x_2 \geq 0$. For convenience, let $\bar{w}:= w_{11}-w_{21}$ and $\underbar{$w$} := w_{22}-w_{12}$. Hence, the optimal decision $\bm y^\star$ will pick route $b$ if and only if $\underbar{$w$}\hat x_2 \leq \bar{w} \hat x_1$.  

For ILO, in order to get the minimum objective, ILO chooses $\mbh W^{\scriptscriptstyle{\text{ILO}}}$ such that the smaller of the two components of $\mbt{z}_{n}$ is chosen. 
Specifically, to minimize ILO loss, we need to pick route $b$ for data points 1, 3 and 4, but pick route $a$ for data point 2. Hence, optimal $\mbh W^{\scriptscriptstyle{\text{ILO}}}$ is attained when 
\begin{align*}
    \bar{w} - \underbar{$w$} \geq 0, \\
    2 \bar{w} - \underbar{$w$} \geq 0, \\
    \bar{w} - 2\underbar{$w$} \leq 0, \\
    3\bar{w} - 2\underbar{$w$} \geq 0.
\end{align*}
Subtracting the LHS of the first inequality from the third one, we have $\bar w-\underbar{$w$}- (\bar w-2\underbar{$w$})\ge0$, i.e.,  $\underbar{$w$} \geq 0$. Then the above four inequalities are equivalent to the condition  $1 \leq \bar{w} / \underbar{$w$} \leq 2$.
In other words, any matrix $\mbh W^{\scriptscriptstyle{\text{ILO}}}$ satisfying $1 \leq \bar{w} / \underbar{$w$} \leq 2$, $\underbar{$w$} \geq 0$ is optimal for the ILO loss function. Thus, ILO chooses route $b$ if and only if $\hat x_2 \leq (\bar{w} / \underbar{$w$}) \hat x_1$ for some $1 \leq \bar{w} / \underbar{$w$} \leq 2$ that is degenerate and indifferentiable in-sample under the ILO loss. \hfill\Halmos
}

\subsection{Proof of Proposition \ref{prop.bv}}
\begin{align*}
 & \E_{\mathcal{D}^N}\big[\big(\bm y^\top\E_{\bm z|\bm x}[\bm z]-\bm y^\top\bm{f}(\bm x; \hat{\bm w})\big)^2\big]\\
 =  & \E_{\mathcal{D}^N}\big[\big(\bm y^\top g(\bm x)-\bm y^\top\bm{f}(\bm x; \hat{\bm w})\big)^2\big]  \\
 = &   \E_{\mathcal{D}^N}\big[\bm y^\top g(\bm x)\big]^2-2\E_{\mathcal{D}^N}\big[\bm y^\top g(\bm x)\bm y^\top\bm{f}(\bm x; \hat{\bm w})\big]+\E_{\mathcal{D}^N}\big[(\bm y^\top\bm{f}(\bm x; \hat{\bm w}))^2\big]  \\
 = &[\bm y^\top g(\bm x)]^2-2\bm y^\top g(\bm x)\E_{\mathcal{D}^N}[\bm y^\top\bm{f}(\bm x; \hat{\bm w})]+\E_{\mathcal{D}^N}\big[(\bm y^\top\bm{f}(\bm x; \hat{\bm w}))^2\big] -(\E_{\mathcal{D}^N}\big[\bm y^\top\bm{f}(\bm x; \hat{\bm w})\big])^2+(\E_{\mathcal{D}^N}\big[\bm y^\top\bm{f}(\bm x; \hat{\bm w})\big])^2\\
 =& \big[\bm y^\top g(\bm x) - \E_{\mathcal{D}^N}\bm y^\top\bm{f}(\bm x; \hat{\bm w}) \big]^2+\E_{\mathcal{D}^N}\big[(\bm y^\top\bm{f}(\bm x; \hat{\bm w}))^2\big] -(\E_{\mathcal{D}^N}\big[\bm y^\top\bm{f}(\bm x; \hat{\bm w})\big])^2.
\end{align*}
where the first equality holds due to $\E_{\bm z|\bm x}[\bm z]=g(\bm x)$, the third equality holds because $\bm y$ and $\bm x$ are deterministic with respect to the data $\mathcal{D}^N$ and the last equality comes from rearranging terms.

\subsection{Proof of Proposition \ref{prop.valueequiv}}

Let $(\alpha_1, \alpha_2)$ be the dual variables, then the Lagrangian relaxation of Problem \eqref{eq.valueconst} is given by:
\begin{align*}
&\min\limits_{\bm y \in \cy}  \quad \Big\{\bm y^\top\bm{f}(\bmt x, \bm w) + \alpha_1 \Big( \bm y^\top \mbt{z} - \bm y^\top\bm{f}(\bmt x, \bm w) - \gamma \Big) + \alpha_2 \Big(- \bm y^\top \mbt{z} + \bm y^\top\bm{f}(\bmt x, \bm w) - \gamma \Big)\Big\}\\
=& \min\limits_{\bm y \in \cy} \quad \Big\{(\alpha_1 - \alpha_2) \bm y^\top \mbt{z} + (1 - \alpha_1 + \alpha_2) \bm y^\top\bm{f}(\bmt x, \bm w) - \gamma (\alpha_1 + \alpha_2)\Big\}.
\end{align*}
Setting $\mu = \alpha_1 - \alpha_2$, we recover the statement of the Proposition \ref{prop.valueequiv}, where the term associated with $\gamma$ is dropped as it do not influence the choice of decisions $\bm{y}$. 

\subsection{Proof of Proposition \ref{prop.compute}}
When $c(\bm y; \bm z) = \bm y^\top \bm z$ and $\cy = \{\mb{A}\bm{y} \leq \mb{b}, \bm{y} \geq \mb{0}\}$, the inner supreme problem of the DDR  becomes 
\begin{align*}
    \max\limits_{\bm y_n \geq 0}~& \bm{y}_n^\top\big(-\mu\bmt z_n - (1-\mu)\bm{f}(\mbt x_n;\bm{w}) \big)\\
    {\rm s.t.}~& \bm A \bm y_n \leq \bm b, 
\end{align*}
which is a linear programming problem and the dual of it is
\begin{align*}
    \min\limits_{\bm \kappa_n \geq 0}~& \bm \kappa_n^\top \bm b\\
    {\rm s.t.}~& \bm A^\top \bm \kappa_n \geq -\mu\bmt z_n - (1-\mu)\bm{f}(\mbt x_n;\bm{w}). 
\end{align*}\hfill\Halmos

\subsection{Motivation for the loss function uncertainty set} \label{app.loss function}
The following illustration gives some explanation about why this geometry for the uncertainty set makes sense. 

\begin{illustration}\label{illust.neypear}
\emph{\textbf{[as seen in \cite{Zhu_Xie_Sim_2022joint}]}}.
Suppose $L(\bm{w}) = -2\log\Big(\prod_n p_{\bm{z}|\bm{x}}(\bm{z}_n; \mb{x}_n, \bm{w})\Big)$ was chosen as the log-likelihood function, where $p_{\bm{z}|\bm{x}}$ is the density of $\bm{z}|\mb{x}$. Then the uncertainty set $\mathcal{U}(\rho) := \{\bm{w}: L(\bm{w}) - L(\bmt w) \leq \rho\}$ reduces to the set of all weights $\bm{w}$, which likelihood ratio against $\mbt{w}$ is no greater than some bound:
\begin{equation}
   \cu(\rho) := \left\{\bm{w}: LR(\bmt w; \bm{w}) := \frac{\prod\limits_{n}p_{\bm{z}|\bm{x}}(\bm{z}_n; \mb{x}_n, \bmt w)}{\prod\limits_{n}p_{\bm{z}|\bm{x}}(\bm{z}_n; \mb{x}_n, \bm{w})} \leq e^{\rho/2}\right\}.
\end{equation}

Suppose we interpret the null hypothesis as $\mathcal{H}_0 : \mbc{w} = \mbt w$ and the alternative hypothesis as $\mathcal{H}_1 : \mbc{w} = \bm{w} \neq \mbt w$. Then Neyman-Pearson Lemma gives that there is some confidence level $\alpha(\rho)$, corresponding to $e^{\rho/2}$, under which, the likelihood ratio test grants the highest power, that is, probability of rejecting $\mathcal{H}_0$. In other words, this uncertainty set is equivalent to considering the maximal set of all weights $\bm{w}$ that, if true, have any chance at all of rejecting $\mbt{w}$ under the significance level corresponding to $e^{\rho/2}$, given the existing dataset. 

\end{illustration}

\subsection{Proof of Theorem \ref{thm.rddr} }

Before proving Theorem \ref{thm.rddr}, we first prove the following lemma.

\begin{lemma} \label{lemma.concavity of valuation function}
The valuation function $v_{\mu}(\bm{w})$ is concave in $\bm{w}$. 
\end{lemma}
\noindent \textit{Proof}.   
Under the assumption that $\bm f(\bm x; \bm w)$ is concave in $\bm w$ for all $\bm x \in \mathcal{X}$, we see that $\bm y^\top \bm{f}(\bm x, \bm w) $ is convex in $\bm y$ and concave in $\bm w$ for all $\bm x$. It follows that $\bm y^\top \bmt z$ is convex in $\bm y$. Now for any $\mu\in[0,1]$, it is easy to see that the objective function $\mu \bm y^\top \bmt z + (1 - \mu) \bm y^\top\bm{f}(\bm x; \bm w)$ is convex in $\bm y$ and concave in $\bm{w}$. The concavity of $v_{\mu}(\bm{w})$ over $\bm{w}$ arises from the infimum operator over $\bm y$.\Halmos

Now we are ready to prove Theorem \ref{thm.rddr}.
For \ref{model:ddr_empirical}, we have
\begin{align}
    &\min\limits_{\bm{w}} \ \Big\{L(\bm w) - \lambda  \frac{1}{N}\sum_{n \in [N]} \min\limits_{\bm y_n\in\cy} \Big\{ \mu \bm y_n^\top \bmt z_n + (1-\mu) \bm y_n^\top\bm{f}(\bmt x_n; \bm w) \Big\}\Big\} \label{model::DDR reformulated}\\
    &=\min\limits_{\bm{w}}\, \bigg\{\frac{1}{N}\sum_{n \in [N]} \max\limits_{\bm y_n \in\cy}  \ \Big\{\ell\big(\bmt z_n;\bm{f}(\bmt x_n; \bm w)\big) - \lambda \Big( \mu \bm y_n^\top \bmt z_n + (1-\mu) \bm y_n^\top\bm{f}(\bmt x_n; \bm w)\Big) \Big\} \bigg\}\nonumber.
\end{align}

In what follows, we denote $\sigma = L(\bmt w)$, which is a known constant given the data. For~\ref{model:robustness_ddr}, we can write the Lagrangian of it as
\begin{align}
    &\max\limits_{\bm y_n\in\cy,\, n \in [N], \alpha \geq 0} \min\limits_{\bm w} \,\,\,\,\,\,  \quad  \alpha\Big(L(\bm w)-\sigma-\rho\Big) - \frac{1}{N}\sum_{n \in [N]} \Big[\mu \bm y_n^\top \bmt z_n + (1-\mu) \bm y_n^\top\bm{f}(\bmt x_n; \bm w) \Big]  \nonumber\\
    &=\quad\min\limits_{\bm w}\max\limits_{\bm y_n\in\cy,\, n \in [N], \alpha \geq 0} \quad \alpha\Big(L(\bm w)-\sigma-\rho\Big) - \frac{1}{N}\sum_{n \in [N]} \Big[\mu \bm y_n^\top \bmt z_n + (1-\mu) \bm y_n^\top\bm{f}(\bmt x_n; \bm w) \Big]  \label{eqn.rddr_ref}\\
    &=\quad\min\limits_{\bm w}\frac{1}{N}\sum_{n \in [N]}\max\limits_{\bm y_n\in\cy, \alpha \geq 0} \quad  \Big\{\alpha \ell\big(\bmt z_n;\bm{f}(\bmt x_n; \bm w)\big) -  \mu \bm y_n^\top \bmt z_n - (1-\mu) \bm y_n^\top\bm{f}(\bmt x_n; \bm w) \Big\} - \alpha(\sigma + \rho)\nonumber.
\end{align}

Notice that $\bm w = \bmt w$ is always a feasible solution by assumption. Moreover, $\rho > 0$, hence $\bm w = \bmt w$ is an interior point in the feasibility region. Additionally, the objective function of the inner maximization problem permits convexity in $\bm w$ and concavity in $(\bm y_n, \alpha)$.
This achieves Slater's condition, and therefore strong duality holds. It is easy to see that the solution of Problem \eqref{model::DDR reformulated} coincides with that of Problem \eqref{eqn.rddr_ref} when letting $\alpha^\star = 1/\lambda$.

\begin{remark}\label{rem.mu=1}
Here, we make a quick comment about the limiting case $\mu = 1$. Notice that if $\mu = 1$, the regularization term does not involve the weights $\bm{w}$ in any form, hence the argmin obtained would coincide with $\displaystyle\argmin_{\bm{w}} L(\bm{w})$, which is just \mbh{w}. However, it is not clear, that given a fixed $\lambda$ bounded away from $0$, that as $\mu \nearrow 1$, it is necessary for $\bm{w}(\mu) \rightarrow \mbh{w}$ uniformly. This is because the behaviour of $\mu$ is asymptotic at $\mu = 1$; beyond $\mu = 1$, there are no consistent ways for ensuring that Lemma~\ref{lemma.concavity of valuation function} holds. Indeed, we see from our numerical simulations that as $\mu$ gets closer to $1$, we do not recover $\mbh{w}$, as long as $\lambda$ does not also uniformly decrease to $0$. \hfill\Halmos
\end{remark}

\subsection{Proof of Proposition \ref{prop.jeroequiv} and Corollary \ref{coro.rddr}}

We prove both the proposition and the corollary simultaneously. 

\begin{definition}[Robustness DDR]\label{model.rddr}
The Robustness Decision-driven Regularization (RDDR, for short) model is defined as the following problem for $\mu\in[0,1)$:
\begin{align}
\max~ &\quad \rho \nonumber \\
\mbox{s.t.}~ & \quad \frac{1}{N}\sum_{n \in [N]}\Big[\mu \bm y_n^\top \bmt z_n + (1 - \mu) \bm y_n^\top\bm{f}(\bmt x_n; \bm w) \Big] \leq \tau~ &&\forall \bm{w} \in\cu(\rho)  \label{model::rddr}\tag{\code{RDDR}}\\
& \quad \rho>0;\, \bm y_n \in\cy &&\forall n \in [N]\nonumber,
\end{align}
\end{definition}

Note that when $\mu = 0$, it recovers the \ref{model:jerolike} model. Once again, denote $\sigma = L(\bmt w)$. Now for the robust counterpart of the first constraint, it can be written as:
\begin{align*}
\max_{\bm w} & \quad \frac{1}{N}\sum_{n \in [N]}\Big[\mu \bm y_n^\top \bmt z_n + (1 - \mu) \bm y_n^\top\bm{f}(\bmt x_n; \bm w) \Big]\\
\text{s.t.} & \quad L(\bm{w}) \le \sigma + \rho,
\end{align*}
the dual of which is 
\begin{align*}
\min_{\alpha \geq 0} \max_{\bm{w}} & \quad \frac{1}{N}\sum_{n \in [N]}\Big[\mu \bm y_n^\top \bmt z_n + (1 - \mu) \bm y_n^\top\bm{f}(\bmt x_n; \bm w) \Big] + \alpha \Big(\sigma + \rho - L(\bm{w})\Big).
\end{align*}

Note that $\bm w = \bmt w$ is feasible by definition. Moreover, since $\rho > 0$, $\bm w = \bmt w$ is an interior point. Hence, Slater's condition is achieved, and strong duality holds. Further, $\rho > 0$ suffices to consider only $\alpha > 0$. Hence, using $\alpha = 1/\beta$ where $\beta>0$,
\begin{align*}
\min_{\alpha > 0} \max_{\bm{w}} & \quad \frac{1}{N}\sum_{n \in [N]}\Big[\mu \bm y_n^\top \bmt z_n + (1 - \mu) \bm y_n^\top\bm{f}(\bmt x_n; \bm w) \Big] + \alpha(\sigma + \rho - L(\bm{w}))\le \tau && \\
\Leftrightarrow\min_{\tfrac{1}{\beta}}\max_{\bm{w}}& \quad \beta \frac{1}{N}\sum_{n \in [N]}\Big[\mu \bm y_n^\top \bmt z_n + (1 - \mu) \bm y_n^\top\bm{f}(\bmt x_n; \bm w) \Big] + \sigma + \rho - L(\bm{w})-\beta \tau\le 0
\end{align*}
it follows that
\begin{equation*}
\rho \le  \max_{\tfrac{1}{\beta}}\min_{\bm{w}}{-\beta \frac{1}{N}\sum_{n \in [N]} \Big[\mu \bm y_n^\top \bmt z_n + (1 - \mu) \bm y_n^\top\bm{f}(\bmt x_n; \bm w)\Big] - \sigma + L(\bm{w}) + \beta \tau}.
\end{equation*}
This allows us to reformulate the \ref{model.rddr} problem as:
\begin{equation*}
\max_{\bm y_n\in\cy,\, n \in [N], \tfrac{1}{\beta}}\min_{\bm{w}}{-\beta \frac{1}{N}\sum_{n \in [N]} \Big[\mu \bm y_n^\top \bmt z_n + (1 - \mu) \bm y_n^\top\bm{f}(\bmt x_n; \bm w) \Big]-\sigma +L(\bm{w})+\beta \tau}.
\end{equation*}

On the other hand, by Theorem \ref{thm.rddr}, \ref{model:ddr_empirical} can be reformulated as:
\begin{equation*}
 \max_{\bm y_n \in \cy,\, n \in [N]} \min_{\bm{w}}\Big\{L(\bm{w}) + \lambda \frac{1}{N}\sum_{n \in [N]} \Big[\tau  - \mu \bm y_n^\top \bmt z_n + (1 - \mu) \bm y_n^\top\bm{f}(\bmt x_n; \bm w) \Big]\Big\}.
\end{equation*}

It follows that by setting $\lambda = \beta^\star$---the optimal $\beta$ attained under the reformulated \ref{model::rddr}---the solution of \ref{model::rddr} coincides with that of~\ref{model:ddr_expectation_form}. \hfill\Halmos

\begin{remark}
Theorem \ref{thm.rddr} guarantees that, out-of-sample, if the optimal objective value of Robust-DDR is $\Gamma^\star$, then $v_{\mu}(\bmc w) \leq \Gamma^\star$, with probability $\mathbb{P}[\bmc w \in \mathcal{U}(\rho)]$, 
is related to the power of the learning model if $L$ is the log-likelihood function. This is almost as good as saying that the out-of-sample costs are upper bounded with high probability, until we recall that the valuation function is only a surrogate for the true cost function.
\end{remark}

\subsection{Proof of Theorem \ref{thm.regret ddr}}
Define
\begin{equation*}
    \cz^\star(\bm w) = \left\{ \bmc z := \big\{\bmc z_n \in \bbr^s \big\}_{ n\in [N]} ~\left|~
    \begin{aligned}
        &\exists~ \bm y_n\neq \bm y^\star(\bmc z_n),\,\, n \in [N]:\\
        &\bigg|\frac{1}{N} \sum_{n \in [N]} \Big[\ell( \bmc z_n; \bmt z_n) - \ell(\bmc z_n;\bm{f}(\bmt x_n; \bm w)) \Big]\bigg| \leq t\\
        &\Big|\bm y_n^\top \bmc z_n - \bm y_n^\top \bmt z_n\Big| \leq \phi \quad && \forall n\in[N] \\
        &\Big|\bm y_n^\top \bmc z_n -  \bm y_n^\top\bm{f}(\bmt x_n; \bm w)\Big| \leq \psi \quad &&\forall n\in[N]\\
        &\Big|\bm y^\star(\bmc z_n)^\top \bmc z_n - \bm y^\star(\bmc z_n)^\top \bmt z_n\Big| \leq \phi &&\forall n\in[N]\\
        &\Big|\bm y^\star(\bmc z_n)^\top \bmc z_n -  \bm y^\star(\bmc z_n)^\top\bm{f}(\bmt x_n; \bm w)\Big| \leq \psi  &&\forall n\in[N]\\
        &\Big|\bm y^\star(\bmc z_n)^\top \bmt z_n -  \bm y^\star(\bmt z_n)^\top \bmt z_n \Big| \leq \eta &&\forall n\in[N]
    \end{aligned}\right. 
    \right\}.
\end{equation*}
Notice that $\cz^\star(\bm w)$ is just $\cz(\bm w)$, where we preserve the global condition and then demand that good estimates of the cost function are satisfied on two decision points, $\bm y^\star(\bm z_n)$ and $\bm y_n$, which will be the optimal recourse solution obtained for the optimization problem at the end. As such, $\cz^\star(\bm w)$ relaxes the conditions for $\bm z_n$, thus $\cz^\star(\bm w) \supseteq \cz(\bm w)$, and the supremum over the latter is bounded above by the supremum of the former. Henceforth, we are interested in bounding the supremum over $\cz^\star(\bm w)$.

For ease of notations, hereinafter in this proof, we let $\displaystyle\bmh y_n \in \argmin_{\bm y \in \cy} \bm y^\top\bm{f}(\bmt x_n;\bm w), \bmc y_n \in \argmin_{\bm y \in \cy} \bm y^\top \bmc z_n$, and $\displaystyle \bmt y_n \in \argmin_{\bm y \in \cy} \bm y^\top \bmt z_n$ for $n \in [N]$. In what ensues, we shall bound the supremum by decomposing $\cz^\star(\bm w)$ into two steps. Denote 
\begin{equation*}
    \cz_0(\bm w) = \left\{ \bmc z := \big\{\bmc z_n \in \bbr^s \big\}_{ n\in [N]} ~\left|~
    \begin{aligned}
        &\Big|\bmc y_n^\top \bmc z_n - \bmc y_n^\top \bmt z_n\Big| \leq \phi && \forall n\in[N]\\
        &\Big|\bmc y_n^\top \bmc z_n -  \bmc y_n^\top\bm{f}(\bmt x_n; \bm w)\Big| \leq \psi  && \forall n\in[N]\\
        &\Big|\bmc y_n^\top \bmt z_n -  \bmt y_n^\top \bmt z_n\Big| \leq \eta  && \forall n\in[N]
    \end{aligned}\right. 
    \right\},
\end{equation*}
and 
\begin{equation*}
    \cz_1(\bm w) = \left\{ \bmc z := \big\{\bmc z_n \in \bbr^s \big\}_{ n\in [N]} ~\left|~
    \begin{aligned}
        &\exists~ \bm y_n\neq \bmc y_n,\,\, n \in [N]:\\
        &\Big|\bm y_n^\top \bmc z_n - \bm y_n^\top \bmt z_n\Big| \leq \phi \quad && \forall n\in[N] \\
        &\Big|\bm y_n^\top \bmc z_n -  \bm y_n^\top\bm{f}(\bmt x_n; \bm w)\Big| \leq \psi \quad &&\forall n\in[N]\\
                &\bigg|\frac{1}{N} \sum_{n \in [N]} \Big[\ell( \bmc z_n; \bmt z_n) - \ell(\bmc z_n;\bm{f}(\bmt x_n; \bm w)) \Big]\bigg| \leq t
    \end{aligned}\right. 
    \right\},
\end{equation*}
such that $\cz^\star(\bm w) = \cz_0(\bm w) \cap \cz_1(\bm w)$. In other words, the worst case regret is:
\begin{align}
     &\displaystyle\max_{\bmc z_n \in \cz(\bm w)} \bigg\{\dfrac{1}{N}\sum_{n\in[N]}\Big[\bmh y_n^\top \bmc z_n -\bmc y_n^\top \bmc z_n \Big] \bigg\} \nonumber \\
     \leq &\displaystyle\max_{\bmc z_n \in \cz^\star(\bm w)} \bigg\{\dfrac{1}{N}\sum_{n\in[N]}\Big[\bmh y_n^\top \bmc z_n -\bmc y_n^\top \bmc z_n \Big] \bigg\} \nonumber \\
     =&\displaystyle\max_{\bmc z_n \in \cz_0(\bm w)\cap\cz_1(\bm w)} \bigg\{\dfrac{1}{N}\sum_{n\in[N]}\Big[\bmh y_n^\top \bmc z_n -\bmc y_n^\top \bmc z_n \Big] \bigg\} \nonumber \\
     =&\displaystyle\max\limits_{ \bmc z_n\in \cz_1(\bm w)}~  \bigg\{\dfrac{1}{N}\sum_{n\in[N]}\Big[\bmh y_n^\top \bmc z_n -\bmc y_n^\top \bmc z_n \Big] \bigg\} \label{eq.firststep} \\
    &\quad {\rm s.t.}~\Big|\bmc y_n^\top \bmc z_n - \bmc y_n^\top \bmt z_n\Big| \leq \phi  \qquad \qquad\quad \forall n \in [N]; \nonumber\\
        &\qquad~~\Big|\bmc y_n^\top \bmc z_n -  \bmc y_n^\top\bm{f}(\bmt x_n; \bm w)\Big| \leq \psi  \qquad\, \forall n \in [N]; \nonumber\\
         &\qquad~~\Big|\bmc y_n^\top \bmt z_n -  \bmt y_n^\top \bmt z_n\Big| \leq \eta \qquad\qquad\quad\, \forall n \in [N].\nonumber 
\end{align}

Here, we consider the Lagrangian relaxation of (\ref{eq.firststep}). Let $\bm \alpha = (\bm\alpha^1, \bm\alpha^2)$, $\bm \beta = (\bm\beta^1, \bm\beta^2)$, $\bm \gamma = (\bm\gamma^1_n, \bm\gamma^2_n)$ be the Lagrange multipliers. Then, 
\begin{align} 
&\max_{\bmc z \in \cz^\star(\bm w)} \bigg\{\dfrac{1}{N}\sum_{n\in[N]}\Big[\bmh y_n^\top \bmc z_n - \bmc y_n^\top \bmc z_n \Big] \bigg\} \nonumber \\ 
\leq &\min_{\bm \alpha, \bm \beta, \bm\gamma \ge \bm 0} \max_{\bmc z\in\mathcal{Z}_1(\bm w)} \quad \bigg\{ 
  \dfrac{1}{N}\sum_{n\in[N]} \Big[ \bmh y_n^\top \bmc z_n-\bmc y_n^\top \bmc z_n \label{eq.lagrange1}\\
  &\quad+\alpha^1_n \Big( \phi - \bmc  y_n^\top \bmc z_n + \bmc  y_n^\top \bmt z_n \Big)
    +\alpha^2_n \Big( \phi + \bmc  y_n^\top \bmc z_n -   \bmc  y_n^\top \bmt z_n \Big) \nonumber\\
&\quad + \beta^1_n \Big( \psi - \bmc  y_n^\top \bmc z_n + \bmc y_n^\top\bm{f}(\bmt x_n; \bm w)\Big)
    + \beta^2_n \Big( \psi + \bmc  y_n^\top \bmc z_n - \bmc y_n^\top\bm{f}(\bmt x_n; \bm w)\Big)\nonumber\\
     &\quad + \gamma^1_n \Big( \eta - \bmc y_n^\top \bmt z_n + \bmt y_n^\top \bmt z_n\Big)
    + \gamma^2_n \Big( \eta + \bmc y_n^\top \bmt z_n - \bmt y_n^\top \bmt z_n\Big) \Big] \nonumber \bigg\}\\
= &\min_{\bm \alpha, \bm \beta, \bm\gamma \ge \bm 0}\max_{\bmc z\in\mathcal{Z}_1(\bm w)} \quad \bigg\{\dfrac{1}{N}\sum_{n\in[N]} \Big[\bmh y_n^\top \bmc z_n 
+ (\alpha^2_n-\alpha^1_n + \beta^2_n -\beta^1_n - 1) \bmc  y_n^\top \bmc z_n  
\nonumber\\
& \quad  + (\alpha^1_n - \alpha^2_n - \gamma^1_n + \gamma^2_n) \bmc  y_n^\top \bmt z_n + (\beta^2_n-\beta^1_n) \Big( \bmh y_n^\top\bm{f}(\bmt x_n; \bm w)-\bmc y_n^\top\bm{f}(\bmt x_n; \bm w) \Big) \nonumber\\
& \quad -(\beta^2_n-\beta^1_n) \bmh y_n^\top\bm{f}(\bmt x_n; \bm w) + (\gamma^1_n - \gamma^2_n) \bmt y_n^\top \bmt z_n + (\alpha^1_n + \alpha^2_n)\phi + (\beta^1_n+\beta^2_n)\psi + (\gamma^1_n+\gamma^2_n)\eta  \Big] \bigg\} \nonumber\\
 \le &  \max_{\bmc z\in\mathcal{Z}_1(\bm w)} \quad \bigg\{\dfrac{1}{N}\sum_{n\in[N]} \Big[ \bmh y_n^\top \bmc z_n -2 \bmh y_n^\top\bm{f}(\bmt x_n; \bm w)\Big] + a_0\bigg\}\label{eq.lagrange2}\\
 \le &  \max_{\bmc z\in\mathcal{Z}_1(\bm w);\ \bm y_n \in \cy,  n \in [N]} \quad \bigg\{\dfrac{1}{N}\sum_{n\in[N]} \Big[ \bm y_n^\top \bmc z_n-2 \bm y_n^\top\bm{f}(\bmt x_n; \bm w)\Big]\bigg\} + a_0 \label{eq.lagrange3},
\end{align}
where $\displaystyle a_0 = \phi + 2\psi + \eta + \frac{1}{N}\sum_{n\in[N]}\bmt y_n^\top \bmt z_n$. Here, inequality (\ref{eq.lagrange1}) follows from weak duality. Inequality (\ref{eq.lagrange2}) follows from two facts. First, we make the following choice of dual variables
\[
\big(\alpha^1_n, \alpha^2_n,\beta^1_n, \beta^2_n, \gamma^1_n, \gamma^2_n\big) = \Big(1,0, 0,2, 1,0\Big) \quad n \in [N].
\]
Second, note that $\bmh y_n^\top\bm{f}(\bmt x_n; \bm w)-\bmc y_n^\top\bm{f}(\bmt x_n; \bm w) \leq 0$, as $\bmh y_n\in \argmin_{\bm y_n} \bm y_n^\top\bm{f}(\bmt x_n; \bm w)$ by definition. Finally, inequality (\ref{eq.lagrange3}) follows trivially from the fact that $\bmh y_n \in \cy$ for all $n \in [N]$. 

Having relaxed $\bmh y_n$ to $\bm y_n$, we once again perform another Lagrangian relaxation, this time on $\mathcal{Z}_1(\bm w)$. Let $\bmb \alpha = (\bmb \alpha^1, \bmb \alpha^2)$,  $\bmb \beta = (\bmb \beta^1, \bmb \beta^2)$, $\bmb \gamma = (\bar \gamma^1, \bar \gamma^2)$ be the Lagrange multipliers. It follows that: 
\begin{align} 
&\max_{\bmc z \in \cz_1(\bm w);\ \bm y_n \in \cy, n\in [N]} \bigg\{\dfrac{1}{N}\sum_{n\in[N]} \Big[ \bm y_n^\top \bmc z_n
-2 \bm y_n^\top\bm{f}(\bmt x_n; \bm w)\Big]\bigg\}\nonumber\\
&\leq \min_{\bmb \alpha, \bmb \beta, \bmb \gamma \ge \bm 0} \max_{\bmc z_n;\ \bm y_n \in \cy, n \in [N]} \quad\bigg\{ \dfrac{1}{N}\sum_{n\in[N]} \Big\{ \bm y_n^\top \bmc z_n - 2 \bm y_n^\top\bm{f}(\bmt x_n; \bm w) \label{eq.lagrange4}\\ 
&\quad +\bar \alpha^1_n \Big( \phi - \bm y_n^\top \bmc z_n + \bm y_n^\top \bmt z_n \Big)
    +\bar \alpha^2_n \Big( \phi + \bm y_n^\top \bmc z_n -   \bm y_n^\top \bmt z_n \Big)\nonumber\\
&\quad + \bar \beta^1_n \Big( \psi - \bm y_n^\top \bmc z_n + \bm y_n^\top\bm{f}(\bmt x_n; \bm w)\Big)
    + \bar \beta^2_n \Big( \psi + \bm y_n^\top \bmc z_n - \bm y_n^\top\bm{f}(\bmt x_n; \bm w)\Big) \Big\}\nonumber\\
& \quad + \bar \gamma^1 \Big(t - \frac{1}{N} \sum_{n \in [N]} \big[\ell( \bmt z_n; \bmc z_n) - \ell\big(\bm{f}(\bmt x_n; \bm w); \bmc z\big)\big] \Big)
    + \bar \gamma^2 \Big( t + \frac{1}{N} \sum_{n \in [N]} \big[\ell( \bmt z_n; \bmc z_n) - \ell\big(\bm{f}(\bmt x_n; \bm w); \bmc z\big) \big]\Big)
    \bigg\}\nonumber\\
&= \min_{\bmb \alpha, \bmb \beta, \bmb \gamma \ge \bm 0}\max_{\bmc z_n;\ \bm y_n \in \cy, n \in [N]} \quad \bigg\{\dfrac{1}{N}\sum_{n\in[N]} \Big[ (1-\bar \alpha^1_n + \bar \alpha^2_n - \bar \beta^1_n+\bar \beta^2_n) \bm  y_n^\top \bmc z_n + (\bar \alpha^1_n - \bar \alpha^2_n) \bm  y_n^\top \bmt z_n \nonumber\\
& \quad  + (\bar \beta^1_n-\bar \beta^2_n-2)\bm y_n^\top\bm{f}(\bmt x_n; \bm w)  
+ (\bar \gamma^1 - \bar \gamma^2) \big( \ell\big(\bm{f}(\bmt x_n; \bm w); \bmc z_n\big) - \ell( \bmt z_n; \bmc z_n)\big) \Big] \bigg\}\nonumber\\ 
& \quad + \dfrac{1}{N}\sum_{n\in[N]} \Big[ (\bar \alpha^1_n+\bar \alpha^2_n)\phi + (\bar \beta^1_n+\bar \beta^2_n)\psi  \Big] + (\bar \gamma^1 +\bar \gamma^2)t \nonumber\\
&\leq \bar{\lambda} \frac{1}{N} \sum_{n \in [N]} \ell\big(\bm{f}(\bmt x_n; \bm w); \bmt z_n\big) +  \dfrac{1}{N}\sum_{n \in [N]} \max_{\bm y_n \in \cy} \Big[ -\mu \bm  y_n^\top \bmt z_n-(1-\mu) \bm y_n^\top\bm{f}(\bmt x_n; \bm w) \Big] + a_1\label{eq.lagrange5},
\end{align}
where $a_1 = \phi + 2\psi + t$. The inequality \eqref{eq.lagrange4} is once again due to the weak duality, and inequality \eqref{eq.lagrange5} is because of the triangle inequality assumption on $\ell(\bm u; \bm v)$ and the following choice of dual variables,
\[
\big(\bar \alpha^1_n, \bar \alpha^2_n, \bar \beta^1_n, \bar \beta^2_n, \bar \gamma^1, \bar \gamma^2\big) = \Big( \frac{1-\mu}{2}, \frac{1+\mu}{2}, \frac{3+\mu}{2}, \frac{1-\mu}{2}, \frac{1+\bar{\lambda}}{2}, \frac{1-\bar{\lambda}}{2}\Big) \quad n \in [N]
\]
for some $\mu \in [-1,1)$ and $\bar{\lambda} \in (0,1]$. 

Therefore, we have 
\begin{align*}
	&\max_{\bmc z \in \cz(\bm w)} \bigg\{ \dfrac{1}{N}\sum_{n\in[N]}\Big[\bm y^\star (\bmh z_n)^\top \bmc z_n - \min_{\bm y_n \in \cy}~\bmc y_n^\top \bmc z_n \Big] \bigg\}\\
    \leq &\max_{\bmc z\in\mathcal{Z}_1(\bm w);\ \bm y_n \in \cy,  n \in [N]} \quad \bigg\{\dfrac{1}{N}\sum_{n\in[N]} \Big[ \bm y_n^\top \bmc z_n-2 \bm y_n^\top\bm{f}(\bmt x_n; \bm w)\Big] + a_0 \bigg\}\\
    \leq & \bar{\lambda} \frac{1}{N} \sum_{n \in [N]} \ell\big(\bm{f}(\bmt x_n; \bm w); \bmt z_n\big) -  \dfrac{1}{N}\sum_{n \in [N]} \min_{\bm y_n \in \cy} \Big[ \mu \bm  y_n^\top \bmt z_n+(1-\mu) \bm y_n^\top\bm{f}(\bmt x_n; \bm w) \Big] + a_0 + a_1= \frac{1}{\lambda} V(\bm w) + a,
\end{align*}
with $\lambda  = 1/\bar{\lambda}$ and $\displaystyle a := a_0 + a_1 =  2\phi + 4\psi + \eta  + t + \frac{1}{N}\sum_{n\in[N]}\bmt y_n^\top \bmt z_n$. Notice here that $\displaystyle \frac{1}{N}\sum_{n\in[N]}\bmt y_n^\top \bmt z_n$ is evaluated only over the training data $\bmt z_n$ and hence is a constant given the training dataset.

Hence, in particular, if we choose $\bm w = \bmh w^{\rm DDR}$, we have
\begin{align*}
    \max_{\bmc z \in \cz(\bm w^{\rm DDR})} \bigg\{ \frac{1}{N}\sum_{n\in N} \Big[\bm y^\star(\bmh z_n)^\top \bmc z_n - \min_{\bm y_n \in \cy}\bm y_n^\top \bmc z_n \Big] \bigg\} \leq \frac{1}{\lambda}V(\bmh w^{\rm DDR}) + a.
\end{align*}
This completes the proof. \hfill\Halmos

\subsection{Proof of Proposition \ref{prop.spoequiv}}
When $\displaystyle L(\bm{w}) = (1-\mu) \frac{1}{N} \sum_{n \in [N]} \bm{y}^\star(\bmt z_n)^\top \bm{f}(\bmt x_n; \bm{w})$, $\lambda = 1$, and $\mu = -1$, \ref{model:ddr_empirical} becomes
\begin{align*}
    &\min\limits_{\bm{w}} \ (1-\mu) \frac{1}{N} \sum_{n \in [N]} \bm{y}^\star(\bmt z_n)^\top \bm{f}(\bmt x_n; \bm{w})- \frac{1}{N}\sum_{n \in [N]} \min\limits_{\bm y_n\in\cy} \Big\{ - \bm y_n^\top \bmt z_n + 2\bm y_n^\top\bm{f}(\bmt x_n; \bm w) \Big\},
\end{align*}
which is exactly the \ref{model:spo+} model. \hfill\Halmos

\begin{remark}\label{rem.canonical_SPO}
\ref{model:spo+} occurs when $\mu = -1$. In this case, \ref{model:ddr_empirical} reduces to, for some $\lambda \geq 0$:
\begin{equation}\label{model:spo+like}
\min_{\bm w}~
L(\bm w) - \lambda \bigg(\dfrac{1}{N}\sum_{n \in [N]}\min\limits_{\bm y_n \in \cy} \Big\{ 2\bm y^{\top}  \bm{f}(\mbt{x}_n, \bm{w})-\bm y^{\top}  \mbt{z}_n \Big\}\bigg).
\end{equation}
One way to understand the motivation of such a model, is to notice that the expression within the infimum has the following decomposition into the estimated cost and its estimation error when compared against the empirical cost:
\begin{equation}\label{model:spo+_decom}
\underbrace{\bm y^{\top}\bm{f}(\mbt{x}_n, \bm{w}) - \bm{y}^{\top} \mbt{z}}_{\text{estimation error}} + \underbrace{\bm y^{\top}\bm{f}(\mbt{x}_n,\bm{w})}_{\text{best estimate}}.
\end{equation}

We remark here that if we let $\phi = 0, \eta = 0$ and $t, \psi$ to be sufficiently large such that the first and the third conditions in $\mathcal{Z}(\bm w)$ no longer exist, {\em i.e.}, 
\begin{equation*}
    \cz(\bm w) = \left\{ \bm z := \big\{\bm z_n \in \bbr^s \big\}_{ n\in [N]} ~\left|~
    \begin{aligned}
        &\forall~ \bm y_n\in \mathcal{Y},\,\, n \in [N]:\\
        &\Big|\bm y_n^{\top} \bm z_n - \bm y_n^{\top} \bmt z_n\Big| = 0 \quad && \forall n\in[N] \\
        &\Big|\bm y^\star(\bm z_n)^{\top} \bmt z_n -  \bm y^\star(\bmt z_n)^{\top} \bmt z_n\Big| = 0 &&\forall n\in[N] 
    \end{aligned}\right. 
    \right\},
\end{equation*}
\ref{model::regret_ddr} is exactly \ref{model:spoloss}. Furthermore, if one follows the proof for Theorem \ref{thm.regret ddr} with these specifications of the parameters, then one will also recover the construction of \ref{model:spo+} from \ref{model:spoloss}. This once again reinforces the notion that the construction of \ref{model:ddr_empirical} as an approximation to \ref{model::regret_ddr} is as canonical as the construction of \ref{model:spo+} as an approximation to \ref{model:spoloss}.
\end{remark}

\section{Sub-optimality of the SLO approach and relationship of DDR to other regularizers} \label{append.relationship}

\subsection{Sub-optimality of the SLO approach}
To better understand why directly minimizing prediction error is sub-optimal, consider the toy knapsack problem of picking smallest cost $\bm z$, where cost function is $c(\bm y; \bm z) = \bm y^{\top}\bm z$, the feasible set is $\mathcal{Y} = \{ \bm y \in [0,1]^N, \bm y^{\top} \bm 1 = 1\}$, and $\bm z$ is generated from some true linear relationship $\bm z = \mbc w^{\top} \bm x + \bm \epsilon$, with true parameter $\check{\bm w}$. Typically, prediction error is measured globally. Thus, the residuals $\bm \epsilon$ of a model that minimizes global prediction error, will be smaller over less extreme-valued regions of $\bm x$ and vice versa, as illustrated by the dark blue shaded confidence interval in Figure \ref{fig.prediction accuracy} below. However, in the knapsack problem, we are concerned about extreme values of $\bm z$, which occur over the region of extreme values of $\bm x$ where the prediction error is currently largest. In other words, more generally, the cost minimization problem induces a region $\overline{\mathcal{X}} \subset \mathcal{X}$ over context $\bm x$, corresponding to a region $\overline{\mathcal{Z}} \subset \mathcal{Z}$ of the uncertainty $\bm z$, where optimal decisions to the cost minimization problem are more likely to occur than not. Moreover, this region $\overline{\mathcal{X}}$ is often extreme in relation to the support $\mathcal{X}$ as optimization, by definition, seeks extreme results. If one is able to improve the predictive accuracy over this region $\overline{\mathcal{Z}}$, then the optimization model is better able to differentiate between alternative decisions $\bm y$. Thus, a model with confidence interval indicated by the broken line in Figure \ref{fig.prediction accuracy} is expected to have higher decision power than its shaded blue counterpart, as it incurs smaller prediction error over the region that matters---the small values of $\bm z$. It may have a larger overall prediction error, but that results from sacrificing prediction accuracy over non-critical regions of $\mathcal{X}$---that lead to large values of $\bm z$---in favour of $\overline{\mathcal{X}}$.

\begin{figure}[htbp]
	\centering
	\includegraphics[width=0.40\linewidth]{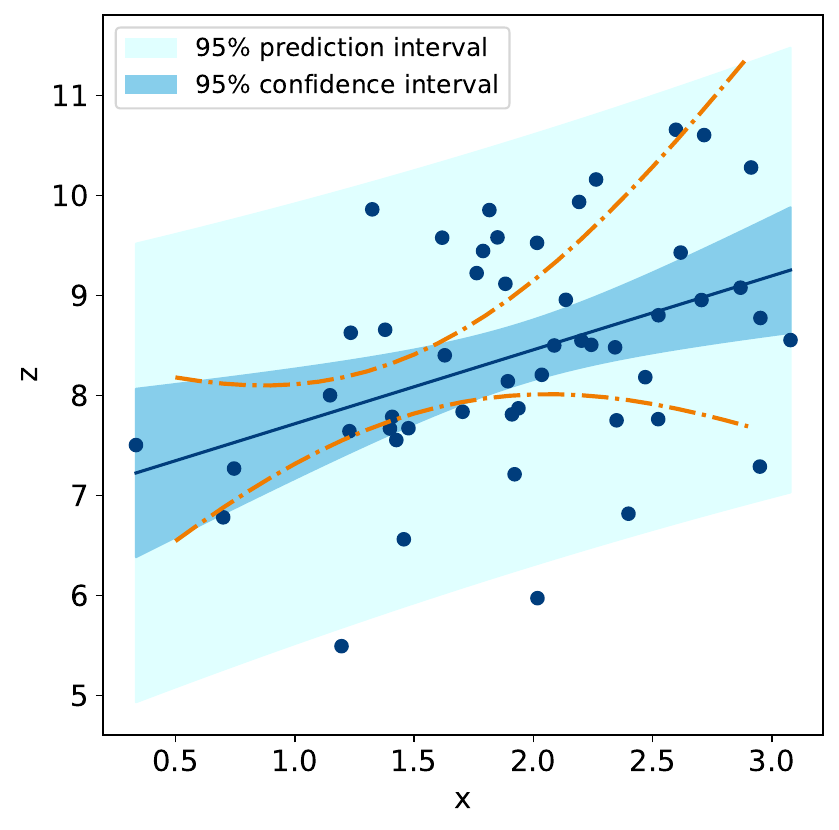}
	\vspace{-0.15cm}
	\caption{One-dimensional illustrative of centred (dark blue shaded area) and non-centred (broken orange line) confidence intervals}
	\vspace{-0.2cm}
	\label{fig.prediction accuracy}
\end{figure}

\subsection{Relationship to other regularizers}
In the \ref{model:ddr_empirical} model, the regularizer $R_{\mu}(\bm w) := - \frac{1}{N}\sum_n \nu_{\mu}^n(\bm w)$ is chosen as the sample average approximation of the valuation function $\nu_{\mu}(\bm w)$. Here, we shorthand $\nu_{\mu}^n(\bm w)$ to mean the valuation function evaluated on the single data point indexed by $n$, $(\bmt x, \bmt z) = (\bmt x_n, \bmt z_n)$.
In both Lasso and Ridge regressions, the regularizer is chosen as some norm: $R(\bm w) = ||\bm w||_q$. Nonetheless, recall that each norm can be represented as the optimizer of some optimization problem. By the definition of dual norm, we have
\begin{align}
    R(\bm w) = ||\bm w||_q =  \max\limits_{\bm \xi} & \;\sum\limits_{i\in [p]} \sum\limits_{j\in [d]} w_{ij}\xi_{i,j}, \label{eq.normdual}\\
    \mbox{s.t.} &\; \mb{\xi} \in \Xi := \big\{\bm \xi \in\mathbb{R}^{p\times d} : ||\mb{\xi}||_{q_*} \leq 1\big\} \nonumber
\end{align}
where $||\cdot||_{q_*}$ represents the dual norm of $||\cdot||_q$, and $\frac{1}{q} + \frac{1}{q_*} = 1$.

Consider the case where the learner is linear, \emph{i.e.}, $\bmt z :=\bm{f}(\bm x; \bm w) = \bm w^{\top} \bm x$. In the case where $\mu = 0$,  the decision-driven regularizer is 
\begin{align*}
    R_0(\bm w) & = \frac{1}{N} \sum_n \nu_0^n(\bm w) =  \max\limits_{y_n\in\mathcal{Y}}- \frac{1}{N} \sum_n \bm{y}_n^{\top} \bm w^{\top} \bm x_n \\
    & = \max\limits_{y_n\in\mathcal{Y}} \sum\limits_{i\in [p]} \sum\limits_{j\in [d]} w_{ij} \left(-\frac{1}{N} \sum_n x_{n,i}y_{n,j} \right).
\end{align*}
Thus, we may express $R_0(\bm w)$ as the following:
\begin{align}
    R_0(\bm w) =  \max\limits_{\bm \xi} & \;\sum\limits_{i\in [p]} \sum\limits_{j\in [d]} w_{ij}\xi_{i,j}, \label{eq.ddr_to_regu}\\
    \mbox{s.t.} &\; \mb{\xi} \in \Xi_0^N := \left\{\mb{\xi} \in\mathbb{R}^{p\times d} : \exists \bm{y}_n\in\mathcal{Y}, n\in [N], \xi_{i,j} = -\frac{1}{N} \sum_n x_{n,i}y_{n,j} \right\} \nonumber
\end{align}
that depends on the dataset $\mathcal{D}^N$. Denote $X^N$ as the $p \times N$ matrix that collects all $p$-dimensional feature vectors in the dataset of $N$ points, and denote $\mathcal{Y}^N := \mathcal{Y} \times \cdots \times \mathcal{Y}$ as the product space of $\mathcal{Y}$. Then, we can notice that $\Xi_0^N = \frac{1}{N}X^{N\top} \mathcal{Y}^N := \{\frac{1}{N}X^{N\top}\bm y : \bm y \in \mathcal{Y}^N \}$, which is the linear transformation of the feasibility space by the data matrix $X^N$.

The forms of \eqref{eq.normdual} and \eqref{eq.ddr_to_regu} allow us to draw the connection between decision-driven regularization and norm-based regularizers like Lasso and Ridge regressions. We wish to analyze this relationship by assuming that we are currently at the OLS solution, \emph{i.e.},  $\bm w  = \hat{\bm w}^{\text{OLS}}$, and investigating what impact adding either of the regularizers has on the performance of the solution -- specifically, how the base solution $\hat{\bm w}^{\text{OLS}}$, results in different weights, such as $\hat{\bm w}^{\text{Lasso}}$ in the case of Lasso. Here, we examine locally around the neighbourhood of OLS weights, $\hat{\bm w}^{\text{OLS}}$, the directions of improving (decreasing) $L(\bm w) + \lambda R(\bm w)$ for different regularizers $R(\cdot)$. As $\hat{\bm w}^{\text{OLS}}$ is the minimizer of $L(\bm w)$ (hence gradient of $L$ at $\bm w = \hat{\bm w}^{\text{OLS}}$ is $0$), the action of adding $R(\bm w)$ is to draw the solution in the direction of the gradient of $R(\bm w)$. Closer examination of problems \eqref{eq.normdual} and  \eqref{eq.ddr_to_regu} reveals that these directions are approximated by the optimal solutions of $\bm \xi$ at $\bm w = \hat{\bm w}^{\text{OLS}}$. This is because $\bm \xi $ are the dual variables in the convex conjugate forms of $R(\bm w)$. In other words, around the neighborhood of base solution $\hat{\bm w}^{\text{OLS}}$, the base solution is pushed along the optimal $\bm \xi$'s to problems \eqref{eq.normdual} and  \eqref{eq.ddr_to_regu} at $\bm w = \hat{\bm w}^{\text{OLS}}$ towards the norm-based regularizer and \ref{model:ddr_empirical} weights, respectively. However, problems \eqref{eq.normdual} and  \eqref{eq.ddr_to_regu} have the same objective, and hence the same optimizing direction, but with different feasibility sets $\Xi$ and $\Xi_0^N$. Thus, the eventual optimal solution of $\bm \xi$ depends on how different these feasibility sets are.

Our goal is to examine the impact on the performance of the eventual decisions, of pulling the OLS weights in the directions of the solutions to \eqref{eq.normdual} and \eqref{eq.ddr_to_regu}. If the solution $\bm \xi$ to problem \eqref{eq.ddr_to_regu} is a local improving direction for the performance, \emph{i.e.}, that \ref{model:ddr_empirical} improves on the performance of OLS on the eventual cost function, and if the angle between the two solutions $\bm \xi$ to problem \eqref{eq.normdual} and \eqref{eq.ddr_to_regu} is fairly large, then we can conclude that the direction in which norm-based regularizers pull the OLS weights towards, \emph{i.e.}, the solution to problem \eqref{eq.normdual}, may very likely be a local worsening direction for the performance of the eventual cost function.

This actually turns out to be the case in our simulations. We see that \ref{model:ddr_empirical} improves over OLS, and that the angles of the local directions from \ref{model:ddr_empirical} to Lasso and Ridge regressions are around 60\degree. In Table \ref{tab.angles}, we present the summary statistics of these angles amongst $100$ simulations. 
These angles are significantly large---in other words, the regularizers in Lasso and Ridge regressions are pulling the OLS weights away in very different directions from the decision-driven regularizer. As a consequence, we discover, ironically, that while Lasso and Ridge regressions both improve the prediction accuracy over OLS, they both lead to significantly poorer performance on the cost function.
\begin{table}[H]
\centering
\caption{Angles of local directions of Lasso and Ridge to \ref{model:ddr_empirical} in $100$ simulations}\label{tab.angles}
\vspace{0.20cm}
\begin{tabular}{l|c|c}
\toprule
\textbf{Angles to \ref{model:ddr_empirical}} & \textbf{Lasso} &\textbf{Ridge}  \\
\midrule
Mean            &  82\degree   &   72\degree         \\\hline
10th Percentile &   80\degree  &     72\degree         \\\hline
90th Percentile & 85\degree   &    73\degree         \\
\bottomrule
\end{tabular}
\end{table}

In conclusion, what we have explained here is that blindly adopting regularizers, even if for the purposes of improving prediction accuracy, might not lead to better performance. In fact, we can potentially see the opposite, where they lead to poorer performance, pulling the optimal solution away from the direction of improving performance. 

\section{Pseudo-code for Solving \ref{model::rddr}} \label{append.pseudo}

\begin{corollary}\label{coro.rddr}
Under the same assumptions as in Theorem \ref{thm.rddr}, then for all $\mu \in [0,1)$ and $\tau$ for which $\exists\bm{y}_n \in\cy,\, n \in [N]$ such that the target is attained $\frac{1}{N}\sum_{n \in [N]} \Big[\mu \bm y_n^\top \bmt z_n + (1 - \mu) c\big(\bm y_n;\bm{f}(\bmt x_n, \bmt w) \big)\Big] \leq \tau$, under learning optimal weights $\bmt w = \argmin L(\bm w)$, there exists some $\lambda := \lambda(\tau) \geq 0$ for which the solutions of {\rm \ref{model:ddr_empirical}} and the worst-case weights attained under optimal uncertainty set size $\rho^\star$ from {\rm \ref{model::rddr}} coincide. 
\end{corollary}

In short, we seek the best radius of the uncertainty set $\rho$, via bisection search, wherein we solve a feasibility sub-problem in each iteration. Here, we present the pseudo-code for the algorithm, however, for more details, the reader can be directed to \cite{Zhu_Xie_Sim_2022joint}, as the procedure for solving the model follows identically.

\begin{algorithm}
    {\footnotesize
    \caption{RDDR}
    \label{algo.rddr}
    \begin{algorithmic}
        \STATE{\bf Input} $\tau$ and $\mbt{z}$. 
        \STATE {\bf Initialization}: $\rho_1 = 0, \rho_2 = \bar{\rho}$, where $\bar{\rho}$ is a sufficiently large number.
        \WHILE{$\rho_2 - \rho_{1} > \epsilon$}
            \STATE $\rho:= (\rho_1+r_2)/2$
            \STATE Solve the subproblem \ref{eq.subrddr}, obtain optimal value $\delta^\star$ and optimal decisions $\bm y_n^\star$.
            \IF{$\delta^\star>0$}
                \STATE $\rho_2 = \rho$
            \ELSE
                \STATE $\rho_1=\rho$
            \ENDIF
        \ENDWHILE
        \STATE {\bf Output}: Optimal $\rho^\star = (\rho_1+\rho_2)/2$ and optimal decision $\bm y^{\rm RDDR}_n = \bm y^\star_n$.
        \STATE{\bf Input} Optimal $\rho = \rho^\star$ and optimal decision $\bm y = \bm y_n^\star$
        \STATE{\bf Do} Solve the Problem \eqref{eq.worstw}, and obtain $\bm{w}^\star$
        \STATE{\bf Output} Worst scenario $\hat{\bm{w}}^{\sstext{RDDR}} = \bm{w}^\star$
    \end{algorithmic} }
\end{algorithm}

In this algorithm, the sub-problem is described as:
\begin{align}
    \min\limits_{t} &\quad t-\tau \label{eq.subrddr}\tag{\code{Sub-RDDR}}\\
    \mbox{s.t.} & \quad \frac{1}{N}\sum_{n \in [N]} \Big[\mu \bm y_n^\top \bmt z_n + (1 - \mu) \bm y_n^\top\bm{f}(\bmt x_n; \bm w)\Big] \leq t, \qquad \forall \bm{w}\in\mathcal{U}(\rho) \nonumber \\
    \mbox{where} & \quad\cu(\rho) := \{\bm{w}:L(\bm{w}) - L(\bmb w) \leq \rho\}. \nonumber
\end{align}

We finally utilize $\rho^\star$ and $\bm y^\star$ derived from the overarching problem to solve for the worst case $\bmh w^{\rm RDDR}$. This is done via,
\begin{align}
    \argmax\limits_{\bm{w}\in\mathcal{U}(\rho^\star)} & \ \frac{1}{N} \sum_{n \in [N]} \Big[\mu\, c(\bm{y}_n^\star; \mbt{z}_n) + (1-\mu)\, c\big(\bm{y}_n^\star;\bm{f}(\bmt x_n; \bm w)\big)\Big] \label{eq.worstw}\tag{\code{ROBUSTW}}
\end{align}

\section{Additional Simulation Results}\label{append.more_simu} 

In this Appendix, we present further results from the numerical illustration that were omitted from the main text for brevity. This includes the robustness of our findings to different parameters and a short study into the case when the network size is varied.

\rvedit{
\subsection{A Closer Look at Over-fitting in the Tree-based SLO Benchmarks}\label{subsec.rmse-depth}

In \S\ref{section:shortest-baseline} of the main text, Table~\ref{tab.accuracy} reports the in- and out-of-sample prediction accuracy of the SLO benchmarks (OLS, RF, XGBoost) and DDR under their default settings, which points to potential over-fitting in the two tree-based learners. To further elucidate the extent to which such over-fitting depends on model complexity, Table~\ref{tab.accuracy.detailed} refines the entries for RF and XGBoost by sweeping the maximum tree depth across $\{2,4,6,8,\text{Default}\}$, while keeping OLS and DDR unchanged as references. We can see that changing the hyperparameters might not resolve over-fitting, but is rather an inherent feature of the learner.


\begin{table}[H]
\centering
\caption{\rvedit{Prediction accuracy of SLO models in terms of root mean square error (RMSE) across multiple datasets}}\label{tab.accuracy.detailed}
\vspace{0.20cm}
\scalebox{0.9}{
\begin{tabular}{l||c|c|c||c|c|c}
\toprule
\multirow{2}{*}{\textbf{Models}} & \multicolumn{3}{c||}{\textbf{In-sample RMSE}} &\multicolumn{3}{c}{\textbf{Out-of-sample RMSE}}  \\ \cline{2-7}
& $5^{\scriptscriptstyle{\text{th}}}\%$-tile & Mean & $95^{\scriptscriptstyle{\text{th}}}\%$-tile & $5^{\scriptscriptstyle{\text{th}}}\%$-tile & Mean & $95^{\scriptscriptstyle{\text{th}}}\%$-tile \\
\midrule
OLS  & 3.8393 & 3.9416 & 4.0682 & 4.1279 & 4.1846 & 4.2395 \\\hline
DDR  & 3.8418 & 3.9440 & 4.0705 & 4.1286 & 4.1830 & 4.2352 \\\hline
RF(Maximum depth = 2)& 4.0068 &	4.1499 & 4.2974	& 4.3693 & 4.4779 &	4.5875 \\\hline
RF(Maximum depth = 4)& 3.36	  & 3.4682 & 3.5785	& 4.3047 & 4.3904 & 4.4816 \\\hline
RF(Maximum depth = 6)& 2.4572 &	2.6123 & 2.782	& 4.3191 & 4.3939 & 4.4766 \\\hline
RF(Maximum depth = 8)& 1.8514 &	1.9727 & 2.1223	& 4.338	 & 4.4143 &	4.504 \\\hline
RF(Default)   & 1.5872 & 1.6531 & 1.7039 & 4.3479 & 4.4255 & 4.5155 \\\hline
XGBoost(Maximum depth = 2)& 3.924 &	4.0634 & 4.221 & 4.3788 & 4.4943 & 4.6057 \\\hline
XGBoost(Maximum depth = 4)& 3.38  & 3.5202 & 3.6778 & 4.4097 & 4.517 & 4.6255 \\\hline
XGBoost(Maximum depth = 6)& 2.9835 & 3.114 & 3.2322	& 4.4822 & 4.5833 &	4.6913 \\\hline
XGBoost(Maximum depth = 8)& 2.8146 & 2.9393 & 3.0729 & 4.5155 &	4.6173 & 4.7142\\\hline
XGBoost(Default)  & 2.9835 & 3.1140 & 3.2322 & 4.4822 & 4.5833 & 4.6913 \\
\bottomrule
\end{tabular}}
\end{table}
}

\subsection{Robustness of Our Model Across Various Settings}
This section presents a series of numerical experiments designed to evaluate the robustness of the proposed \ref{model:ddr_empirical} model under varying conditions. Specifically, we investigate the influence of training data size, feature dimensionality, noise level, and model misspecification.
To assess the impact of training data size, the sample size $N$ is varied across the set $\{50, 100, 200, 500, 1000\}$, while all other parameters are held constant at their baseline values. Likewise, the number of features $p$ is selected from $\{1, 3, 5, 7, 10, 15\}$. To evaluate the effect of model misspecification, the parameter $\beta$ is varied over the set $\{0.4, 0.8, 1.0, 1.4, 1.6, 2.0, 3.0, 4.0\}$. Finally, the noise level $\bar{\epsilon}$ is examined by selecting values from $\{0.25, 0.5, 0.75, 1.0\}$. The comparison between \ref{model:ddr_empirical} and OLS is summarized in Figure~\ref{fig:DDR-OLS-3by3-mu=0.75-lamb=0.8_various_settings}.
\begin{figure}[htbp]
\vspace{-0.45cm}
	\centering
	\includegraphics[width=0.4\linewidth]{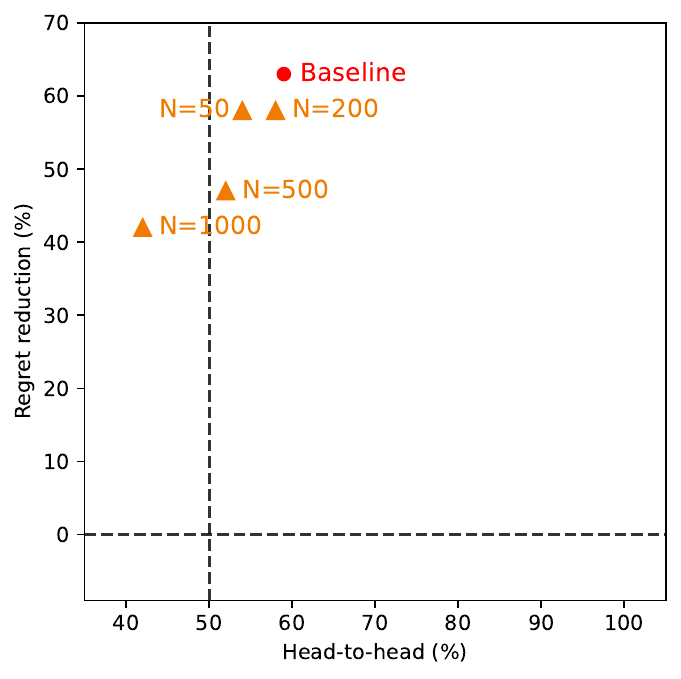}
	\includegraphics[width=0.4\linewidth]{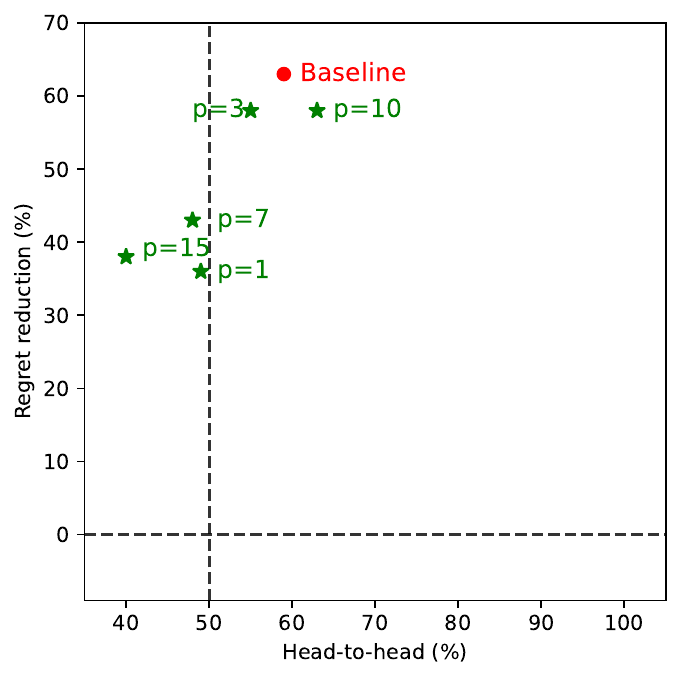}

        \includegraphics[width=0.4\linewidth]{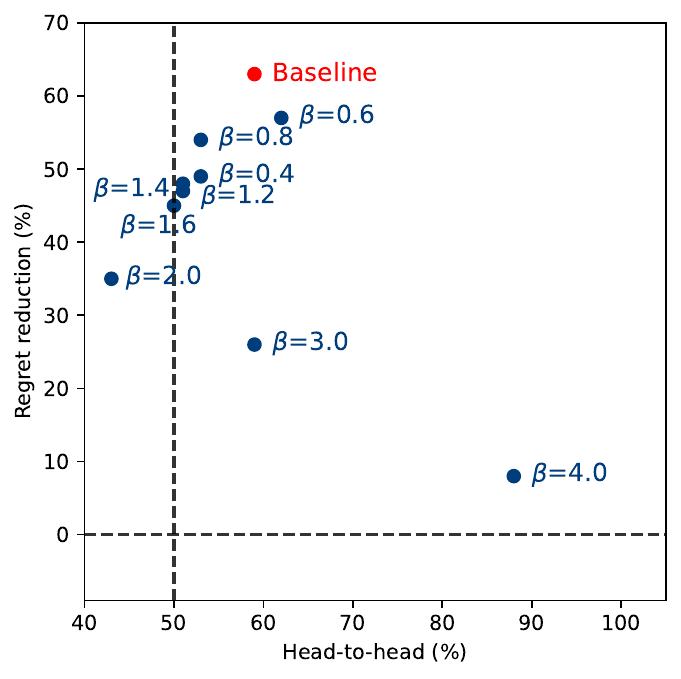}
	\includegraphics[width=0.4\linewidth]{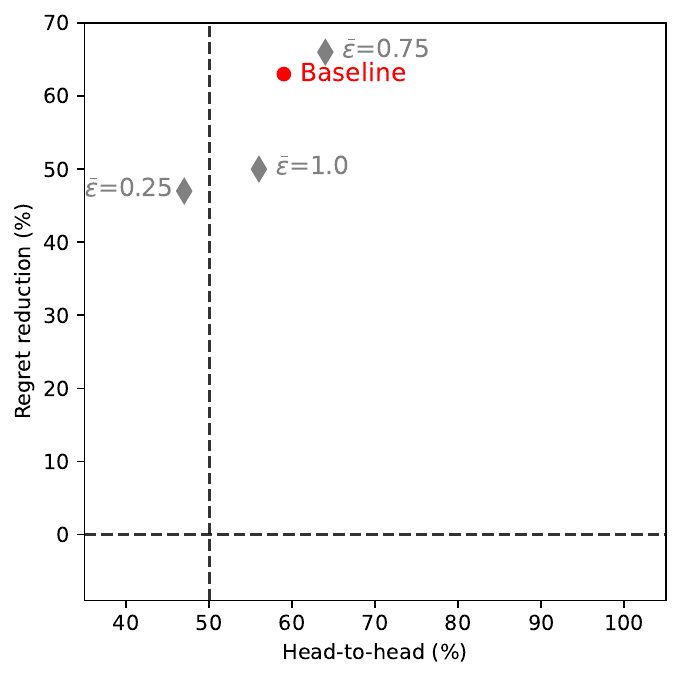} 
	\caption{Performance of \ref{model:ddr_empirical} against OLS under $3 \times 3$ grid and $\mu=0.75$ and $\lambda=0.8$ under various settings}
	\label{fig:DDR-OLS-3by3-mu=0.75-lamb=0.8_various_settings}
\end{figure}

From Figure~\ref{fig:DDR-OLS-3by3-mu=0.75-lamb=0.8_various_settings}, we first observe that the \ref{model:ddr_empirical} model consistently attains superior regret performance, as indicated by positive regret reduction across all evaluated settings. This finding highlights the advantage of employing the \ref{model:ddr_empirical} model. However, it is also apparent that in certain scenarios, the head-to-head performance of the \ref{model:ddr_empirical} model does not surpass that of OLS. For instance, when assessing the impact of training data size, the \ref{model:ddr_empirical} model exhibits improved head-to-head performance across all sample sizes except at $N = 1000$. This outcome is anticipated, as the OLS estimator's accuracy improves with larger sample sizes, thereby reducing the performance gap between \ref{model:ddr_empirical} and OLS in such cases. Nevertheless, Figure~\ref{fig:DDR-OLS-3by3-mu=0.75-lamb=0.8_various_settings} demonstrates that our \ref{model:ddr_empirical} model outperforms in terms of both regret reduction and head-to-head metrics in the majority of cases.

Similarly, the comparison between the \ref{model:ddr_empirical} and \ref{model:spo+} approaches is presented in Figure~\ref{fig:DDR-SPO-3by3-mu=0.75-lamb=0.8_various_settings}. The results indicate that the \ref{model:ddr_empirical} model consistently outperforms \ref{model:spo+} across all settings in terms of both regret reduction and head-to-head performance metrics, thereby demonstrating the superiority of the \ref{model:ddr_empirical} model over the \ref{model:spo+} baseline.

\begin{figure}[htbp]
    \vspace{-0.45cm}
	\centering
	\includegraphics[width=0.4\linewidth]{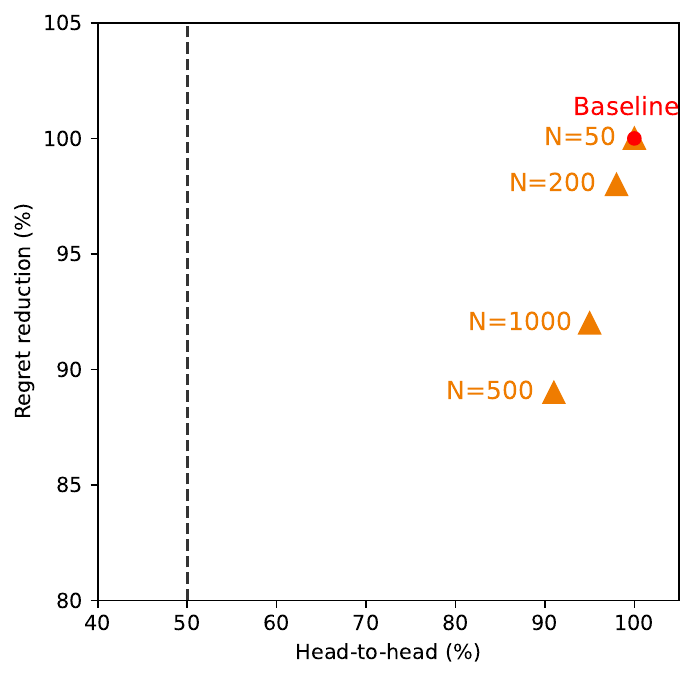} 
	\includegraphics[width=0.4\linewidth]{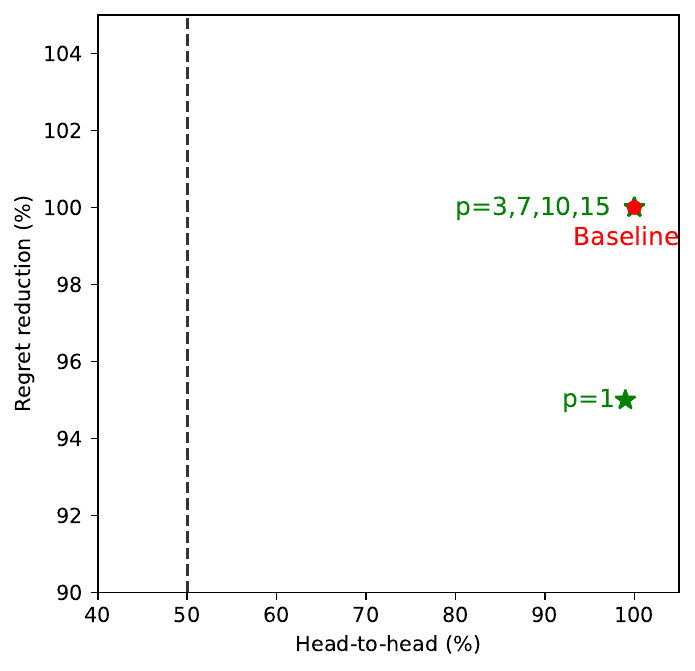} 

        \includegraphics[width=0.4\linewidth]{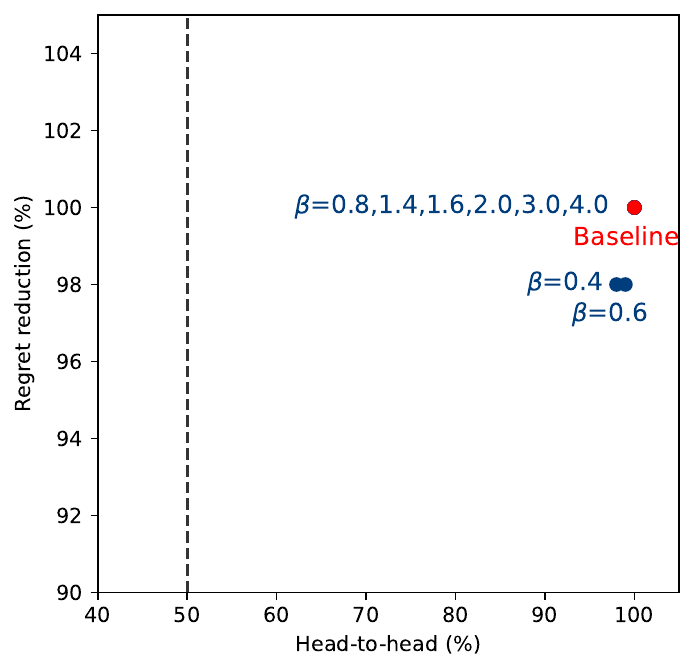} 
        \includegraphics[width=0.4\linewidth]{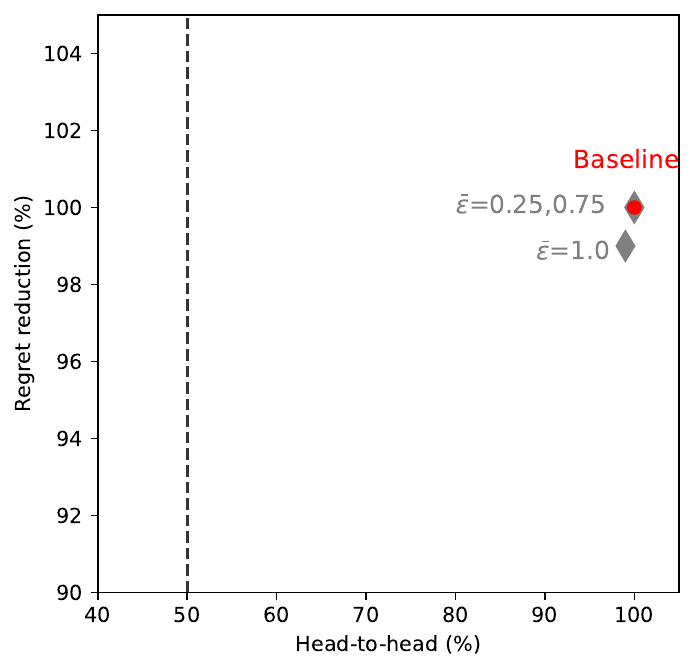}
	\caption{Performance of \ref{model:ddr_empirical} against SPO+ under $3 \times 3$ grid and $\mu=0.75$ and $\lambda=0.8$ under various settings}
	\label{fig:DDR-SPO-3by3-mu=0.75-lamb=0.8_various_settings}
	\vspace{-0.25cm}
\end{figure}

\section{DDR Tree}\label{append.tree}
While we have elected to use a linear learner in our numerical results, \ref{model:ddr_empirical} can more generally be applied to other learners. To illustrate this, in this section, we consider a tree-based learner in \ref{model:ddr_empirical}, and term it as DDR tree.

Consider a decision tree with $K$ leaf nodes, where the predictions $\bmh z$ are constrained to be equal if they fall into the same leaf node. We model this by defining $R_{k}$ the set of samples assigned to leaf node $k$, and predictions $\bmh z_k$ as being indexed by the leaf nodes $k \in [K]$. We write the DDR loss objective function as:
$$
L_{\rm DDR}(R_k, \bmh z_{k}) = \sum\limits_{k \in [K]}\sum\limits_{n \in R_{k}}\|\bmt z_{n} - \bmh z_{k}\|_{2}^{2} + \frac{\lambda}{N}\sum_{k \in [K]}\sum_{n \in R_{k}}\max_{\bm y_{n} \in \mathcal{Y}}\big\{-\mu \bm y_{n}^\top \bmt z_{n} - (1-\mu)\bm y_{n}^\top \bmh z_{k}\big\}
$$
where the first term represents the prediction accuracy and the second term represents the decision cost. Here, sets $R_{k}$ and predictions $\bmh z_{k}$ are the parameters (``weights") of the tree to be learnt and determined.

Using Proposition \ref{prop.compute}, the DDR loss function can be further reformulated as:
\begin{equation*}
    \begin{array}{llll}
         L_{\rm DDR}(R_k, \bmh z_{k})  & \displaystyle = \frac{\lambda}{N}\sum\limits_{k \in [K]}\sum\limits_{n \in R_{k}}\left( \frac{N}{\lambda}\|\bmt z_{n} - \bmh z_{k}\|_{2}^{2} + \min\limits_{\bm A^\top \bm \kappa_{n} \geq -\mu \bmt z_{n} - (1-\mu) \bmh z_{k}}\bm \kappa_{n}^\top \bm b \right).
    \end{array}
\end{equation*}
Accordingly, the prediction cost vectors for the leaf nodes, given the partitioning of observations into leaf nodes, \emph{i.e.}, the sets 
$R_{k}$, are determined by solving the following constrained optimization problem:
\begin{align*}
\bmh z^{\rm DDR} = \argmin\limits_{ \bmh z_{k} \; : \; \forall k \in [K]} & \left\{ \frac{\lambda}{N}\sum\limits_{k \in [K]}\sum\limits_{n \in R_{k}} \left( \frac{N}{\lambda}\|\bmt z_{n} - \bmh z_{k}\|_{2}^{2} + \bm \kappa_{n}^\top \bm b \right) \,\middle|\, \bm A^\top \bm \kappa_{n} \geq -\mu \bmt z_{n} - (1-\mu) \bmh z_{k}, \forall n \in R_{k},k \in [K] \Bigg. \right\}
\end{align*}
If the sets $R_{k}$ are known, the above optimization problem is easy. Instead, we propose the following two methodologies to optimize for $R_k$.

\subsection{Recursive Partitioning Approach.}
Emulating the CART algorithm, we construct the decision tree via a greedy recursive partitioning strategy. Let $x_{nj}$ denote the $j^{\scriptscriptstyle{\text{th}}}$ feature component of the $n^{\scriptscriptstyle{\text{th}}}$ observation in the training data. Starting from the entire training dataset, we consider binary splits defined by pairs $(j,s)$, where $j$ indicates the selected feature and $s$ denotes the split threshold. This split partitions the data into two disjoint subsets:
$$
R_{1}(j,s) = \{n \in [N] \,|\, x_{nj} \leq s \} \quad \text{ and } \quad R_{2}(j,s) = \{n \in [N] \,|\, x_{ij} > s \}.
$$
The optimal first split of the decision tree is determined by selecting the pair $(j,s)$ that minimizes the total $L_{\scriptscriptstyle{\text{DDR}}}$ loss over both subsets:
\begin{align*}
    &&\min\limits_{ \bmh z_{1}, \bmh z_{2}} & \quad \frac{\lambda}{N} \left[\sum\limits_{n \in R_{1}} \left( \frac{N}{\lambda}\|\bmt z_{n} - \bmh z_{1}\|_{2}^{2} + \bm \kappa_{n}^\top \bm b \right) + \sum\limits_{n \in R_{2}} \left( \frac{N}{\lambda}\|\bmt z_{n} - \bmh z_{2}\|_{2}^{2} + \bm \kappa_{n}^\top \bm b \right) \right] \\
    && \mbox{s.t.} & \quad \bm A^\top \bm \kappa_{n} \geq -\mu \bmt z_{n} - (1-\mu) \bmh z_{k}, \qquad \forall n \in R_{k}, k = 1,2 
\end{align*}
Once the optimal split is selected, the procedure is then recursively applied greedily to obtain binary splits in the resulting leaves until the stopping criteria is met. The prediction cost vectors $\bmh z^{\rm DDR}$ for each leaf node are then optimized simultaneously, after the complete tree structure has been determined.

\subsection{Integer Programming Approach.}
Motivated by recent advances in decision tree construction through optimization techniques \citep{Bertsimas_Dunn_2017_classification}, DDR tree can be solved as an integer programming model. Specifically, we aim to construct a decision tree based on a pre-specified tree depth $H$ and a minimum number of training observations $N_{min}$ required in each leaf node.

Note that not all leaf nodes are required to be active (\textit{i.e.}, contain training observations), and similarly, not all branch nodes need to be active splits (\textit{i.e.}, nodes that partition the data). 
Denote the index sets of leaf and branch nodes by 
$\mathcal{T}_{K}= \{1,\ldots,K\}$ and $\mathcal{T}_{B} = \{1, \ldots, B\}$, respectively. Let 
$g_{k} = \mathbb{I}\{\text{leaf~} k\text{~is not empty}\}$ and $d_{t} = \mathbb{I}\{\text{branch node}~t~\text{is an active split}\}$ be binary variables that indicate the activation status of the corresponding nodes.

Each node $t$ is associated with a split defined by decision variables $\bm a_{t} \in \{0,1\}^{p}$ and $b_{t} \in [0,1]$, where $\bm a_{t}$ indicates which feature component is involved with the split, and $b_t$ indicates the split point, assuming that all feature components have been normalized to the interval $[0,1]$. Note that $\bm a_{t}$ and $b_{t}$ are the decision variables.

Let $p(t)$ denote the parent node of $t$, and $A_{L}(t)$ the set of left ancestor nodes of node $t$, \emph{i.e.}, the set of ancestors of $t$ whose left branch has been followed on the path from the root to $t$. Similarly, we let $A_{R}(t)$ as the set of right ancestor nodes of $t$.

To ensure sufficient separation between data points in the splitting process, we define the minimum resolution for each feature dimension. Let $\varepsilon_{j} = \min \{x_{j}^{(n+1)} - x_{j}^{(n)}\left| x_{j}^{(n+1)} \neq x_{j}^{(n)}, n \in [N-1]\right.\}$ is the smallest nonzero difference between observed value of feature component $j$, where $x_{j}^{(n)}$ is the $n$-th largest value in the $j$-th feature. In addition, we denote $\varepsilon_{max} = \max_{j} \varepsilon_{j}$.

The integer programming formulation of the DDR decision tree is given as follows:
\begin{equation*}
        \begin{array}{llll}
        \min & \displaystyle \frac{\lambda}{N}\sum\limits_{k \in [K]}\sum\limits_{n \in [N]}r_{nk}\Big( \frac{N}{\lambda}\|\bmt z_{n} - \bmh z_{k}\|_{2}^{2} + \bm \kappa_{n}^\top \bm b \Big) \\
         \text{s.t }& \bm A^\top \bm \kappa_{n} \geq -\mu \bmt z_{n} - (1-\mu) \bmh z_{k} + (1-r_{nk})M, &\forall n \in [N],k \in [K] \\
        & \sum\limits_{k \in [K]}r_{nk} = 1, & \forall n \in [N]  \\
        & r_{nk} \leq g_{k}, & \forall n \in [N], l \in \mathcal{T}_{K} \\
        & \sum\limits_{n \in [N]}r_{nk} \geq N_{min}g_{k}, & \forall l \in \mathcal{T}_{K} \\
        & \bm a_{m}^\top \bmt x_{n} \geq b_{m} - (1 - r_{nk}), &\forall l \in \mathcal{T}_{K}, n \in [N], m \in A_{R}(k) \\
        & \bm a_{m}^\top (\bmt x_{n} + \bm \varepsilon) \leq b_{m} + (1 + \varepsilon_{max})(1 - r_{nk}), &\forall l \in \mathcal{T}_{K}, n \in [N], m \in A_{L}(k) \\
        & \sum\limits_{j \in [p]}a_{jt} = d_{t}, &\forall t \in \mathcal{T}_{B} \\
        & 1 - d_{t} \leq b_{t} \leq 1, & \forall t \in \mathcal{T}_{B} \\
        & d_{t} \leq d_{p(t)}, & \forall t \in \mathcal{T}_{B}\setminus \{1\} \\
        & a_{jt},d_{t} \in \{0,1\}, & \forall j \in [p], t \in \mathcal{T}_{B} \\
        & r_{nk},g_{k} \in \{0,1\}, & \forall n \in [N], l \in \mathcal{T}_{K}
    \end{array}
\end{equation*}
The objective function and the first four constraints form the DDR optimization problem, while the remaining constraints are adapted from the requirements of the optimal classification tree. For more details about the tree constraints, the reader can consult \cite{Bertsimas_Dunn_2017_classification}.

The objective function
$\frac{\lambda}{N}\sum\limits_{k \in [K]}\sum\limits_{n \in [N]}r_{nk}\Big( \frac{N}{\lambda}\|\bmt z_{n} - \bmh z_{k}\|_{2}^{2} + \bm \kappa_{n}^\top \bm b \Big)$ can be further reformulated as the following inequalities:
\begin{equation*}
    \begin{array}{llll}
         &  \displaystyle \frac{\lambda}{N}\sum\limits_{k \in [K]}\sum\limits_{n \in [N]}r_{nk}\Big( \frac{N}{\lambda}\varrho_{nk} + \bm \kappa_{n}^\top \bm b \Big)\\
         & (\varrho_{nk} + 1)^{2} \geq (\varrho_{nk} - 1)^{2} + 4\varpi^{2}\\
         & \varpi^{2} \geq \|\bmt z_{n} - \bmh z_{k}\|_{2}^{2}
    \end{array}
\end{equation*}
Moreover, the bilinear terms $r_{nk} \varrho_{nk} $ and $r_{nk} \kappa_{n}$ can be linearized by introducing additional variables and applying standard linearization techniques, such as McCormick envelopes.

\end{APPENDICES}
\end{document}